\PassOptionsToPackage{numbers}{natbib}
\documentclass{article}

\usepackage[preprint]{configuration/neurips_2025}

\usepackage[utf8]{inputenc} %
\usepackage[T1]{fontenc}    %
\usepackage{hyperref}       %
\usepackage{url}            %
\usepackage{booktabs}       %
\usepackage{amsfonts}       %
\usepackage{nicefrac}       %
\usepackage{microtype}      %
\usepackage{xcolor}         %
\usepackage[numbers]{natbib}

\usepackage{amsmath,amssymb,amsthm}

\providecommand{\abs}[1]{\left\lvert#1\right\rvert}
\providecommand{\norm}[1]{\left\lVert#1\right\rVert}

\providecommand{\R}{\mathbb{R}} %

\providecommand{\E}{{\mathbb E}}
\providecommand{\E}[1]{{\mathbb E}\left.#1\right. }        %
\providecommand{\EE}[2]{{\mathbb E}_{#1}\left.#2\right. }  %
\providecommand{\EEb}[2]{{\mathbb E}_{#1}\left[#2\right] } %

\providecommand{\dm}{\mathrm{d}}

\providecommand{\0}{\mathbf{0}}
\providecommand{\1}{\mathbf{1}}
\renewcommand{\aa}{\mathbf{a}}
\providecommand{\bb}{\mathbf{b}}
\providecommand{\cc}{\mathbf{c}}
\providecommand{\dd}{\mathbf{d}}

\renewcommand{\gg}{\mathbf{g}}

\providecommand{\hh}{\mathbf{h}}

\let\lll\ll
\renewcommand{\ll}{\mathbf{l}}

\providecommand{\nn}{\mathbf{n}}

\providecommand{\tt}{\mathbf{t}}

\providecommand{\vv}{\mathbf{v}}

\providecommand{\xx}{\mathbf{x}}
\providecommand{\yy}{\mathbf{y}}
\providecommand{\zz}{\mathbf{z}}

\providecommand{\mF}{\mathbf{F}}

\providecommand{\mI}{\mathbf{I}}

\providecommand{\mK}{\mathbf{K}}

\providecommand{\mTheta}{\boldsymbol{\Theta}}

\providecommand{\mxi}{\boldsymbol{\xi}}
\providecommand{\mepsilon}{\boldsymbol{\epsilon}}

\providecommand{\mtheta}{\boldsymbol{\theta}}
\providecommand{\mmu}{\boldsymbol{\mu}}
\providecommand{\mf}{\boldsymbol{f}}
\providecommand{\mv}{\boldsymbol{v}}
\providecommand{\mg}{\boldsymbol{g}}
\providecommand{\mmF}{\boldsymbol{F}}
\providecommand{\mmA}{\boldsymbol{A}}
\providecommand{\mmB}{\boldsymbol{B}}
\providecommand{\mmG}{\boldsymbol{G}}

\providecommand{\cJ}{\mathcal{J}}

\providecommand{\cL}{\mathcal{L}}

\providecommand{\cN}{\mathcal{N}}

\providecommand{\cT}{\mathcal{T}}

\providecommand{\cV}{\mathcal{V}}
\providecommand{\cX}{\mathcal{X}}

\newenvironment{talign*}
{\let\displaystyle\textstyle\csname align*\endcsname}
{\endalign}

\usepackage[utf8]{inputenc}         %
\usepackage[T1]{fontenc}            %
\usepackage{url}                    %
\usepackage{booktabs}               %
\usepackage{amsfonts}               %
\usepackage{nicefrac}               %
\usepackage{microtype}              %
\usepackage{xcolor}                 %
\usepackage{algorithm}
\usepackage{algorithmic}
\usepackage{graphicx}
\usepackage{subcaption}
\usepackage[flushleft]{threeparttable}
\usepackage{float}
\usepackage{multirow}
\usepackage{xspace}
\usepackage{enumitem}
\usepackage[font=small]{caption}
\usepackage{autobreak}
\usepackage{sidecap}
\usepackage{wrapfig}
\usepackage{bbding}
\usepackage[toc, page, header]{appendix}
\usepackage{tikz}
\usepackage{xcolor}
\usepackage{pifont}
\usepackage{mdframed}
\usepackage{colortbl}

\hypersetup{
  colorlinks=true,
  linkcolor=blue,
  citecolor=blue,
  urlcolor=blue
}

\definecolor{coral}{RGB}{255,127,80}
\definecolor{darkgreen}{RGB}{0,100,0}
\definecolor{darkyellow}{RGB}{204,153,0}
\definecolor{salmon}{RGB}{250,128,114}
\definecolor{darkred}{RGB}{150,0,0}
\newcommand{\darkredtext}[1]{{\color{darkred}#1}}

\newcommand{\secref}[1]{\hyperref[#1]{\darkredtext{Sec.~\ref*{#1}}}}
\newcommand{\thmref}[1]{\hyperref[#1]{\darkredtext{Thm.~\ref*{#1}}}}
\newcommand{\defref}[1]{\hyperref[#1]{\darkredtext{Def.~\ref*{#1}}}}
\newcommand{\propref}[1]{\hyperref[#1]{\darkredtext{Prop.~\ref*{#1}}}}
\newcommand{\assumpref}[1]{\hyperref[#1]{\darkredtext{Assump.~\ref*{#1}}}}
\newcommand{\remarkref}[1]{\hyperref[#1]{\darkredtext{Rem.~\ref*{#1}}}}
\newcommand{\hypref}[1]{\hyperref[#1]{\darkredtext{Hyp.~\ref*{#1}}}}
\newcommand{\conjref}[1]{\hyperref[#1]{\darkredtext{Conj.~\ref*{#1}}}}
\newcommand{\lemref}[1]{\hyperref[#1]{\darkredtext{Lem.~\ref*{#1}}}}
\newcommand{\corref}[1]{\hyperref[#1]{\darkredtext{Cor.~\ref*{#1}}}}
\newcommand{\noteref}[1]{\hyperref[#1]{\darkredtext{Nota.~\ref*{#1}}}}
\newcommand{\claimref}[1]{\hyperref[#1]{\darkredtext{Clm.~\ref*{#1}}}}
\newcommand{\obsref}[1]{\hyperref[#1]{\darkredtext{Obs.~\ref*{#1}}}}
\newcommand{\algref}[1]{\hyperref[#1]{\darkredtext{Alg.~\ref*{#1}}}}
\newcommand{\figref}[1]{\hyperref[#1]{\darkredtext{Fig.~\ref*{#1}}}}
\newcommand{\tabref}[1]{\hyperref[#1]{\darkredtext{Tab.~\ref*{#1}}}}
\newcommand{\appref}[1]{\hyperref[#1]{\darkredtext{App.~\ref*{#1}}}}

\newtheoremstyle{custom}
{1pt} %
{1pt} %
{\itshape} %
{} %
{\bfseries} %
{} %
{ } %
{\thmname{#1} \thmnumber{#2} \thmnote{(#3)} . } %

\theoremstyle{custom}

\newtheorem{innerdefinition}{Definition}
\newtheorem{innerproposition}{Proposition}
\newtheorem{innerassumption}{Assumption}
\newtheorem{innerremark}{Remark}
\newtheorem{innertheorem}{Theorem}
\newtheorem{innerhypothesis}{Hypothesis}
\newtheorem{innerconjecture}{Conjecture}
\newtheorem{innerlemma}{Lemma}
\newtheorem{innercorollary}{Corollary}
\newtheorem{innernotation}{Notation}
\newtheorem{innerclaim}{Claim}
\newtheorem{innerproblem}{Problem}

\newtheorem{innerobservation}{Observation}

\newmdenv[
  backgroundcolor=gray!10,
  linecolor=gray!100,
  linewidth=0.8pt,
  skipabove=2pt,
  skipbelow=2pt,
  innertopmargin=10pt,
  innerbottommargin=5pt,
  innerleftmargin=5pt,
  innerrightmargin=5pt,
]{definitionframe}

\newmdenv[
  backgroundcolor=blue!10,
  linecolor=blue!100,
  linewidth=0.8pt,
  skipabove=2pt,
  skipbelow=2pt,
  innertopmargin=10pt,
  innerbottommargin=5pt,
  innerleftmargin=5pt,
  innerrightmargin=5pt,
]{propositionframe}

\newmdenv[
  backgroundcolor=green!10,
  linecolor=green!100,
  linewidth=0.8pt,
  skipabove=2pt,
  skipbelow=2pt,
  innertopmargin=10pt,
  innerbottommargin=5pt,
  innerleftmargin=5pt,
  innerrightmargin=5pt,
]{assumptionframe}

\newmdenv[
  backgroundcolor=yellow!10,
  linecolor=yellow!100,
  linewidth=0.8pt,
  skipabove=2pt,
  skipbelow=2pt,
  innertopmargin=10pt,
  innerbottommargin=5pt,
  innerleftmargin=5pt,
  innerrightmargin=5pt,
]{remarkframe}

\newmdenv[
  backgroundcolor=red!10,
  linecolor=red!100,
  linewidth=0.8pt,
  skipabove=2pt,
  skipbelow=2pt,
  innertopmargin=10pt,
  innerbottommargin=5pt,
  innerleftmargin=5pt,
  innerrightmargin=5pt,
]{theoremframe}

\newmdenv[
  backgroundcolor=purple!10,
  linecolor=purple!100,
  linewidth=0.8pt,
  skipabove=2pt,
  skipbelow=2pt,
  innertopmargin=10pt,
  innerbottommargin=5pt,
  innerleftmargin=5pt,
  innerrightmargin=5pt,
]{hypothesisframe}

\newmdenv[
  backgroundcolor=orange!10,
  linecolor=orange!100,
  linewidth=0.8pt,
  skipabove=2pt,
  skipbelow=2pt,
  innertopmargin=10pt,
  innerbottommargin=5pt,
  innerleftmargin=5pt,
  innerrightmargin=5pt,
]{conjectureframe}

\newmdenv[
  backgroundcolor=cyan!10,
  linecolor=cyan!100,
  linewidth=0.8pt,
  skipabove=2pt,
  skipbelow=2pt,
  innertopmargin=10pt,
  innerbottommargin=5pt,
  innerleftmargin=5pt,
  innerrightmargin=5pt,
]{lemmaframe}

\newmdenv[
  backgroundcolor=magenta!10,
  linecolor=magenta!100,
  linewidth=0.8pt,
  skipabove=2pt,
  skipbelow=2pt,
  innertopmargin=10pt,
  innerbottommargin=5pt,
  innerleftmargin=5pt,
  innerrightmargin=5pt,
]{corollaryframe}

\newmdenv[
  backgroundcolor=pink!10,
  linecolor=pink!100,
  linewidth=0.8pt,
  skipabove=2pt,
  skipbelow=2pt,
  innertopmargin=10pt,
  innerbottommargin=5pt,
  innerleftmargin=5pt,
  innerrightmargin=5pt,
]{notationframe}

\newmdenv[
  backgroundcolor=violet!10,
  linecolor=violet!100,
  linewidth=0.8pt,
  skipabove=2pt,
  skipbelow=2pt,
  innertopmargin=10pt,
  innerbottommargin=5pt,
  innerleftmargin=5pt,
  innerrightmargin=5pt,
]{claimframe}

\newmdenv[
  backgroundcolor=salmon!10,
  linecolor=salmon!100,
  linewidth=0.8pt,
  skipabove=2pt,
  skipbelow=2pt,
  innertopmargin=10pt,
  innerbottommargin=5pt,
  innerleftmargin=5pt,
  innerrightmargin=5pt,
]{problemframe}

\newmdenv[
  backgroundcolor=lavender!10,
  linecolor=lavender!100,
  linewidth=0.8pt,
  skipabove=2pt,
  skipbelow=2pt,
  innertopmargin=10pt,
  innerbottommargin=5pt,
  innerleftmargin=5pt,
  innerrightmargin=5pt,
]{observationframe}

\newenvironment{proposition}
{\begin{propositionframe}\begin{innerproposition}}
      {\end{innerproposition}\end{propositionframe}}

\newenvironment{remark}
{\begin{remarkframe}\begin{innerremark}}
      {\end{innerremark}\end{remarkframe}}

\newenvironment{theorem}
{\begin{theoremframe}\begin{innertheorem}}
      {\end{innertheorem}\end{theoremframe}}

\newenvironment{lemma}
{\begin{lemmaframe}\begin{innerlemma}}
      {\end{innerlemma}\end{lemmaframe}}

\newenvironment{corollary}
{\begin{corollaryframe}\begin{innercorollary}}
      {\end{innercorollary}\end{corollaryframe}}

\newcommand{\method}{\textsc{UCGM}\xspace} %
\newcommand{\methods}{\method-\{T,\,S\}\xspace} %

\newcommand{\sg}{\mathop{{\color{red}\operatorname{sg}}}\nolimits}
\newcommand{\dg}[1]{\textcolor{darkgreen}{$^{\downarrow#1}$}}
\newcommand{\up}[1]{\textcolor{red}{$^{\uparrow#1}$}}
\newcommand{\rp}{$\boldsymbol{\color{red!70!black}\oplus}$}

\title{Unified Continuous Generative Models}

\author{Peng Sun$^{1,2}$ \quad Yi Jiang$^{2}$ \quad Tao Lin$^{1,}\thanks{Corresponding author.} $ \\
$^{1}$Westlake University \quad $^{2}$Zhejiang University \\
{\tt\small sunpeng@westlake.edu.cn, yi\_jiang@zju.edu.cn, lintao@westlake.edu.cn} \\
}

\begin{document}

\maketitle

\begin{abstract}
    Recent advances in continuous generative models, encompassing multi-step processes such as diffusion and flow-matching (typically requiring $8$-$1000$ sampling steps) and few-step methods like consistency models (typically $1$-$8$ steps), have yielded impressive generative performance.
    Existing work, however, often treats these approaches as distinct learning paradigms, leading to disparate training and sampling methodologies.
    We propose a unified framework designed for the training, sampling, and understanding of these models. Our implementation, the \underline{\textbf{U}}nified \underline{\textbf{C}}ontinuous \underline{\textbf{G}}enerative \underline{\textbf{M}}odels \underline{\textbf{T}}rainer and \underline{\textbf{S}}ampler (\methods), demonstrates state-of-the-art (SOTA) capabilities.
    For instance, on ImageNet $256\times256$ using a $675\text{M}$ diffusion transformer model, \method-T trains a multi-step model achieving $1.30$ FID in $20$ sampling steps, and a few-step model achieving $1.42$ FID in $2$ sampling steps.
    Furthermore, applying our \method-S to a pre-trained model from prior work improves its FID from $1.26$ (at $250$ steps) to $1.06$ in only $40$ steps, incurring no additional cost.
    Code: \url{https://github.com/LINs-lab/UCGM}.
\end{abstract}

\section{Introduction}
\label{sec:intro}

\begin{figure}[h]
    \centering
    \begin{subfigure}[b]{0.48\textwidth}
        \centering
        \includegraphics[width=\linewidth]{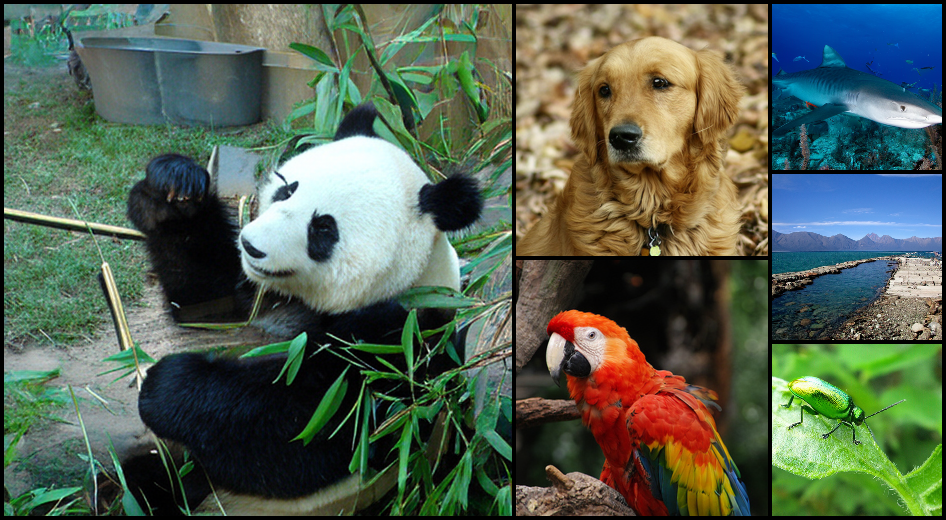}
        \caption{\textbf{NFE~$=40$, FID~$=1.48$.}}
    \end{subfigure}
    \hfill
    \begin{subfigure}[b]{0.48\textwidth}
        \centering
        \includegraphics[width=\linewidth]{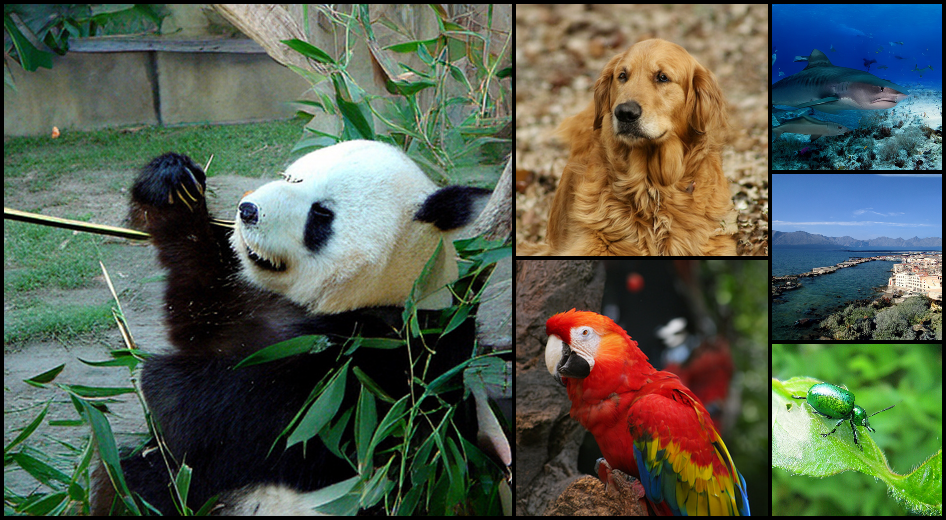}
        \caption{\textbf{NFE~$=2$, FID~$=1.75$.}}
    \end{subfigure}
    \vspace{-0.5em}
    \caption{\small{
            \textbf{Generated samples from two $675\mathrm{M}$ diffusion transformers trained with our \method~on ImageNet-1K $512\!\times\!512$.}
            The figure showcases generated samples illustrating the flexibility of Number of Function Evaluation (NFE) and superior performance achieved by our \method.
            The left subfigure presents results with NFE~$=40$ (multi-step), while the right subfigure shows results with NFE~$=2$ (few-step).
            Note that the samples are sampled \textit{\textbf{without} classifier-free guidance (CFG) or other guidance} techniques.
        }}
    \vspace{-0.5em}
    \label{fig:key_result}
\end{figure}

Continuous generative models, encompassing diffusion models~\citep{ho2020denoising,song2020denoising}, flow-matching models~\citep{lipman2022flow,ma2024sit}, and consistency models~\citep{song2023consistency,lu2024simplifying}, have demonstrated remarkable success in synthesizing high-fidelity data across diverse applications, including image and video generation~\citep{peebles2023scalable,chen2024deep,ma2024sit,xie2024sana,ho2022video,chen2025sana}.

Training and sampling of these models necessitate substantial computational resources~\citep{karras2022elucidating,karras2024analyzing}.
Moreover, current research largely treats distinct model paradigms independently, leading to paradigm-specific training and sampling methodologies. This fragmentation introduces two primary challenges:
(a) \textbf{a deficit in unified theoretical and empirical understanding}, which constrains the transfer of advancements across different paradigms;
and (b) \textbf{limited cross-paradigm generalization}, as algorithms optimized for one paradigm (e.g., diffusion models) are often incompatible with others.

To address these limitations, we introduce \method, a novel framework that establishes a unified foundation for the training, sampling, and conceptual understanding of continuous generative models.
Specifically, the unified trainer \method-T~is built upon a unified training objective, parameterized by a consistency ratio $\lambda \in [0,1]$.
This allows a single training paradigm to flexibly produce models tailored for different inference regimes: models behave akin to multi-step diffusion or flow-matching approaches when $\lambda$ is close to $0$, and transition towards few-step consistency-like models as $\lambda$ approaches $1$.
Moreover, this versatility can extend to compatibility with various noise schedules (e.g., linear, triangular, quadratic) without requiring bespoke algorithm modifications.

Complementing the unified trainer \method-T, we propose a unified sampling algorithm, \method-S, designed to work seamlessly with models trained via our objective.
Crucially, \method-S enhances and accelerates sampling from pre-trained models, encompassing those from distinct prior paradigms and models trained with \method-T.
The unifying nature of our \method is further underscored by its ability to encapsulate prominent existing continuous generative paradigms as specific instantiations of \method, as detailed in \tabref{tab:key_result}.
Moreover, as illustrated in~\figref{fig:key_result}, models trained with \method~can achieve excellent sample quality across a wide range of Number of Function Evaluations (NFEs).

A key innovation within \method~is the introduction of self-boosting techniques for both training and sampling.
The training-time self-boosting mechanism enhances model quality and training efficiency~(cf., \secref{sec:uni_training}), significantly reducing or eliminating the need for computationally expensive guidance techniques~\citep{ho2022classifier,karras2024guiding} during inference.
The sampling-time self-boosting, through our proposed estimation extrapolation~(cf., \secref{sec:uni_sampling}), markedly improves generation fidelity while minimizing NFEs \textit{without requiring additional cost}.
\textbf{In summary, our contributions are:}
\begin{enumerate}[label=(\alph*), nosep, leftmargin=16pt]
    \item A unified trainer (\method-T) that seamlessly bridges few-step (e.g., consistency models) and multi-step (e.g., diffusion, flow-matching) generative paradigms, accommodating diverse model architectures, latent autoencoders, and noise schedules.
    \item A versatile and unified sampler (\method-S) compatible with our trained models and, importantly, adaptable for accelerating and improving pre-trained models from existing yet distinct paradigms.
    \item A self-boosting training mechanism enhances model performance and efficiency while reducing reliance on external guidance techniques.
          Separately, a computation-free self-boosting sampling technique significantly enhances generation quality with reduced inference costs.
\end{enumerate}
Extensive experiments validate the effectiveness and efficiency of \method. Our approach consistently matches or surpasses SOTA methods across various datasets, architectures, and resolutions, for both few-step and multi-step generation tasks (cf., the experimental results in \secref{sec:exp}).

\begin{table}[]
    \caption{\textbf{Existing continuous generative paradigms as special cases of our \method.} Prominent continuous generative models, such as Diffusion, Flow Matching, and Consistency models, can be formulated as specific parameterizations of our \method. The columns detail the required parameterizations for the transport coefficients $\alpha(\cdot), \gamma(\cdot), \hat{\alpha}(\cdot), \hat{\gamma}(\cdot)$ and parameters $\lambda, \rho, \nu$ of \method. Note that $\sigma(t)$ is defined as $e^{4(2.68t-1.59)}$ in this table.}
    \label{tab:key_result}
    \resizebox{\textwidth}{!}{
        \begin{tabular}{@{}cc|ccccccc@{}}
            \toprule
            \multicolumn{2}{c|}{Paradigm}                                                & \multicolumn{7}{c}{\method-based Parameterization}                                                                                                                                                                                                                                                                                                                                                                                              \\ \midrule
            \multicolumn{1}{c|}{Type}                                                    & e.g.,                                              & \multicolumn{1}{c|}{$\alpha(t)=$}                                       & \multicolumn{1}{c|}{$\gamma(t)=$}                               & \multicolumn{1}{c|}{$\hat{\alpha}(t)=$}                            & \multicolumn{1}{c|}{$\hat{\gamma}(t)=$}                                  & \multicolumn{1}{c|}{$\lambda\in[0,1]$} & \multicolumn{1}{c|}{$\rho\in[0,1]$} & $\nu\in\{1,2\}$ \\ \midrule
            \multicolumn{1}{c|}{Diffusion}                                               & EDM\citep{karras2022elucidating}                   & \multicolumn{1}{c|}{$\frac{\sigma(t)}{\sqrt{\sigma^2(t)+\frac{1}{4}}}$} & \multicolumn{1}{c|}{$\frac{1}{\sqrt{\sigma^2(t)+\frac{1}{4}}}$} & \multicolumn{1}{c|}{$\frac{-0.5}{\sqrt{\sigma^2(t)+\frac{1}{4}}}$} & \multicolumn{1}{c|}{$\frac{2\sigma(t)}{\sqrt{\sigma^2(t)+\frac{1}{4}}}$} & \multicolumn{1}{c|}{$0$}               & \multicolumn{1}{c|}{$\geq0$}        & $2$             \\ \midrule
            \multicolumn{1}{c|}{\begin{tabular}[c]{@{}c@{}}Flow\\ Matching\end{tabular}} & OT\citep{lipman2022flow}                           & \multicolumn{1}{c|}{$t$}                                                & \multicolumn{1}{c|}{$1-t$}                                      & \multicolumn{1}{c|}{$1$}                                           & \multicolumn{1}{c|}{$-1$}                                                & \multicolumn{1}{c|}{$0$}               & \multicolumn{1}{c|}{$\geq0$}        & $1$             \\ \midrule
            \multicolumn{1}{c|}{Consistency}                                             & sCM\citep{lu2024simplifying}                       & \multicolumn{1}{c|}{$\sin(t\cdot \frac{\pi}{2})$}                       & \multicolumn{1}{c|}{$\cos(t\cdot \frac{\pi}{2})$}               & \multicolumn{1}{c|}{$\cos(t\cdot \frac{\pi}{2})$}                  & \multicolumn{1}{c|}{$\sin(t\cdot \frac{-\pi}{2})$}                       & \multicolumn{1}{c|}{$1$}               & \multicolumn{1}{c|}{$1$}            & $1$             \\ \bottomrule
        \end{tabular}}
    \vspace{-0.5em}
\end{table}

\section{Preliminaries}
\label{sec:preliminaries}

Given a training dataset $D$, let $p(\xx)$ represent its underlying data distribution, or $p(\xx|\cc)$ under a condition $\cc$.
Continuous generative models seek to learn an estimator that gradually transforms a simple source distribution $p(\zz)$ into a complex target distribution $p(\xx)$ within a continuous space.
Typically, $p(\zz)$ is represented by the standard Gaussian distribution $\cN(\0, \mI)$.
For instance, diffusion models generate samples by learning to reverse a noising process that gradually perturbs a data sample $\xx \sim p(\xx)$ into a noisy version $\xx_t = \alpha(t) \xx + \sigma(t) \zz$, where $\zz \sim \cN(\0, \mI)$.
Over the range $t \in [0, T]$, the perturbation intensifies with increasing $t$, where higher $t$ values indicate more pronounced noise.
Below, we introduce three prominent learning paradigms for deep continuous generative models.

\paragraph{Diffusion models~\citep{ho2020denoising,song2020score,karras2022elucidating}.}
In the widely adopted EDM method~\citep{karras2022elucidating}, the noising process is defined by setting $\alpha(t)=1$, $\sigma(t)=t$.
The training objective is given by
$\EEb{\xx,\zz,t}{ \omega(t) \norm{ \mf_{\mtheta}(\xx_t,t) - \xx }_2^2 }$
where \( \omega(t) \) is a weighting function.
The diffusion model is parameterized by
$\mf_{\mtheta}(\xx_t,t) = c_{\text{skip}}(t)\xx_t + c_{\text{out}}(t)\mmF_{\mtheta}(c_{\text{in}}(t)\xx_t, c_{\text{noise}}(t))$
where \( \mmF_{\theta} \) is a neural network, and the coefficients \( c_{\text{skip}} \), \( c_{\text{out}} \), \( c_{\text{in}} \), and \( c_{\text{noise}} \) are manually designed.
During sampling, EDM solves the Probability Flow Ordinary Differential Equation (PF-ODE)~\citep{song2020score}:
$\frac{\dm \xx_t}{\dm t} = [\xx_t - \mf_{\mtheta}(\xx_t,t)]/t$,
integrated from \( t = T \) to \( t = 0 \).

\paragraph{Flow matching~\citep{lipman2022flow}.}
Flow matching models are similar to diffusion models but differ in the transport process from the source to the target distribution and in the neural network training objective.
The forward transport process utilizes differentiable coefficients $\alpha(t)$ and $\gamma(t)$, such that $\xx_t = \alpha(t)\zz + \gamma(t)\xx$.
Typically, the coefficients satisfy the boundary conditions $\alpha(1) = \gamma(0) = 1$ and $\alpha(0) = \gamma(1) = 0$. The training objective is given by
$
    \EEb{\xx, \zz, t}{ \omega(t) \norm{ \mmF_{\mtheta}(\xx_t, t) - (\frac{\dm \alpha_t}{\dm t} \zz + \frac{\dm \sigma_t}{\dm t} \xx) }_2^2 }
$.
Similar to diffusion models, the reverse transport process (i.e., sampling process) begins at $t = 1$ with $\xx_1 \sim \cN(\0, \mI)$ and solves the PF-ODE:
$
    \frac{\dm \xx_t}{\dm t} = \mmF_{\mtheta}(\xx_t, t)
$,
integrated from $t = 1$ to $t = 0$.

\paragraph{Consistency models~\citep{song2023consistency,lu2024simplifying}.}
A consistency model $\mf_{\mtheta}(\xx_t, t)$ is trained to map the noisy input $\xx_t$ directly to the corresponding clean data $\xx$ in one or few steps by following the sampling trajectory of the PF-ODE starting from $\xx_t$.
To be valid, $\mf_{\mtheta}$ must satisfy the boundary condition $\mf_{\mtheta}(\xx, 0) \equiv \xx$.
Inspired by EDM~\citep{karras2022elucidating}, one approach to enforce this condition is to parameterize the consistency model as
$\mf_{\mtheta}(\xx_t,t) = c_{\text{skip}}(t)\xx_t + c_{\text{out}}(t)\mmF_{\mtheta} \left( c_{\text{in}}(t)\xx_t, c_{\text{noise}}(t) \right)$
with $c_{\text{skip}}(0) = 1$ and $c_{\text{out}}(0)=0$.
The training objective is defined between two adjacent time steps with a finite distance:
$
    \EEb{\xx_t,t}{
        \omega(t) d\left( \mf_{\mtheta}(\xx_t,t), \mf_{{\mtheta}^-}(\xx_{t-\Delta t}, t-\Delta t) \right)
    }
$,
where ${\mtheta}^{-}$ denotes $\operatorname{stopgrad}({\mtheta})$, $\Delta t > 0$ is the distance between adjacent time steps, and $d(\cdot, \cdot)$ is a metric function.
Discrete-time consistency models are sensitive to the choice of $\Delta t$, necessitating manually designed annealing schedules~\citep{song2023improved,geng2024consistency} for rapid convergence.
This limitation is addressed by proposing a training objective for continuous consistency models~\citep{lu2024simplifying}, derived by taking the limit as $\Delta t \to 0$.

In summary, both diffusion and flow-matching models are multi-step frameworks operating within a continuous space, whereas consistency models are designed as few-step approaches.

\section{Methodology}
\label{sec:methodology}

We first introduce our unified training objective and algorithm, \method-T, applicable to both few-step and multi-step models, including consistency, diffusion, and flow-matching frameworks.
Additionally, we present \method-S, our unified sampling algorithm, which is effective across all these models.

\subsection{Unifying Training Objective for Continuous Generative Models}
\label{sec:uni_training}

We first propose a unified training objective for diffusion and flow-matching models, which constitute all multi-step continuous generative models.
Moreover, we extend this unified objective to encompass both few-step and multi-step models.

\paragraph{Unified training objective for multi-step continuous generative models.}
We introduce a generalized training objective below that effectively trains generative models while encompassing the formulations presented in existing studies:
\begin{equation}
    \label{eq:objective}
    \cL(\mtheta):=\EEb{(\zz,\xx)\sim p(\zz,\xx),t}{ \frac{1}{\omega(t)} \norm{ \mmF_{\mtheta}(\xx_t,t) - \zz_t }_2^2 } \, ,
\end{equation}
where time \( t \in [0,1] \), \( \omega(t) \) is the weighting function for the loss, \( \mmF_{\mtheta} \) is a neural network\footnote{For simplicity, unless otherwise specified, we assume that any conditioning information $\cc$ is incorporated into the network input. Thus, $\mmF_{\mtheta}(\xx_t, t)$ should be understood as $\mmF_{\mtheta}(\xx_t, t, \cc)$ when $\cc$ is applicable.}
with parameters \( \mtheta \), \( \xx_t = \alpha(t)\zz + \gamma(t)\xx \), and \( \zz_t = \hat{\alpha}(t)\zz + \hat{\gamma}(t)\xx \).
Here, \( \alpha(t) \), \( \gamma(t) \), \( \hat{\alpha}(t) \), and \( \hat{\gamma}(t) \) are the unified transport coefficients defined for \method.
Additionally,
to efficiently and robustly train multi-step continuous generative models using objective \eqref{eq:objective}, we propose three \textit{necessary constraints}:
\begin{enumerate}[label=(\alph*), nosep, leftmargin=16pt]
    \item \( \alpha(t) \) is continuous over the interval \( t \in [0,1] \), with \( \alpha(0) = 0 \), \( \alpha(1) = 1 \), and \( \frac{\mathrm{d}\alpha(t)}{\mathrm{d}t} \geq 0 \).
    \item \( \gamma(t) \) is continuous over the interval \( t \in [0,1] \), with \( \gamma(0) = 1 \), \( \gamma(1) = 0 \), and \( \frac{\mathrm{d}\gamma(t)}{\mathrm{d}t} \leq 0 \).
    \item For all \( t \in (0,1) \), it holds that \( |\alpha(t)\cdot \hat{\gamma}(t) - \hat{\alpha}(t)\cdot \gamma(t)| > 0 \) to ensure that \( \alpha(t)\cdot \hat{\gamma}(t) - \hat{\alpha}(t)\cdot \gamma(t) \) is non-zero and can serve as the denominator in \eqref{eq:transformed_objective}.
\end{enumerate}

Under these constraints, diffusion and flow-matching models are special cases of our unified training objective \eqref{eq:objective} with additional restrictions (\appref{app:cal_transport} details EDM models transformation to \method):
\begin{enumerate}[label=(\alph*), nosep, leftmargin=16pt]
    \item For example, following EDM~\citep{karras2022elucidating,karras2024analyzing}, by setting $\alpha(t) = 1$ and $\sigma(t) = t$, diffusion models based on EDM can be derived from \eqref{eq:objective} provided that the constraint \( \gamma(t) / \alpha(t) = t \) is satisfied\footnote{
              In EDM, with $\sigma(t) = t$, the input of neural network $\mmF_{\mtheta}$ is \( c_{\text{in}}(t)\xx_t = c_{\text{in}}(t) \cdot (\xx + t \cdot \zz) \). Although \( c_{\text{in}}(t) \) can be manually adjusted, the coefficient before \( \zz \) remains \( t \) times that of \( \xx \).
          }.
    \item Similarly, flow-matching models can be derived only when \( \hat{\alpha}(t) = \frac{\mathrm{d}\alpha(t)}{\mathrm{d}t} \) and \( \hat{\gamma}(t) = \frac{\mathrm{d}\gamma(t)}{\mathrm{d}t} \) (see~\secref{sec:preliminaries} for more technical details about EDM-based and flow-based models).
\end{enumerate}

\begin{algorithm*}[t]
    \caption{(\textbf{\method-T}). A Unified and Efficient Trainer for Few-step and Multi-step Continuous Generative Models (including Diffusion, Flow Matching, and Consistency Models)}
    \label{alg:udmm}
    \begin{algorithmic}[1]
        \REQUIRE Dataset $D$, transport coefficients \{$\alpha(\cdot)$, $\gamma(\cdot)$, $\hat{\alpha}(\cdot)$, $\hat{\gamma}(\cdot)$\}, neural network $\mmF_{\mtheta}$, enhancement ratio $\zeta$, Beta distribution parameters $(\theta_1,\theta_2)$, learning rate $\eta$, stop gradient operator $\sg$.
        \ENSURE Trained neural network $\mmF_{\mtheta}$ for generating samples from $p(\xx)$.
        \REPEAT
        \STATE Sample $\zz \sim \mathcal{N}(\mathbf{0}, \mathbf{I})$, $\xx \sim D$, $t \sim \phi(t):=\mathrm{Beta}(\theta_1, \theta_2)$
        \STATE Compute input data, such as $\xx_t=\alpha(t)\cdot\zz+\gamma(t)\cdot\xx$ and $\xx_{\lambda t}=\alpha(\lambda t)\cdot\zz+\gamma(\lambda t)\cdot\xx$
        \STATE Compute model output $\mmF_t = \mmF_{\mtheta}(\xx_{t}, t)$ and set $\zz^\star = \zz$ and $\xx^\star = \xx$
        \IF{$\zeta \in (0,1)$}
        \STATE Let $\mmF^\varnothing_t = \mmF_{\mtheta^-}(\xx_{t}, t, \varnothing)$ to get enhanced $\xx^\star = \mxi(\xx, t, \mf^{\mathrm{\xx}}(\sg(\mmF_t), \xx_t, t), \mf^{\mathrm{\xx}}(\mmF^\varnothing_t, \xx_t, t))$ and $\zz^\star = \mxi(\zz, t, \mf^{\mathrm{\zz}}(\sg(\mmF_t), \xx_t, t), \mf^{\mathrm{\zz}}(\mmF^\varnothing_t, \xx_t, t))$\COMMENT{\textcolor{darkgreen}{Note that
                $
                    \mxi(\aa, t, \bb, \dd) := \aa + \left(\zeta + \1_{t > s} \left(\frac{1}{2} - \zeta\right)\right) \cdot \left(\bb - \1_{t > s} \cdot \aa - \dd (1 - \1_{t > s})\right)
                $,
                where $\1(\cdot)$ is the indicator function}}
        \ENDIF
        \IF{$\lambda \in [0,1)$}
        \STATE Compute $\xx_t^\star = \alpha(t)\cdot\zz^\star+\gamma(t)\cdot\xx^\star$ and $\xx_{\lambda t}^\star=\alpha(\lambda t)\cdot\zz^\star+\gamma(\lambda t)\cdot\xx^\star$
        \STATE Compute $\Delta \mf^{\mathrm{\xx}}_t= \mf^{\mathrm{\xx}}(\sg(\mmF_t), \xx_t^\star, t) \cdot (\frac{1}{t - \lambda t}) - \mf^{\mathrm{\xx}}(\mmF_{\mtheta^-}(\xx_{\lambda t},\lambda t), \xx_{\lambda t}^\star, \lambda t) \cdot (\frac{1}{t - \lambda t})$\COMMENT{\textcolor{darkgreen}{Note that for $\lambda=0$, $\mf^{\mathrm{\xx}}(\mmF_{\mtheta^-}(\xx_{0},0), \xx_{0}^\star, 0) = \xx^\star$}}
        \ELSIF{$\lambda = 1$}
        \STATE Comupte $\xx_{t+\epsilon}^\star=\alpha(t+\epsilon)\cdot\zz^\star+\gamma(t+\epsilon)\cdot\xx^\star$ and $\xx_{t-\epsilon}=\alpha(t-\epsilon)\cdot\zz^\star+\gamma(t-\epsilon)\cdot\xx^\star$
        \STATE Let $\Delta\mf^{\mathrm{\xx}}_t= \mf^{\mathrm{\xx}}(\mmF_{\mtheta^-}(\xx_{t+\epsilon},t+\epsilon), \xx_{t+\epsilon}^\star, t+\epsilon) \cdot (\frac{1}{2\epsilon}) - \mf^{\mathrm{\xx}}(\mmF_{\mtheta^-}(\xx_{t-\epsilon},t-\epsilon), \xx_{t-\epsilon}^\star, t-\epsilon)\cdot (\frac{1}{2\epsilon})$
        \ENDIF
        \STATE Compute $\mmF_t^{\mathrm{target}} = \sg(\mmF_t) - \frac{4 \alpha(t)}{\alpha(t)\cdot \hat{\gamma}(t) - \hat{\alpha}(t) \cdot \gamma(t)} \cdot \frac{\mathrm{clip}(\Delta \mf^{\mathrm{\xx}}_t, -1, 1)}{\sin(t)}$
        \STATE Compute loss $\cL_t(\mtheta)=\cos(t) \left\| \mmF_t -  \mmF_t^{\mathrm{target}} \right\|_2^2$ and update $\mtheta \gets \mtheta - \eta \nabla_{\mtheta} \int_{0}^{1} \phi(t) \cL_t(\mtheta) \dm t$
        \UNTIL{Convergence}
    \end{algorithmic}
\end{algorithm*}

\paragraph{Unified training objective for both multi-step and few-step models.}

To facilitate the interpretation of our technical framework, we define two prediction functions based on model $\mmF_{\mtheta}$ as:
\begin{equation}
    \mf^{\mathrm{\xx}}(\mmF_t, \xx_t, t) := \frac{\alpha(t)\cdot\mmF_t - \hat{\alpha}(t)\cdot\xx_t}{\alpha(t)\cdot \hat{\gamma}(t) - \hat{\alpha}(t) \cdot \gamma(t)} \quad \text{\&} \quad \mf^{\mathrm{\zz}}(\mmF_t, \xx_t, t) := \frac{\hat{\gamma}(t)\cdot\xx_t - \gamma(t)\cdot\mmF_t}{\alpha(t)\cdot \hat{\gamma}(t) - \hat{\alpha}(t) \cdot \gamma(t)} \, ,
\end{equation}
where we define $\mmF_t := \mmF_{\mtheta}(\xx_t,t)$.
The training objective \eqref{eq:objective} can thus become (cf., \appref{app:lambda_to_0}):
\begin{equation}
    \label{eq:transformed_objective}
    \cL(\mtheta)=\EE{(\zz,\xx)\sim p(\zz,\xx),t}{\left[ \frac{1}{\hat{\omega}(t)}\| \mf^{\mathrm{\xx}}(\mmF_{\mtheta}(\xx_t,t), \xx_t, t) - \xx \|_2^2\right]} \, .
\end{equation}
To align with the gradient of our training objective \eqref{eq:objective}, we define a new weighting function $\hat{\omega}(t)$ in~\eqref{eq:transformed_objective} as
$
    \hat{\omega}(t) := \frac{\alpha(t)\cdot\alpha(t) \cdot \omega(t)}{\left(\alpha(t)\cdot \hat{\gamma}(t) - \hat{\alpha}(t) \cdot \gamma(t)\right)^2} \, .
$
To unify few-step models (such as consistency models) with multi-step models, we adopt a modified version of \eqref{eq:transformed_objective} by incorporating a consistency ratio $\lambda \in [0, 1]$:
\begin{equation}
    \label{eq:unified_objective}
    \cL(\mtheta)=\EEb{(\zz,\xx)\sim p(\zz,\xx),t}{ {\frac{1}{\hat{\omega}(t)} \norm{ \mf^{\mathrm{\xx}}(\mmF_{\mtheta}(\xx_t,t), \xx_t, t) - \mf^{\mathrm{\xx}}(\mmF_{\mtheta^-}(\xx_{\lambda t},\lambda t), \xx_{\lambda t}, \lambda t) }_2^2}} \, ,
\end{equation}
where consistency models and conventional multi-steps models are special cases within the context of \eqref{eq:unified_objective} (cf., \appref{app:lambda_to_0} and \appref{app:lambda_to_1}).
Specifically, setting $\lambda = 0$ yields diffusion and flow-matching models, while setting $\lambda \to 1 - \Delta t$ with $\Delta t \to 0$ recovers consistency models.
Following previous studies~\citep{song2023consistency}, we set $\hat{\omega}(t) = \frac{\tan(t)}{4}$.
As a result, the explicit minimization objective $\mathcal{L}(\mtheta)$ is given by:
\begin{equation}
    \label{eq:final_unified_objective}
    \EEb{(\zz,\xx)\sim p(\zz,\xx),t}{ {{\cos(t)} \norm{ \mmF_{\mtheta}(\xx_{t}, t) - \mmF_{\mtheta^-}(\xx_{t}, t) + \frac{4 \alpha(t) \Delta \mf^{\mathrm{\xx}}_t}{{\sin(t)} \cdot \left( \alpha(t)\cdot \hat{\gamma}(t) - \hat{\alpha}(t) \cdot \gamma(t) \right)} }_2^2} } \, ,
\end{equation}
where the detailed derivation from~\eqref{eq:unified_objective} to~\eqref{eq:final_unified_objective} is provided in~\appref{app:unified_training_objective}, and we define $\Delta \mf^{\mathrm{\xx}}_t$ in \eqref{eq:final_unified_objective} as
\begin{equation}
    \label{eq:delta_fx}
    \Delta \mf^{\mathrm{\xx}}_t := \frac{\mf^{\mathrm{\xx}}(\mmF_{\mtheta^-}(\xx_{t}, t), \xx_t, t) - \mf^{\mathrm{\xx}}(\mmF_{\mtheta^-}(\xx_{\lambda t},\lambda t), \xx_{\lambda t}, \lambda t)}{t - \lambda t} \, .
\end{equation}
However, optimizing the unified objective in \eqref{eq:final_unified_objective} presents a challenge: stabilizing the training process as $\lambda$ approaches 1. In this regime, the training dynamics resemble those of consistency models, known for unstable gradients, especially with BF16 precision~\citep{song2023consistency,lu2024simplifying}. To address this, we propose several stabilizing training techniques stated below.
\paragraph{Stabilizing gradient as $\lambda \to 1$.}
We identify that the instability in objective~\eqref{eq:final_unified_objective} primarily arises from numerical computational errors in the term $\Delta \mf^{\mathrm{\xx}}_t$, which subsequently affect the training target $\mmF_t^{\mathrm{target}}$.
Specifically, our theoretical analysis reveals that as $\lambda \to 1$, $\Delta \mf^{\mathrm{\xx}}_t$ approaches $\frac{\dm \mf^{\mathrm{\xx}}(\mmF_{\mtheta^-}(\xx_{t}, t), \xx_t, t)}{\dm t}$.
\eqref{eq:delta_fx} then serves as a first-order difference approximation of $\frac{\dm \mf^{\mathrm{\xx}}(\mmF_{\mtheta^-}(\xx_{t}, t), \xx_t, t)}{\dm t}$, which would become highly susceptible to numerical precision errors, primarily due to \textit{catastrophic cancellation}.
To mitigate this issue, we propose a second-order difference estimation technique by redefining $\Delta\mf^{\mathrm{\xx}}_t$ as
\begin{align*}
    \Delta\mf^{\mathrm{\xx}}_t = \frac{1}{2\epsilon} \left( \mf^{\mathrm{\xx}}(\mmF_{\mtheta^-}(\xx_{t+\epsilon}, t+\epsilon), \xx_{t+\epsilon}, t+\epsilon) - \mf^{\mathrm{\xx}}(\mmF_{\mtheta^-}(\xx_{t-\epsilon}, t-\epsilon), \xx_{t-\epsilon}, t-\epsilon) \right) \,.
\end{align*}
To further stabilize the training, we implement the following two strategies for $\Delta\mf^{\mathrm{\xx}}_t$:
\begin{enumerate}[label=(\alph*), nosep, leftmargin=16pt]
    \item We adopt a distributive reformulation of the second-difference term to prevent direct subtraction between nearly identical quantities, which can induce catastrophic cancellation, especially under limited numerical precision (e.g., BF16). Specifically, we factor out the shared scaling coefficient \(\frac{1}{2\epsilon}\), namely,
          $
              \Delta\mf^{\mathrm{\xx}}_t = \mf^{\mathrm{\xx}}(\mmF_{\mtheta^-}(\xx_{t+\epsilon}, t+\epsilon), \xx_{t+\epsilon}, t+\epsilon) \cdot \frac{1}{2\epsilon} - \mf^{\mathrm{\xx}}(\mmF_{\mtheta^-}(\xx_{t-\epsilon}, t-\epsilon), \xx_{t-\epsilon}, t-\epsilon) \cdot \frac{1}{2\epsilon}.
          $
          In this paper, we consistently set $\epsilon$ to $0.005$.
          See~\appref{app:derivative_estimation} for further analysis of this technique.

    \item We observe that applying numerical truncation~\citep{lu2024simplifying} to $\Delta\mf^{\mathrm{\xx}}_t$ enhances training stability. Specifically, we clip $\Delta \mf^{\mathrm{\xx}}_t$ to the range \([-1, 1]\), which prevents abnormal numerical outliers.
\end{enumerate}

\paragraph{Unified distribution transformation of time.}
Previous studies~\citep{yao2025reconstruction,esser2024scaling,song2023consistency,lu2024simplifying,karras2022elucidating,karras2024analyzing} employ non-linear functions to transform the time variable \( t \), initially sampled from a uniform distribution \( t \sim \mathcal{U}(0,1) \).
This transformation shifts the distribution of sampled times, effectively performing importance sampling and thereby accelerating the training convergence rate.
For example, the \texttt{lognorm} function \( f_{\mathrm{lognorm}}(t;\mu,\sigma) = \frac{1}{1 + \exp(-\mu - \sigma \cdot \Phi^{-1}(t))} \) is widely used~\citep{yao2025reconstruction,esser2024scaling}, where \( \Phi^{-1}(\cdot) \) denotes the inverse Cumulative Distribution Function (CDF) of the standard normal distribution.

In this work, we demonstrate that most commonly used non-linear time transformation functions can be effectively approximated by the regularized incomplete beta function:
$
    f_{\mathrm{Beta}}(t; a, b) = \nicefrac{ \int_0^t \tau^{a - 1}(1 - \tau)^{b - 1} \, \mathrm{d}\tau }{ \int_0^1 \tau^{a - 1}(1 - \tau)^{b - 1} \, \mathrm{d}\tau }
$, where a detailed analysis defers to \appref{app:beta_transformation}.
Consequently, we simplify the process by directly sampling time from a Beta distribution, i.e., \( t \sim \mathrm{Beta}(\theta_1, \theta_2) \), where \( \theta_1 \) and \( \theta_2 \) are parameters that control the shape of distribution (cf.,~\appref{app:impl_details} for their settings).
\looseness=-1

\paragraph{Learning enhanced target score function.}
Directly employing objective~\eqref{eq:final_unified_objective} to train models for estimating the conditional distribution $p(\xx|\cc)$ results in models incapable of generating realistic samples without Classifier-Free Guidance (CFG)~\citep{ho2022classifier}.
While enhancing semantic information, CFG approximately doubles the number of function evaluations, incurring significant computational overhead.
\looseness=-1

A recent work~\citep{tang2025diffusion} proposes modifying the target score function (see definition in~\citep{song2020score}) from $\nabla_{\xx_t}\log(p_t(\xx_t | \cc))$ to an enhanced version
$
    \nabla_{\xx_t}\log\left(p_t(\xx_t|\cc) \left(\nicefrac{p_{t,\mtheta}(\xx_t|\cc)}{p_{t,\mtheta}(\xx_t)}\right)^\zeta\right)
$,
where $\zeta \in (0,1)$ denotes the enhancement ratio.
By eliminating dependence on CFG, this approach enables high-fidelity sample generation with significantly reduced inference cost.
\looseness=-1

Inspired by this, we propose enhancing the target score function in a manner compatible with our unified training objective~\eqref{eq:final_unified_objective}. Specifically, we introduce a time-dependent enhancement strategy:

\begin{enumerate}[label=(\alph*), nosep, leftmargin=16pt]
    \item For $t \in [0, s]$, enhance $\xx$ and $\zz$ by applying
          $
              \xx^\star = \xx + \zeta \cdot \left(\mf^{\mathrm{\xx}}(\mmF_t, \xx_t, t) - \mf^{\mathrm{\xx}}(\mmF^\varnothing_t, \xx_t, t)\right)
          $,
          $
              \zz^\star = \zz + \zeta \cdot \left(\mf^{\mathrm{\zz}}(\mmF_t, \xx_t, t) - \mf^{\mathrm{\zz}}(\mmF^\varnothing_t, \xx_t, t)\right)
          $.
          Here, $\mmF^\varnothing_t = \mmF_{\mtheta^-}(\xx_t, t, \varnothing)$ and $\mmF_t = \mmF_{\mtheta^-}(\xx_t, t)$.

    \item For $t \in (s, 1]$, enhance $\xx$ and $\zz$ by applying
          $
              \xx^\star = \xx + \frac{1}{2} \left(\mf^{\mathrm{\xx}}(\mmF_t, \xx_t, t) - \xx\right)
          $ and
          $
              \zz^\star = \zz + \frac{1}{2} \left(\mf^{\mathrm{\zz}}(\mmF_t, \xx_t, t) - \zz\right)
          $. We consistently set $s = 0.75$ (cf., \appref{app:enhanced_target_score} for more analysis).
\end{enumerate}
An ablation study for this technique is shown in \secref{sec:ablation}, and the training process is shown in \algref{alg:udmm}.

\begin{algorithm*}[t]
    \caption{(\textbf{\method-S}). A Unified and Efficient Sampler for Few-step and Multi-step Continuous Generative Models (including Diffusion, Flow Matching, and Consistency Models)}
    \label{alg:uni_sampling}
    \begin{algorithmic}[1]
        \REQUIRE
        Initial $\tilde{\xx} \sim \mathcal{N}(\mathbf{0}, \mathbf{I})$,
        transport coefficients \{$\alpha(\cdot)$, $\gamma(\cdot)$, $\hat{\alpha}(\cdot)$, $\hat{\gamma}(\cdot)$\},
        trained model $\mmF_{\mtheta}$,
        sampling steps $N$,
        order $\nu \in \{1, 2\}$,
        time schedule $\mathcal{T}$,
        extrapolation ratio $\kappa$,
        stochastic ratio $\rho$.
        \ENSURE
        Final generated sample $\tilde{\xx} \sim p(\xx)$ and history samples $\{\hat{\xx}_i\}_{i=0}^N$ over generation process.
        \STATE Let $N \gets \lfloor (N+1) / 2 \rfloor$ if using second order sampling ($\nu=2$) \COMMENT{\textcolor{darkgreen}{Adjusts total steps to match first-order evaluation count}}
        \FOR{$i=0$ to $N-1$}
        \STATE Compute model output $ \mmF = \mmF_{\mtheta^-}(\tilde{\xx}, t_{i})$, and then $\hat{\xx}_i = \mf^{\mathrm{\xx}}(\mmF, \tilde{\xx}, t_{i})$ and $\hat{\zz}_i = \mf^{\mathrm{\zz}}(\mmF, \tilde{\xx}, t_{i})$
        \IF{$i\geq 1$}
        \STATE Compute extrapolated estimation $\hat{\zz} = \hat{\zz}_i + \kappa \cdot (\hat{\zz}_i - \hat{\zz}_{i-1})$ and $\hat{\xx} = \hat{\xx}_i + \kappa \cdot (\hat{\xx}_i - \hat{\xx}_{i-1})$
        \ENDIF
        \STATE Sample $\zz \sim \mathcal{N}(\mathbf{0}, \mathbf{I})$ \COMMENT{\textcolor{darkgreen}{An example choice of $\rho$ for performing SDE-similar sampling is: $\rho = \mathrm{clip}(\frac{\abs{t_{i}-t_{i+1}}\cdot2\alpha(t_{i})}{\alpha(t_{i+1})},0,1)$}}
        \STATE Compute estimated next time sample $\xx^{\prime} = \alpha(t_{i+1}) \cdot (\sqrt{1-\rho} \cdot \hat{\zz} + \sqrt{\rho} \cdot \zz ) + \gamma(t_{i+1}) \cdot \hat{\xx}$
        \IF{order $\nu=2$ \textbf{and} $i < N-1$}
        \STATE Compute prediction $\mmF^\prime = \mmF_{\mtheta}(\xx^{\prime}, t_{i+1})$, $\hat{\xx}^{\prime} = \mf^{\mathrm{\xx}}(\mmF^\prime, \xx^{\prime}, t_{i+1})$ and $\hat{\zz}^{\prime} = \mf^{\mathrm{\zz}}(\mmF^\prime, \xx^{\prime}, t_{i+1})$
        \STATE Compute corrected next time sample $\xx^{\prime} = \tilde{\xx} \cdot \frac{\gamma(t_{i+1})}{\gamma(t_{i})} + \left(\alpha(t_{i+1}) - \frac{\gamma(t_{i+1}) \alpha(t_{i})}{\gamma(t_{i})}\right) \cdot \frac{\hat{\xx} + \hat{\xx}^{\prime}}{2} $
        \ENDIF
        \STATE Reset $\tilde{\xx} \gets \xx^{\prime}$
        \ENDFOR

    \end{algorithmic}
\end{algorithm*}

\subsection{Unifying Sampling Process for Continuous Generative Models}
\label{sec:uni_sampling}
In this section, we introduce our unified sampling algorithm applicable to both consistency models and diffusion/flow-based models.

For classical iterative sampling models, such as a trained flow-matching model $\mf_{\mtheta}$, sampling from the learned distribution $p(\xx)$ involves solving the PF-ODE~\citep{song2020score}.
This process typically uses numerical ODE solvers, such as the Euler or Runge-Kutta methods~\citep{ma2024sit}, to iteratively transform the initial Gaussian noise $\tilde{\xx}$ into a sample from $p(\xx)$ by solving the ODE
(i.e.,
$
    \frac{\dm \tilde{\xx}_t}{\dm t} = \mf_{\mtheta}(\tilde{\xx}_t, t)
$),
Similarly, sampling processes in models like EDM~\citep{karras2022elucidating,karras2024analyzing} and consistency models~\citep{song2023consistency} involve a comparable gradual denoising procedure.
Building on these observations and our unified trainer \method-T, we first propose a general iterative sampling process with two stages, i.e., (a) and (b):
\begin{enumerate}[label=(\alph*), nosep, leftmargin=16pt]
    \item \textbf{Decomposition:} At time \(t\), the current input \(\tilde{\xx}_t\) is decomposed into two components:
          $
              \tilde{\xx}_t = \alpha(t) \cdot \hat{\zz}_t + \gamma(t) \cdot \hat{\xx}_t
          $.
          This decomposition uses the estimation model \(\mmF_{\mtheta}\).
          Specifically, the model output \(\mmF_t = \mmF_{\mtheta^-}(\tilde{\xx}_t, t)\) is computed, yielding the estimated clean component \(\hat{\xx}_t = \mf^{\mathrm{\xx}}(\mmF_t, \tilde{\xx}_t, t)\) and the estimated noise component \(\hat{\zz}_t = \mf^{\mathrm{\zz}}(\mmF_t, \tilde{\xx}_t, t)\).
    \item \textbf{Reconstruction:} The next time step's input, \(t'\), is generated by combining the estimated components:
          $
              \tilde{\xx}_{t'} = \alpha(t') \cdot \hat{\zz}_t + \gamma(t') \cdot \hat{\xx}_t
          $.
          The process then iterates to stage (a).
\end{enumerate}

We then introduce two enhancement techniques below to optimize the sampling process:

(i) \textbf{Extrapolating the estimation.}
Directly utilizing the estimated $\hat{\xx}_t$ and $\hat{\zz}_t$ to reconstruct the subsequent input $\tilde{\xx}_{t^\prime}$ can result in significant estimation errors, as the estimation model $\mmF_{\mtheta}$ does not perfectly align with the target function $\mmF^{\mathrm{target}}$ for solving the PF-ODE.

Note that CFG~\citep{ho2022classifier} guides a conditional model using an unconditional model, namely,
$
    \mf_{\mtheta}(\tilde{\xx}, t) = \mf_{\mtheta}(\tilde{\xx}, t) + \kappa \cdot \left(\mf_{\mtheta}^{\varnothing}(\tilde{\xx}, t) - \mf_{\mtheta}(\tilde{\xx}, t)\right)
$,
where $\kappa$ is the guidance ratio. This approach can be interpreted as leveraging a less accurate estimation to guide a more accurate one~\citep{karras2024guiding}.

Extending this insight, we propose to extrapolate the next time-step estimates $\hat{\xx}_{t^\prime}$ and $\hat{\zz}_{t^\prime}$ using the previous estimates $\hat{\xx}_t$ and $\hat{\zz}_t$, formulated as:
$
    \hat{\xx}_{t^\prime} \gets \hat{\xx}_{t^\prime} + \kappa \cdot (\hat{\xx}_{t^\prime} - \hat{\xx}_t)
$ and
$
    \hat{\zz}_{t^\prime} \gets \hat{\zz}_{t^\prime} + \kappa \cdot (\hat{\zz}_{t^\prime} - \hat{\zz}_t),
$
where $\kappa \in [0, 1]$ is the extrapolation ratio. This extrapolation process can significantly enhance sampling quality and reduce the number of sampling steps.
Notably, this technique is compatible with CFG and does not introduce additional computational overhead (see \secref{sec:SOTA_multi} for experimental details and \appref{app:extrapolating} for theoretical analysis).

(ii) \textbf{Incorporating stochasticity.}
During the aforementioned sampling process, the input $\tilde{\xx}_t$ is deterministic, potentially limiting the diversity of generated samples. To mitigate this, we introduce a stochastic term $\rho$ to $\tilde{\xx}_t$, defined as:
$
    \tilde{\xx}_{t^\prime} = \alpha(t^\prime) \cdot \left(\sqrt{1-\rho} \cdot \hat{\zz}_t + \sqrt{\rho} \cdot \zz\right) + \gamma(t^\prime) \cdot \hat{\xx}_t,
$
where $\zz \sim \mathcal{N}(\mathbf{0}, \mathbf{I})$ is a random noise vector, and $\rho$ is the stochasticity ratio. This stochastic term acts as a random perturbation to $\tilde{\xx}_t$, thereby enhancing the diversity of generated samples.

We find that setting $\rho=\lambda$ consistently yields optimal performance in terms of generation quality across all experiments, and we leave the analysis of this phenomenon for future research.
Furthermore, empirical investigation of $\kappa$ indicates that the range $[0.2, 0.6]$ is consistently beneficial (cf., \secref{sec:ablation} and \appref{app:impl_details}).
Model performance remains relatively stable within this range.

\paragraph{Unified sampling algorithm \method-S.}
Putting all these factors together, here we introduce a unified sampling algorithm applicable to consistency models and diffusion/flow-based models, as presented in~\algref{alg:uni_sampling}.
This framework demonstrates that classical samplers, such as the Euler sampler utilized for flow-matching models~\citep{ma2024sit}, constitute a special case of our \method-S (cf. \appref{app:unified_sampling} for analysis).
Extensive experiments (cf., \secref{sec:exp}) demonstrate two key features of this algorithm:
\begin{enumerate}[label=(\alph*), nosep, leftmargin=16pt]
    \item {Reduced computational resources:} It decreases the number of sampling steps required by existing models while maintaining or enhancing performance.
    \item {High compatibility:} It is compatible with existing models, irrespective of their training objectives or noise schedules, without necessitating modifications to model architectures or tuning.
\end{enumerate}

\section{Experiment}
\label{sec:exp}
\vspace{-0.25em}
This section details the experimental setup and evaluation of our proposed methodology, \methods.
Note that our approach relies on specific parameterizations of the transport coefficients $\alpha(\cdot)$, $\gamma(\cdot)$, $\hat{\alpha}(\cdot)$, and $\hat{\gamma}(\cdot)$, as detailed in \algref{alg:udmm} and \algref{alg:uni_sampling}.
Therefore, \tabref{tab:transport} summarizes the parameterizations used in experiments, including configurations for compatibility with prior methods.

\subsection{Experimental Setting} \label{sec:expset}
\vspace{-0.25em}
\paragraph{Datasets.}
We utilize ImageNet-1K~\citep{deng2009imagenet} at resolutions of $512 \times 512$ and $256 \times 256$ as our primary datasets, following prior studies~\citep{karras2024analyzing,song2023consistency} and adhering to ADM's data preprocessing protocols~\citep{dhariwal2021diffusion}. Additionally, CIFAR-10~\citep{krizhevsky2009cifar} at a resolution of $32 \times 32$ is employed for ablation studies.

For both $512 \times 512$ and $256 \times 256$ images, experiments are conducted using latent space generative modeling in line with previous works. Specifically:
(a) For $256 \times 256$ images, we employ multiple widely-used autoencoders, including SD-VAE~\citep{rombach2022high}, VA-VAE~\citep{yao2025reconstruction}, and E2E-VAE~\citep{leng2025repa}.
(b) For $512 \times 512$ images, a DC-AE (\textit{f32c32})~\citep{chen2024deep} with a higher compression rate is used to conserve computational resources. When utilizing SD-VAE for $512 \times 512$ images, a $2\times$ larger patch size is applied to maintain computational parity with the $256 \times 256$ setting.
Consequently, the computational burden for generating images at both $512 \times 512$ and $256 \times 256$ resolutions remains comparable across our trained models\footnote{Previous works often employed the same autoencoders and patch sizes for both resolutions, resulting in higher computational costs for generating $512 \times 512$ images. For example, the DiT-XL/2 model requires $524.60$ GFLOPs for $512 \times 512$ generation, in contrast to $118.64$ GFLOPs for $256 \times 256$.}.
Further details on datasets and autoencoders are provided in~\appref{app:exp_dataset}.

\paragraph{Neural network architectures.}
We evaluate \method-S sampling using models trained with established methodologies.
These models employ various architectures from two prevalent families commonly used in continuous generative models:
(a) Diffusion Transformers, including variants such as DiT~\citep{peebles2023scalable}, UViT~\citep{bao2023all}, SiT~\citep{ma2024sit}, Lightening-DiT~\citep{yao2025reconstruction}, and DDT~\citep{wang2025ddt}.
(b) UNet-based convolutional networks, including improved UNets~\citep{karras2022elucidating,song2020score} and EDM2-UNets~\citep{karras2024analyzing}.
For training models specifically for \method-T, we consistently utilize DiT as the backbone architecture. We train models of various sizes (B: 130M, L: 458M, XL: 675M parameters) and patch sizes.
Notation such as XL/2 denotes the XL model with a patch size of 2.
Following prior work~\citep{yao2025reconstruction,wang2025ddt}, minor architectural modifications are applied to enhance training stability (details in \appref{app:neural_architecture}).

\begin{table*}[t]
    \centering
    \caption{\small{\textbf{System-level quality comparison for multi-step generation task on class-conditional ImageNet-1K.} Notation A\rp B denotes the result obtained by combining methods A and B. \dg{}/\up{} indicate a decrease/increase, respectively, in the metric compared to the baseline performance of the pre-trained models.}}
    \vspace{-0.5em}
    \label{tab:in512and256_multistep}
    \resizebox{\textwidth}{!}{%
        \begin{tabular}{lcccc|lcccc}
            \toprule
            \multicolumn{5}{c|}{$512\times512$}     & \multicolumn{5}{c}{$256\times256$}                                                                                                                                                                                                                       \\
            \cmidrule(lr){1-5} \cmidrule(lr){6-10}
            \textbf{METHOD}                         & \textbf{NFE} ($\downarrow$)        & \textbf{FID} ($\downarrow$) & \textbf{\#Params} & \textbf{\#Epochs} & \textbf{METHOD}                           & \textbf{NFE} ($\downarrow$) & \textbf{FID} ($\downarrow$) & \textbf{\#Params} & \textbf{\#Epochs} \\
            \midrule
            \multicolumn{10}{c}{\textbf{Diffusion \& flow-matching Models}}                                                                                                                                                                                                                                    \\
            \midrule
            ADM-G~\citep{dhariwal2021diffusion}     & 250$\times$2                       & 7.72                        & 559M              & 388               & ADM-G~\citep{dhariwal2021diffusion}       & 250$\times$2                & 4.59                        & 559M              & 396               \\
            U-ViT-H/4~\citep{bao2023all}            & 50$\times$2                        & 4.05                        & 501M              & 400               & U-ViT-H/2~\citep{bao2023all}              & 50$\times$2                 & 2.29                        & 501M              & 400               \\
            DiT-XL/2~\citep{peebles2023scalable}    & 250$\times$2                       & 3.04                        & 675M              & 600               & DiT-XL/2~\citep{peebles2023scalable}      & 250$\times$2                & 2.27                        & 675M              & 1400              \\
            SiT-XL/2~\citep{ma2024sit}              & 250$\times$2                       & 2.62                        & 675M              & 600               & SiT-XL/2~\citep{ma2024sit}                & 250$\times$2                & 2.06                        & 675M              & 1400              \\
            MaskDiT~\citep{zheng2023fast}           & 79$\times$2                        & 2.50                        & 736M              & -                 & MDT~\citep{gao2023masked}                 & 250$\times$2                & 1.79                        & 675M              & 1300              \\
            EDM2-S~\citep{karras2024analyzing}      & 63                                 & 2.56                        & 280M              & 1678              & REPA-XL/2~\citep{yu2024representation}    & 250$\times$2                & 1.96                        & 675M              & 200               \\
            EDM2-L~\citep{karras2024analyzing}      & 63                                 & 2.06                        & 778M              & 1476              & REPA-XL/2~\citep{yu2024representation}    & 250$\times$2                & 1.42                        & 675M              & 800               \\
            EDM2-XXL~\citep{karras2024analyzing}    & 63                                 & 1.91                        & 1.5B              & 734               & Light.DiT~\citep{yao2025reconstruction}   & 250$\times$2                & 2.11                        & 675M              & 64                \\
            DiT-XL/1\rp\citep{chen2024deep}         & 250$\times$2                       & 2.41                        & 675M              & 400               & Light.DiT~\citep{yao2025reconstruction}   & 250$\times$2                & 1.35                        & 675M              & 800               \\
            U-ViT-H/1\rp\citep{chen2024deep}        & 30$\times$2                        & 2.53                        & 501M              & 400               & DDT-XL/2~\citep{wang2025ddt}              & 250$\times$2                & 1.31                        & 675M              & 256               \\
            REPA-XL/2~\citep{yu2024representation}  & 250$\times$2                       & 2.08                        & 675M              & 200               & DDT-XL/2~\citep{wang2025ddt}              & 250$\times$2                & 1.26                        & 675M              & 400               \\
            DDT-XL/2~\citep{wang2025ddt}            & 250$\times$2                       & \textbf{1.28}               & 675M              & -                 & REPA-E-XL~\citep{leng2025repa}            & 250$\times$2                & \textbf{1.26}               & 675M              & 800               \\
            \midrule
            \multicolumn{10}{c}{\textbf{GANs \& masked \& autoregressive models}}                                                                                                                                                                                                                              \\
            \midrule
            VQGAN\rp\citep{esser2021taming}         & 256                                & 18.65                       & 227M              & -                 & VQGAN\rp\citep{sun2024autoregressive}     & -                           & 2.18                        & 3.1B              & 300               \\
            MAGVIT-v2~\citep{yu2023language}        & 64$\times$2                        & 1.91                        & 307M              & 1080              & MAR-L~\citep{li2024autoregressive}        & 256$\times$2                & 1.78                        & 479M              & 800               \\
            MAR-L~\citep{li2024autoregressive}      & 256$\times$2                       & \textbf{1.73}               & 479M              & 800               & MAR-H~\citep{li2024autoregressive}        & 256$\times$2                & \textbf{1.55}               & 943M              & 800               \\
            VAR-$d$36-s~\citep{tian2024visual}      & 10$\times$2                        & 2.63                        & 2.3B              & 350               & VAR-$d$30-re~\citep{tian2024visual}       & 10$\times$2                 & 1.73                        & 2.0B              & 350               \\
            \midrule
            \multicolumn{10}{c}{\textbf{Ours: \method-S sampling with models trained by prior works}}                                                                                                                                                                                                          \\
            \midrule
            \method-S\rp\citep{karras2024analyzing} & 40\dg{23}                          & 2.53\dg{0.03}               & 280M              & -                 & \method-S\rp\citep{wang2025ddt}           & 100\dg{400}                 & 1.27\up{0.01}               & 675M              & -                 \\
            \method-S\rp\citep{karras2024analyzing} & 50\dg{13}                          & 2.04\dg{0.02}               & 778M              & -                 & \method-S\rp\citep{yao2025reconstruction} & 100\dg{400}                 & 1.21\dg{0.14}               & 675M              & -                 \\
            \method-S\rp\citep{karras2024analyzing} & 40\dg{23}                          & 1.88\dg{0.03}               & 1.5B              & -                 & \method-S\rp\citep{leng2025repa}          & 80\dg{420}                  & \textbf{1.06}\dg{0.20}      & 675M              & -                 \\
            \method-S\rp\citep{wang2025ddt}         & 200\dg{300}                        & \textbf{1.25}\dg{0.03}      & 675M              & -                 & \method-S\rp\citep{leng2025repa}          & 20\dg{480}                  & 2.00\up{0.74}               & 675M              & -                 \\
            \midrule
            \multicolumn{10}{c}{\textbf{Ours: models trained and sampled using \methods~(setting $\lambda=0$)}}                                                                                                                                                                                                \\
            \midrule
            \rp DC-AE~\citep{chen2024deep}          & 40                                 & \textbf{1.48}               & 675M              & 800               & \rp SD-VAE~\citep{rombach2022high}        & 60                          & 1.41                        & 675M              & 400               \\
            \rp DC-AE~\citep{chen2024deep}          & 20                                 & 1.68                        & 675M              & 800               & \rp VA-VAE~\citep{yao2025reconstruction}  & 60                          & 1.21                        & 675M              & 400               \\
            \rp SD-VAE~\citep{rombach2022high}      & 40                                 & 1.67                        & 675M              & 320               & \rp E2E-VAE~\citep{leng2025repa}          & 40                          & \textbf{1.21}               & 675M              & 800               \\
            \rp SD-VAE~\citep{rombach2022high}      & 20                                 & 1.80                        & 675M              & 320               & \rp E2E-VAE~\citep{leng2025repa}          & 20                          & 1.30                        & 675M              & 800               \\
            \bottomrule
        \end{tabular}%
    }
    \vspace{-1em}
\end{table*}

\begin{table*}[t]
    \centering
    \caption{\small{\textbf{System-level quality comparison for few-step generation task on class-conditional ImageNet-1K.}}}
    \vspace{-0.5em}
    \label{tab:in512and256_fewsteps}
    \resizebox{\textwidth}{!}{%
        \begin{tabular}{lcccc|lcccc}
            \toprule
            \multicolumn{5}{c|}{$512\times512$} & \multicolumn{5}{c}{$256\times256$}                                                                                                                                                                                                                      \\
            \cmidrule(lr){1-5} \cmidrule(lr){6-10}
            \textbf{METHOD}                     & \textbf{NFE} ($\downarrow$)        & \textbf{FID} ($\downarrow$) & \textbf{\#Params} & \textbf{\#Epochs} & \textbf{METHOD}                          & \textbf{NFE} ($\downarrow$) & \textbf{FID} ($\downarrow$) & \textbf{\#Params} & \textbf{\#Epochs} \\
            \midrule
            \multicolumn{10}{c}{\textbf{Consistency training \& distillation}}                                                                                                                                                                                                                            \\
            \midrule
            sCT-M~\citep{lu2024simplifying}     & 1                                  & 5.84                        & 498M              & 1837              & iCT~\citep{song2023improved}             & 2                           & 20.3                        & 675M              & -                 \\
                                                & 2                                  & 5.53                        & 498M              & 1837              & Shortcut-XL/2~\citep{frans2024one}       & 1                           & 10.6                        & 676M              & 250               \\
            sCT-L~\citep{lu2024simplifying}     & 1                                  & 5.15                        & 778M              & 1274              &                                          & 4                           & 7.80                        & 676M              & 250               \\
                                                & 2                                  & 4.65                        & 778M              & 1274              &                                          & 128                         & 3.80                        & 676M              & 250               \\
            sCT-XXL~\citep{lu2024simplifying}   & 1                                  & 4.29                        & 1.5B              & 762               & IMM-XL/2~\citep{zhou2025inductive}       & 1$\times$2                  & 7.77                        & 675M              & 3840              \\
                                                & 2                                  & 3.76                        & 1.5B              & 762               &                                          & 2$\times$2                  & 5.33                        & 675M              & 3840              \\
            sCD-M~\citep{lu2024simplifying}     & 1                                  & 2.75                        & 498M              & 1997              &                                          & 4$\times$2                  & 3.66                        & 675M              & 3840              \\
                                                & 2                                  & 2.26                        & 498M              & 1997              &                                          & 8$\times$2                  & 2.77                        & 675M              & 3840              \\
            sCD-L~\citep{lu2024simplifying}     & 1                                  & 2.55                        & 778M              & 1434              & IMM ($\omega=1.5$)                       & 1$\times$2                  & 8.05                        & 675M              & 3840              \\
                                                & 2                                  & 2.04                        & 778M              & 1434              &                                          & 2$\times$2                  & 3.99                        & 675M              & 3840              \\
            sCD-XXL~\citep{lu2024simplifying}   & 1                                  & 2.28                        & 1.5B              & 921               &                                          & 4$\times$2                  & 2.51                        & 675M              & 3840              \\
                                                & 2                                  & \textbf{1.88}               & 1.5B              & 921               &                                          & 8$\times$2                  & \textbf{1.99}               & 675M              & 3840              \\
            \midrule
            \multicolumn{10}{c}{\textbf{GANs \& masked \& autoregressive models}}                                                                                                                                                                                                                         \\
            \midrule
            BigGAN~\citep{brock2018large}       & 1                                  & 8.43                        & 160M              & -                 & BigGAN~\citep{brock2018large}            & 1                           & 6.95                        & 112M              & -                 \\
            StyleGAN~\citep{sauer2022stylegan}  & 1$\times$2                         & \textbf{2.41}               & 168M              & -                 & GigaGAN~\citep{kang2023scaling}          & 1                           & 3.45                        & 569M              & -                 \\
            MAGVIT-v2~\citep{yu2023language}    & 64$\times$2                        & 1.91                        & 307M              & 1080              & StyleGAN~\citep{sauer2022stylegan}       & 1$\times$2                  & \textbf{2.30}               & 166M              & -                 \\
            VAR-$d$36-s~\citep{tian2024visual}  & 10$\times$2                        & 2.63                        & 2.3B              & 350               & VAR-$d$30-re~\citep{tian2024visual}      & 10$\times$2                 & 1.73                        & 2.0B              & 350               \\
            \midrule
            \multicolumn{10}{c}{\textbf{Ours: models trained and sampled using \methods~(setting $\lambda=0$)}}                                                                                                                                                                                           \\
            \midrule
            \rp DC-AE~\citep{chen2024deep}      & 32                                 & \textbf{1.55}               & 675M              & 800               & \rp VA-VAE~\citep{yao2025reconstruction} & 16                          & 2.11                        & 675M              & 400               \\
            \rp DC-AE~\citep{chen2024deep}      & 16                                 & 1.81                        & 675M              & 800               & \rp VA-VAE~\citep{yao2025reconstruction} & 8                           & 6.09                        & 675M              & 400               \\
            \rp DC-AE~\citep{chen2024deep}      & 8                                  & 3.07                        & 675M              & 800               & \rp E2E-VAE~\citep{leng2025repa}         & 16                          & \textbf{1.40}               & 675M              & 800               \\
            \rp DC-AE~\citep{chen2024deep}      & 4                                  & 74.0                        & 675M              & 800               & \rp E2E-VAE~\citep{leng2025repa}         & 8                           & 2.68                        & 675M              & 800               \\
            \midrule
            \multicolumn{10}{c}{\textbf{Ours: models trained and sampled using \methods~(setting $\lambda=1$)}}                                                                                                                                                                                           \\
            \midrule
            \rp DC-AE~\citep{chen2024deep}      & 1                                  & 2.42                        & 675M              & 840               & \rp VA-VAE~\citep{yao2025reconstruction} & 2                           & \textbf{1.42}               & 675M              & 432               \\
            \rp DC-AE~\citep{chen2024deep}      & 2                                  & \textbf{1.75}               & 675M              & 840               & \rp VA-VAE~\citep{yao2025reconstruction} & 1                           & 2.19                        & 675M              & 432               \\
            \rp SD-VAE~\citep{rombach2022high}  & 1                                  & 2.63                        & 675M              & 360               & \rp SD-VAE~\citep{rombach2022high}       & 1                           & 2.10                        & 675M              & 424               \\
            \rp SD-VAE~\citep{rombach2022high}  & 2                                  & 2.11                        & 675M              & 360               & \rp E2E-VAE~\citep{leng2025repa}         & 1                           & 2.29                        & 675M              & 264               \\
            \bottomrule
        \end{tabular}%
    }
    \vspace{-1em}
\end{table*}

\paragraph{Implementation details.}
Our implementation is developed in PyTorch~\citep{paszke2019pytorch}.
Training employs AdamW~\citep{loshchilov2017decoupled} for multi-step sampling models. For few-step sampling models, RAdam~\citep{liu2019variance} is used to improve training stability.
Consistent with standard practice in generative modeling~\citep{yu2024representation,ma2024sit}, an exponential moving average (EMA) of model weights is maintained throughout training using a decay rate of $0.9999$.
All reported results utilize the EMA model.
Comprehensive hyperparameters and additional implementation details are provided in \appref{app:impl_details}.
Consistent with prior work \citep{song2020score,ho2020denoising,lipman2022flow,brock2018large}, we adopt standard evaluation protocols.
The primary metric for assessing image quality is the Fr\'echet Inception Distance (FID) \citep{heusel2017gans}, calculated on $50,000$ images (FID-$50\mathrm{K}$).

\subsection{Comparison with SOTA Methods for Multi-step Generation}
\label{sec:SOTA_multi}
Our experiments on ImageNet-1K at $512\!\times\!512$ and $256\!\times\!256$ resolutions systematically validate the three key advantages of \method: (1) {sampling acceleration} via \method-S on pre-trained models, (2) {ultra-efficient generation} with joint \method-T + \method-S, and (3) {broad compatibility}.

\paragraph{\method-S: Plug-and-play sampling acceleration without additional cost.}
\method-S provides free sampling acceleration for pre-trained generative models. It reduces the required Number of Function Evaluations (NFEs) while preserving or improving generation quality, as measured by FID.
Applied to $512 \times 512$ image generation, the approach demonstrates notable efficiency gains:
\begin{enumerate}[label=(\alph*), nosep, leftmargin=16pt]
    \item For the diffusion-based models, such as a pre-trained EDM2-XXL model, \method-S reduced NFEs from $63$ to $40$ (a $36.5\%$ reduction), concurrently improving FID from $1.91$ to $1.88$.
    \item When applied to the flow-based models, such as a pre-trained DDT-XL/2 model, \method-S achieved an FID of $1.25$ with $200$ NFEs, compared to the original $1.28$ FID requiring $500$ NFEs. This demonstrates a performance improvement achieved alongside enhanced efficiency.
\end{enumerate}
This approach generalizes across different generative model frameworks and resolutions. For instance, on $256 \times 256$ resolution using the flow-based {REPA-E-XL} model, \method-S attained {$1.06$ FID at $80$ NFEs}, which surpasses the baseline performance of $1.26$ FID achieved at $500$ NFEs.

In summary, \textit{\method-S acts as a broadly applicable technique for efficient sampling, demonstrating cases where performance (FID) improves despite a reduction in sampling steps}.
\vspace{-1.0em}
\paragraph{\method-T + \method-S: Synergistic efficiency.}
The combination of \method-T training and \method-S sampling yields highly competitive generative performance with minimal NFEs:
\begin{enumerate}[label=(\alph*), nosep, leftmargin=16pt]
    \item \textbf{$512 \times 512$}: With a DC-AE autoencoder, our framework achieved {$1.48$ FID at $40$ NFEs}.
          This outperforms DiT-XL/1\rp DC-AE ($2.41$ FID, $500$ NFEs) and EDM2-XXL ($1.91$ FID, $63$ NFEs), with comparable or reduced model size.
    \item \textbf{$256 \times 256$}: With an E2E-VAE autoencoder, we attained {$1.21$ FID at $40$ NFEs}. This result exceeds prior SOTA models like MAR-H ($1.55$ FID, $512$ NFEs) and REPA-E-XL ($1.26$ FID, $500$ NFEs).
\end{enumerate}

Importantly, models trained with \method-T maintain robustness under extremely low-step sampling regimes.
At {$20$ NFEs}, the $256 \times 256$ performance degrades gracefully to {$1.30$ FID}, a result that still exceeds the performance of several baseline models sampling with significantly higher NFEs.

In summary, \textit{the demonstrated robustness and efficiency of \methods across various scenarios underscore the high potential of our \method~for multi-step continuous generative modeling}.

\vspace{-0.5em}
\subsection{Comparison with SOTA Methods for Few-step Generation}
\label{sec:SOTA_few}
\vspace{-0.5em}

As evidenced by the results in \tabref{tab:in512and256_fewsteps}, our \methods framework exhibits superior performance across two key settings: $\lambda=0$, characteristic of a multi-step regime akin to diffusion and flow-matching models, and $\lambda=1$, indicative of a few-step regime resembling consistency models.
\vspace{-1.0em}
\paragraph{Few-step regime ($\lambda=1$).}
Configured for few-step generation, \methods achieves SOTA sample quality with minimal NFEs, surpassing existing specialized consistency models and GANs:
\begin{enumerate}[label=(\alph*), nosep, leftmargin=16pt]
    \item \textbf{$512\times512$}: Using a DC-AE autoencoder, our model achieves an FID of $1.75$ with $2$ NFEs and $675\mathrm{M}$ parameters. This outperforms sCD-XXL, a leading consistency distillation model, which reports $1.88$ FID with $2$ NFEs and $1.5\mathrm{B}$ parameters.
    \item \textbf{$256\times256$}: Using a VA-VAE autoencoder, our model achieves an FID of $1.42$ with $2$ NFEs. This is a notable improvement over IMM-XL/2, which obtains $1.99$ FID with $8\times2=16$ NFEs, demonstrating higher sample quality while requiring $8\times$ fewer sampling steps.
\end{enumerate}
In summary, \textit{these results demonstrate the capability of \methods to deliver high-quality generation with minimal sampling cost, which is advantageous for practical applications}.
\vspace{-0.5em}
\paragraph{Multi-step regime ($\lambda=0$).}
Even when models are trained for multi-step generation,
it nonetheless demonstrates competitive performance even when utilizing a moderate number of sampling steps.
\begin{enumerate}[label=(\alph*), nosep, leftmargin=16pt]
    \item \textbf{$512\times512$}: Using a DC-AE autoencoder, our model obtains an FID of $1.81$ with $16$ NFEs and $675\mathrm{M}$ parameters. This result is competitive with or superior to existing methods such as VAR-$d$30-s, which reports $2.63$ FID with $10\times2=20$ NFEs and $2.3\mathrm{B}$ parameters.
    \item \textbf{$256\times256$}: Using an E2E-VAE autoencoder, our model achieves an FID of $1.40$ with $16$ NFEs. This surpasses IMM-XL/2, which obtains $1.99$ FID with $8\times2=16$ NFEs, demonstrating improved quality at the same sampling cost.
\end{enumerate}

In summary, \textit{our \methods framework demonstrates versatility and high performance across both few-step ($\lambda=1$) and multi-step ($\lambda=0$) sampling regimes. As shown, it consistently achieves SOTA or competitive sample quality relative to existing methods, often requiring fewer sampling steps or parameters, which are important factors for efficient high-resolution image synthesis}.

\begin{figure}[t]
    \centering
    \begin{subfigure}[b]{0.325\textwidth}
        \centering
        \includegraphics[width=\linewidth]{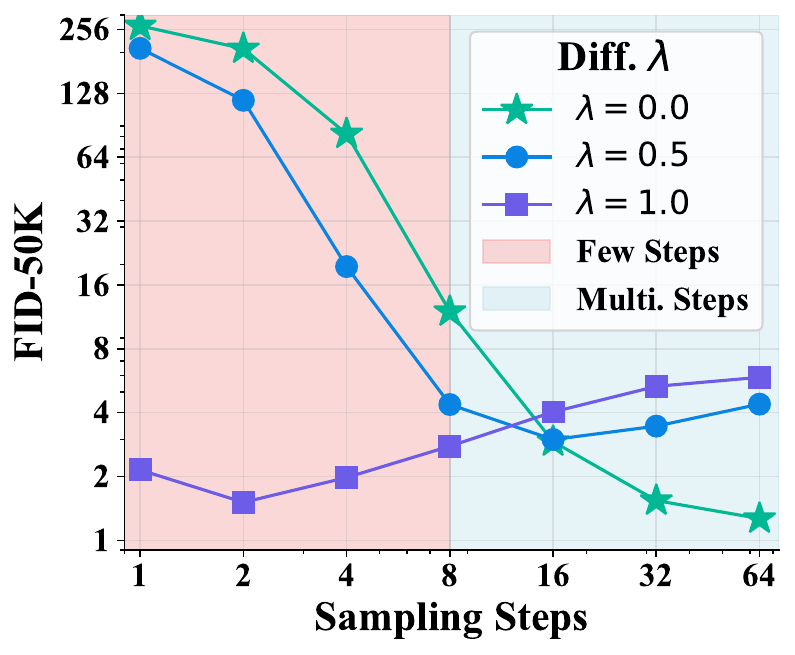}
        \caption{\textbf{Various $\lambda$ and sampling steps.}}
        \label{fig:diff_lambda}
    \end{subfigure}
    \hfill
    \begin{subfigure}[b]{0.33\textwidth}
        \centering
        \includegraphics[width=\linewidth]{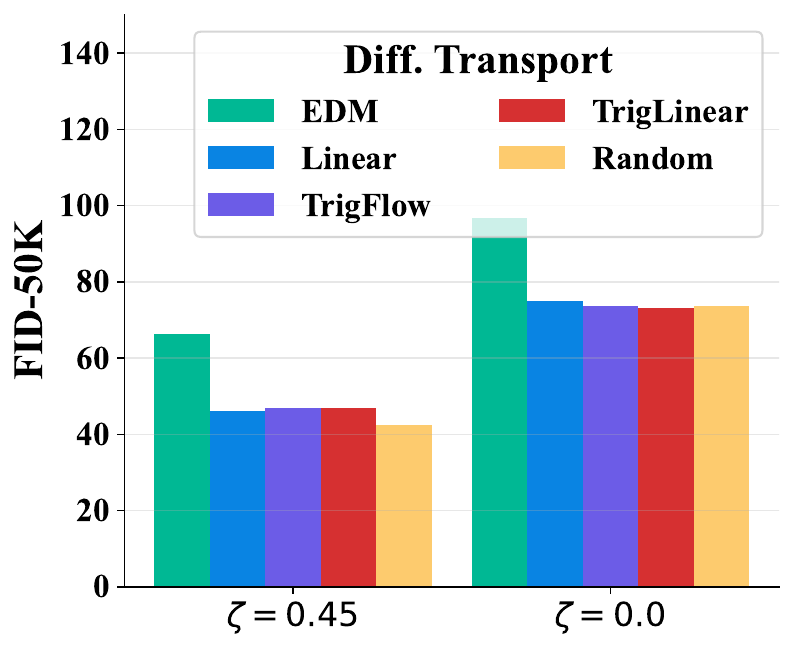}
        \caption{\textbf{Different $\zeta$ and transport types.}}
        \label{fig:diff_transport}
    \end{subfigure}
    \hfill
    \begin{subfigure}[b]{0.325\textwidth}
        \centering
        \includegraphics[width=\linewidth]{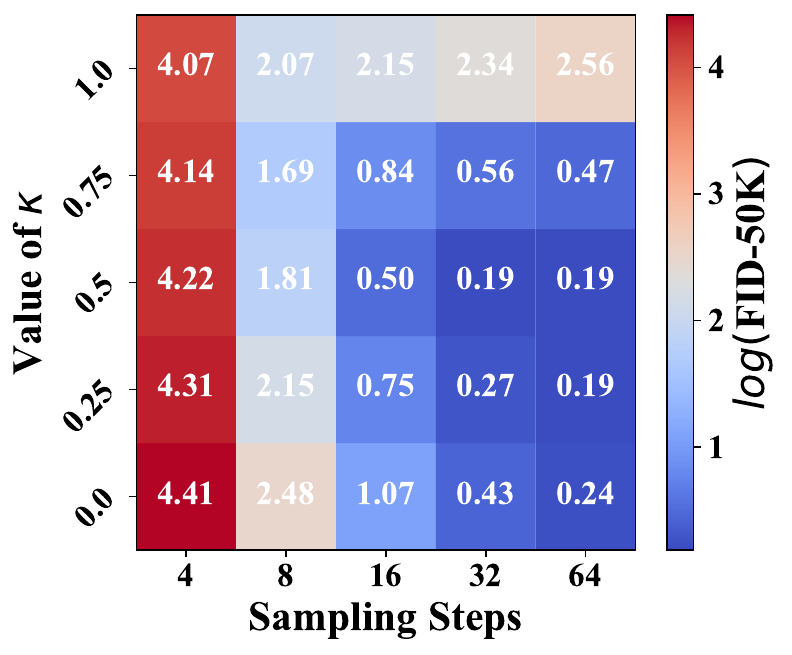}
        \caption{\textbf{Various $\kappa$ and sampling steps.}}
        \label{fig:diff_kappa}
    \end{subfigure}
    \vspace{-0.5em}
    \caption{\small{
            \textbf{Ablation studies of \method~on ImageNet-1K $256\!\times\!256$.}
            These studies evaluate key factors of the proposed \method. Ablations presented in (a) and (c) utilize XL/1 models with the VA-VAE autoencoder. For the results shown in (b), B/2 models with the SD-VAE autoencoder are used to facilitate more efficient training.
        }}
    \vspace{-1em}
\end{figure}

\vspace{-0.5em}
\subsection{Ablation Study over the Key Factors of \method}
\label{sec:ablation}
\vspace{-0.5em}
Unless otherwise specified, experiments in this section are conducted with $\kappa=0.0$ and $\lambda=0.0$.
\vspace{-0.5em}
\paragraph{Effect of $\lambda$ in \method-T.}
\figref{fig:diff_lambda} demonstrates that varying $\lambda$ influences the range of effective sampling steps for trained models. For instance, with $\lambda=1$, optimal performance is attained at $2$ sampling steps. In contrast, with $\lambda=0.5$, optimal performance is observed at $16$ steps.
\vspace{-0.5em}
\paragraph{Impact of $\zeta$ and transport type in \method.}
The results in \figref{fig:diff_transport} demonstrates that \methods is applicable with various transport types, albeit with some performance variation.
Investigating these performance differences constitutes future work.
The results also illustrate that the enhanced training objective (achieved with $\zeta=0.45$ compared to $\zeta=0.0$, per \secref{sec:methodology}) consistently improves performance across all tested transport types, underscoring the efficacy of this technique.
\vspace{-0.5em}
\paragraph{Setting different $\kappa$ in \method-S.}
Experimental results, depicted in \figref{fig:diff_kappa}, illustrate the impact of $\kappa$ on the trade-off between sampling steps and generation quality:
(a) High $\kappa$ values (e.g., $1.0$ and $0.75$) prove beneficial for extreme few-step sampling scenarios (e.g., $4$ steps);
(b) Moreover, mid-range $\kappa$ values ($0.25$ to $0.5$) achieve superior performance with fewer steps compared to $\kappa=0.0$.

\vspace{-0.5em}
\paragraph{Conclusion}
We introduce \method, a unified and efficient framework for the training and sampling of few-step and multi-step continuous generative models. Extensive experiments demonstrate \method achieves SOTA performance across various tasks, underscoring the efficacy of its constituent techniques. Additional experimental results and theoretical analysis are provided in \appref{app:detailed_exp} and \appref{app:theoretical_analysis}.
\looseness=-1

\bibliographystyle{configuration/ieeenat_fullname}
\bibliography{resources/reference}

\appendix

\onecolumn
{
    \hypersetup{linkcolor=black}
    \parskip=0em
    \renewcommand{\contentsname}{Contents}
    \tableofcontents
    \addtocontents{toc}{\protect\setcounter{tocdepth}{3}}
}

\newpage

\vspace{-0.5em}
\section{Broader Impacts}
\label{app:broader_impacts}
\vspace{-0.5em}
This paper proposes a unified implementation and theoretical framework for recent popular continuous generative models, such as diffusion models, flow matching models, and consistency models.
This work should provide positive impacts for the generative modeling community.
\vspace{-0.5em}
\section{Limitations}
\label{app:limitations}

\vspace{-0.5em}
\paragraph{Integration of training acceleration techniques.}
This work does not explore the integration of advanced training acceleration methods for diffusion models, such as REPA~\cite{yu2024representation}.

\vspace{-0.5em}
\paragraph{Exploration of downstream applications.}
The current study focuses on establishing the foundational framework. Comprehensive exploration of its application to complex downstream generative tasks, including text-to-image and text-to-video generation, is reserved for future research.

\vspace{-0.5em}
\section{Detailed Experiment}
\label{app:detailed_exp}
\vspace{-0.5em}
\subsection{Detailed Experimental Setting}
\label{app:exp_setting}
\vspace{-0.5em}
\subsubsection{Detailed Datasets}
\label{app:exp_dataset}

\paragraph{Image datasets.}
We conduct experiments on two datasets: {CIFAR-10}~\citep{krizhevsky2009learning}, {ImageNet-1K}~\citep{deng2009imagenet}:
\begin{enumerate}[label=(\alph*), nosep, leftmargin=16pt]
    \item {CIFAR-10} is a widely used benchmark dataset for image classification and generation tasks. It consists of $60,000$ color images, each with a resolution of $32 \times 32$ pixels, categorized into $10$ distinct classes. The dataset is divided into $50,000$ training images and $10,000$ test images.
    \item {ImageNet-1K} is a large-scale dataset containing over $1.2$ million high-resolution images across $1,000$ categories.
\end{enumerate}

\paragraph{Latent space datasets.}
However, directly training diffusion transformers in the pixel space is computationally expensive and inefficient.
Therefore, following previous studies~\citep{yu2024representation,ma2024sit}, we train our diffusion transformers in latent space instead.
\tabref{tab:vae_comparison} presents a comparative analysis of various Variational Autoencoder (VAE) architectures. SD-VAE is characterized by a higher spatial resolution in its latent representation (e.g., $\mathrm{H}/8 \times \mathrm{W}/8$) combined with a lower channel capacity ($4$ channels).
Conversely, alternative models such as VA-VAE, E2E-VAE, and DC-AE achieve more significant spatial compression (e.g., $\mathrm{H}/16 \times \mathrm{W}/16$ or $\mathrm{H}/32 \times \mathrm{W}/32$) at the expense of an increased channel depth (typically $32$ channels).

A key consideration is that the computational cost of a diffusion transformer subsequently processing these latent representations is primarily dictated by their spatial dimensions, rather than their channel capacity~\citep{chen2024deep}.
Specifically, if the latent map is processed by a transformer by dividing it into non-overlapping patches, the cost is proportional to the number of these patches. This quantity is given by $(\mathrm{H} / \text{Compression Ratio} / \text{Patch Size}) \times (\mathrm{W} / \text{Compression Ratio} / \text{Patch Size})$.
Here, $\mathrm{H}$ and $\mathrm{W}$ are the input image dimensions, $\text{Compression Ratio}$ refers to the spatial compression factor of the VAE (e.g., $8$, $16$, $32$ as detailed in \tabref{tab:vae_comparison}), and $\text{Patch Size}$ denotes the side length of the patches processed by the transformer.

\begin{table}[h]
    \centering
    \caption{\textbf{Comparison of different VAE architectures in terms of latent space dimensions and channel capacity.} The table contrasts four variational autoencoder variants (SD-VAE, VA-VAE, E2E-VAE, and DC-AE) by their spatial compression ratios (latent size) and feature channel dimensions.
        Here, $\mathrm{H}$ and $\mathrm{W}$ denote input image height and width (e.g., $256 \times 256$ or $512 \times 512$), respectively.}
    \label{tab:vae_comparison}
    \resizebox{\textwidth}{!}{
        \begin{tabular}{@{}c|c|c|c|c@{}}
            \toprule
                        & SD-VAE (both \texttt{ema} and \texttt{mse} versions)~\citep{rombach2022high} & VA-VAE~\citep{yao2025reconstruction}          & E2E-VAE~\citep{leng2025repa}                  & DC-AE (\textit{f32c32})~\citep{chen2024deep}  \\ \midrule
            Latent Size & $(\mathrm{H} / 8) \times (\mathrm{W} / 8) $                                  & $(\mathrm{H} / 16) \times (\mathrm{W} / 16) $ & $(\mathrm{H} / 16) \times (\mathrm{W} / 16) $ & $(\mathrm{H} / 32) \times (\mathrm{W} / 32) $ \\
            Channels    & $4$                                                                          & $32$                                          & $32$                                          & 32                                            \\ \bottomrule
        \end{tabular}}
\end{table}

\subsubsection{Detailed Neural Architecture}
\label{app:neural_architecture}

Diffusion Transformers (DiTs) represent a paradigm shift in generative modeling by replacing the traditional U-Net backbone with a Transformer-based architecture.
Proposed by \textit{Scalable Diffusion Models with Transformers} \cite{peebles2023scalable}, DiTs exhibit superior scalability and performance in image generation tasks.
In this paper, we utilize three key variants—DiT-B (130M parameters), DiT-L (458M parameters), and DiT-XL (675M parameters).

To improve training stability, informed by recent studies \citep{yao2025reconstruction,wang2025ddt}, we incorporate several architectural modifications into the DiT model: (a) SwiGLU feed-forward networks (FFN) \citep{shazeer2020glu}; (b) RMSNorm \citep{zhang2019root} without learnable affine parameters; (c) Rotary Positional Embeddings (RoPE) \citep{su2024roformer}; and (d) parameter-free RMSNorm applied to Key (K) and Query (Q) projections in self-attention layers \citep{vaswani2017attention}.

\subsubsection{Detailed Implementation Details}
\label{app:impl_details}

Experiments were conducted on a cluster equipped with $8$ H800 GPUs, each with $80$ GB of VRAM.

\paragraph{Hyperparameter configuration.}
Detailed hyperparameter configurations are provided in~\tabref{tab:hyper_param} to ensure reproducibility.
The design of time schedules for sampling processes varies in complexity.
For few-step models, typically employing 1 or 2 sampling steps, manual schedule design is straightforward.
However, the time schedule $\mathcal{T}$ utilized by our \method-S often comprises a large number of time points, particularly for a large number of sampling steps $N$.
Manual design of such dense schedules is challenging and can limit the achievable performance of our \methods, as prior work~\citep{yao2025reconstruction,wang2025ddt} has established that carefully designed schedules significantly enhance multi-step models, including flow-matching variants.
To address this, we propose transforming each time point $t \in \mathcal{T}$ using a generalized Kumaraswamy transformation: $f_{\textrm{Kuma}}(t;a,b,c) = (1 - (1 - t^a)^b)^c$.
This choice is motivated by the common practice in prior studies of applying non-linear transformations to individual time points to construct effective schedules.
A specific instance of such a transformation is the \texttt{timeshift} function $f_\textrm{shift}(t;s) = \frac{st}{1 + (s - 1)t}$, where $s > 0$ \citep{yao2025reconstruction}.
We find that the Kumaraswamy transformation, by appropriate selection of parameters $a,b,c$, can effectively approximate $f_\textrm{shift}$ and other widely-used functions (cf., \appref{app:kuma_transformation}), including the identity function $f(t)=t$~\citep{yu2024representation,leng2025repa}.
Empirical evaluations suggest that the parameter configuration $(a, b, c) = (1.17, 0.8, 1.1)$ yields robust performance across diverse scenarios, corresponding to the "Auto" setting in \tabref{tab:hyper_param}.

\paragraph{Detailed implementation techniques of enhancing target score function.}
We enhance the target score function for conditional diffusion models by modifying the standard score $\nabla_{\xx_t}\log p_t(\xx_t | \cc)$ \citep{song2020score} to an enhanced version derived from the density $p_t(\xx_t|\cc) \left(\nicefrac{p_{t,\mtheta}(\xx_t|\cc)}{p_{t,\mtheta}(\xx_t)}\right)^\zeta$. This corresponds to a target score of $\nabla_{\xx_t}\log p_t(\xx_t|\cc) + \zeta \left(\nabla_{\xx_t} \log p_{t,\mtheta}(\xx_t|\cc) - \nabla_{\xx_t} \log p_{t,\mtheta}(\xx_t)\right)$. The objective is to guide the learning process towards distributions that yield higher quality conditional samples.

Accurate estimation of the model probabilities $p_{t,\mtheta}$ is crucial for the effectiveness of this enhancement. We find that using parameters from an Exponential Moving Average (EMA) of the model during training improves the stability and quality of these estimates, resulting better $\xx^\star$ and $\zz^\star$ in \algref{alg:udmm}.

When training few-step models, direct computation of the enhanced target score gradient typically requires evaluating the model with and without conditioning (for the $p_{t,\mtheta}$ terms), incurring additional computational cost. To address this, we propose an efficient approximation that leverages a well-pre-trained multi-step model, denoted by parameters $\mtheta^\star$. Instead of computing the score gradient explicitly, the updates for the variables $\xx^\star$ and $\zz^\star$ (as used in \algref{alg:udmm}) are calculated based on features or outputs derived from a single forward pass of the pre-trained model $\mtheta^\star$.

Specifically, we compute $\mmF_t = \mmF_{\mtheta^\star}(\xx_t, t)$, representing features extracted by the pre-trained model $\mtheta^\star$ at time $t$ given input $\xx_t$. The enhanced updates $\xx^\star$ and $\zz^\star$ are then computed as follows:
\begin{enumerate}[label=(\alph*), nosep, leftmargin=16pt]
    \item For $t \in [0, s]$, the updates are:
          $
              \xx^\star \gets \xx + \zeta \cdot \left(\mf^{\mathrm{\xx}}(\mmF_t, \xx_t, t) - \xx\right)
          $,
          $
              \zz^\star \gets \zz + \zeta \cdot \left(\mf^{\mathrm{\zz}}(\mmF_t, \xx_t, t) - \zz\right)
          $.
    \item For $t \in (s, 1]$, the updates are:
          $
              \xx^\star \gets \xx + \frac{1}{2} \left(\mf^{\mathrm{\xx}}(\mmF_t, \xx_t, t) - \xx\right)
          $ and
          $
              \zz^\star \gets \zz + \frac{1}{2} \left(\mf^{\mathrm{\zz}}(\mmF_t, \xx_t, t) - \zz\right)
          $.
\end{enumerate}
We consistently set the time threshold $s = 0.75$. This approach allows us to incorporate the guidance from the enhanced target signal with the computational cost equivalent to a single forward evaluation of the pre-trained model $\mtheta^\star$ per step.
The enhancement ratio $\zeta$ is constrained to $[0, \infty)$ in this case.

\begin{table}[t]
    \caption{\small{
            \textbf{Hyperparameter configurations for \methods training and sampling on ImageNet-1K.}
            We maintain a consistent batch size of 1024 across all experiments. Training durations (epoch counts) are provided in other tables throughout the paper. The table specifies optimizer choices, learning rates, and key parameters for both \method-T and \method-S variants across different model architectures and datasets.}
    }
    \label{tab:hyper_param}
    \resizebox{\textwidth}{!}{
        \begin{tabular}{@{}cccccccccccccc@{}}
            \toprule
            \multicolumn{3}{c|}{Task} & \multicolumn{3}{c|}{Optimizer} & \multicolumn{4}{c|}{\method-T} & \multicolumn{4}{c}{\method-S}                                                                                                                                                                                \\ \midrule
            Resolution                & VAE/AE                         & \multicolumn{1}{c|}{Model}     & Type                          & lr     & \multicolumn{1}{c|}{($\beta_1$,$\beta_2$)} & Transport & ($\theta_1$,$\theta_2$) & $\lambda$ & \multicolumn{1}{c|}{$\zeta$} & $\rho$ & $\kappa$ & $\cT$     & $\nu$ \\ \midrule
            \multicolumn{14}{c}{Multi-step model training and sampling}                                                                                                                                                                                                                                                \\ \midrule
                                      & E2E-VAE                        & \multicolumn{1}{c|}{XL/1}      & AdamW                         & 0.0002 & \multicolumn{1}{c|}{(0.9,0.95)}            & Linear    & (1.0,1.0)               & 0         & \multicolumn{1}{c|}{0.67}    & 0      & 0.5      & Auto      & 1     \\
            256                       & SD-VAE                         & \multicolumn{1}{c|}{XL/2}      & AdamW                         & 0.0002 & \multicolumn{1}{c|}{(0.9,0.95)}            & Linear    & (2.4,2.4)               & 0         & \multicolumn{1}{c|}{0.44}    & 0      & 0.21     & Auto      & 1     \\
                                      & VA-VAE                         & \multicolumn{1}{c|}{XL/1}      & AdamW                         & 0.0002 & \multicolumn{1}{c|}{(0.9,0.95)}            & Linear    & (1.0,1.0)               & 0         & \multicolumn{1}{c|}{0.47}    & 0      & 0.5      & Auto      & 1     \\ \midrule
            512                       & DC-AE                          & \multicolumn{1}{c|}{XL/1}      & AdamW                         & 0.0002 & \multicolumn{1}{c|}{(0.9,0.95)}            & Linear    & (1.0,1.0)               & 0         & \multicolumn{1}{c|}{0.57}    & 0      & 0.46     & Auto      & 1     \\
                                      & SD-VAE                         & \multicolumn{1}{c|}{XL/4}      & AdamW                         & 0.0002 & \multicolumn{1}{c|}{(0.9,0.95)}            & Linear    & (2.4,2.4)               & 0         & \multicolumn{1}{c|}{0.60}    & 0      & 0.4      & Auto      & 1     \\ \midrule
            \multicolumn{14}{c}{Few-step model training and sampling}                                                                                                                                                                                                                                                  \\ \midrule
                                      & E2E-VAE                        & \multicolumn{1}{c|}{XL/1}      & RAdam                         & 0.0001 & \multicolumn{1}{c|}{(0.9,0.999)}           & Linear    & (0.8,1.0)               & 1         & \multicolumn{1}{c|}{1.3}     & 1      & 0        & \{1,0.5\} & 1     \\
            256                       & SD-VAE                         & \multicolumn{1}{c|}{XL/2}      & RAdam                         & 0.0001 & \multicolumn{1}{c|}{(0.9,0.999)}           & Linear    & (0.8,1.0)               & 1         & \multicolumn{1}{c|}{2.0}     & 1      & 0        & \{1,0.3\} & 1     \\
                                      & VA-VAE                         & \multicolumn{1}{c|}{XL/2}      & RAdam                         & 0.0001 & \multicolumn{1}{c|}{(0.9,0.999)}           & Linear    & (0.8,1.0)               & 1         & \multicolumn{1}{c|}{2.0}     & 1      & 0        & \{1,0.3\} & 1     \\ \midrule
            512                       & DC-AE                          & \multicolumn{1}{c|}{XL/1}      & RAdam                         & 0.0001 & \multicolumn{1}{c|}{(0.9,0.999)}           & Linear    & (0.8,1.0)               & 1         & \multicolumn{1}{c|}{1.5}     & 1      & 0        & \{1,0.6\} & 1     \\
                                      & SD-VAE                         & \multicolumn{1}{c|}{XL/4}      & RAdam                         & 0.0001 & \multicolumn{1}{c|}{(0.9,0.999)}           & Linear    & (0.8,1.0)               & 1         & \multicolumn{1}{c|}{1.5}     & 1      & 0        & \{1,0.5\} & 1     \\ \bottomrule
        \end{tabular}}
    \vspace{-1em}
\end{table}

\begin{table}[t]
    \centering
    \caption{\textbf{Comparison of different transport types employed during the sampling and training phases of our \methods.} ``TrigLinear'' and ``Random'' are introduced herein specifically for ablation studies. ``TrigLinear'' is constructed by combining the transport coefficients of ``Linear'' and ``TrigFlow''. ``Random'' represents a randomly designed transport type used to demonstrate the generality of our \method. Other transport types are adapted from existing methods and transformed into the transport coefficient representation used by \method.}
    \label{tab:transport}
    \setlength\tabcolsep{11.5pt}
    \resizebox{\textwidth}{!}{
        \begin{tabular}{@{}c|c|c|c|c|c|c@{}}
            \toprule
                              & Linear                                 & ReLinear                                  & TrigFlow                               & EDM ($\sigma(t)=e^{4\cdot(2.68t-1.59)}$)                              & TrigLinear                  & Random                      \\ \midrule
            $\alpha(t)$       & $t$                                    & $1-t$                                     & $\sin(t\cdot\frac{\pi}{2})$            & $\nicefrac{\sigma(t)}{\sqrt{\sigma^2(t)+0.25}}$                       & $\sin(t\cdot\frac{\pi}{2})$ & $\sin(t\cdot\frac{\pi}{2})$ \\
            $\gamma(t)$       & $1-t$                                  & $t$                                       & $\cos(t\cdot\frac{\pi}{2})$            & $\nicefrac{1}{\sqrt{\sigma^2(t)+0.25}}$                               & $\cos(t\cdot\frac{\pi}{2})$ & $1-t$                       \\
            $\hat{\alpha}(t)$ & $1$                                    & $-1$                                      & $\cos(t\cdot\frac{\pi}{2})$            & $\nicefrac{-0.5}{\sqrt{\sigma^2(t)+0.25}}$                            & $1$                         & $1$                         \\
            $\hat{\gamma}(t)$ & $-1$                                   & $1$                                       & $-\sin(t\cdot\frac{\pi}{2})$           & $\nicefrac{2\sigma(t)}{\sqrt{\sigma^2(t)+0.25}}$                      & $-1$                        & $-1 - e^{-5t}$              \\ \midrule
            e.g.,             & \citep{ma2024sit,yu2024representation} & \citep{yao2025reconstruction,wang2025ddt} & \citep{lu2024simplifying,chen2025sana} & \citep{song2023consistency,karras2024analyzing,karras2022elucidating} & N/A                         & N/A                         \\ \bottomrule
        \end{tabular}}
    \vspace{-1em}
\end{table}

\paragraph{Baselines.}
We compare our approach against several SOTA continuous and discrete generative models.
We broadly categorize these baselines by their generation process:
\begin{enumerate}[label=(\alph*), nosep, leftmargin=16pt]
    \item Multi-step models. These methods typically synthesize data through a sequence of steps. We include various diffusion models, encompassing classical formulations like DDPM and score-based models \citep{song2020denoising,ho2020denoising}, and advanced variants focusing on improved sampling or performance in latent spaces \citep{dhariwal2021diffusion,karras2022elucidating,peebles2023scalable,zheng2023fast,bao2023all}. We also consider flow-matching models \citep{lipman2022flow}, which leverage continuous normalizing flows and demonstrate favorable training properties, along with subsequent scaling efforts \citep{ma2024sit,yu2024representation,yao2025reconstruction}.
          Additionally, we also include autoregressive models \citep{li2024autoregressive,tian2024visual,yu2023language} as the baselines, which generate data sequentially, often in discrete domains.
    \item Few-step models. These models are designed for efficient, often single-step or few-step, generation. This category includes generative adversarial networks \citep{goodfellow2020generative}, which achieve efficient one-step synthesis through adversarial training, and their large-scale variants \citep{brock2018large,sauer2022stylegan,kang2023scaling}. We also evaluate consistency models \citep{song2023consistency}, proposed for high-quality generation adaptable to few sampling steps, and subsequent techniques aimed at improving their stability and scalability \citep{song2023improved,lu2024simplifying,zhou2025inductive}.
\end{enumerate}

Crucially, we demonstrate the compatibility of \method-S with models pre-trained using these methods.
We show how these models can be represented within the \method framework by defining the functions $\alpha(\cdot)$, $\gamma(\cdot)$, $\hat{\alpha}(\cdot)$, and $\hat{\gamma}(\cdot)$. Detailed parameterizations are provided in~\tabref{tab:transport}, with guidance for their specification presented in~\appref{app:cal_transport}.

\vspace{-0.5em}
\subsection{Experimental Results on Small Datasets}

Since most existing few-step generation methods~\citep{song2023consistency,geng2024consistency} are limited to training models on low-resolution, small-scale datasets like CIFAR-10~\citep{krizhevsky2009learning}, we conduct our comparative experiments on CIFAR-10 to ensure fair comparison.
To demonstrate the versatility of our \method, we employ both the "EDM" transport (see \tabref{tab:transport} for definition) and the standard $56\mathrm{M}$-parameter UNet architecture, following established practices in prior work~\cite{song2023consistency,geng2024consistency}.
\begin{table*}[h!]
    \centering
    \caption{\small{\textbf{System-level quality comparison for few-step generation task on unconditional CIFAR-10 ($32\times32$).}}}
    \label{tab:cf10}
    \resizebox{\textwidth}{!}{%
        \begin{tabular}{@{}l|ccccccccccccccc@{}}
            \toprule
            \textbf{Metric}             & PD~\citep{salimans2022progressive} & 2-RF~\citep{liu2022flow} & DMD~\citep{yin2024one} & CD~\citep{song2023consistency} & sCD~\citep{lu2024simplifying} & \multicolumn{2}{c}{iCT~\citep{song2023improved}} & \multicolumn{2}{c}{ECT~\citep{geng2024consistency}} & \multicolumn{2}{c}{sCT~\citep{lu2024simplifying}} & \multicolumn{2}{c}{{IMM}~\citep{zhou2025inductive}} & \multicolumn{2}{c}{\textbf{\method}}                                               \\ \midrule
            \textbf{FID} ($\downarrow$) & 4.51                               & 4.85                     & 3.77                   & 2.93                           & 2.52                          & 2.83                                             & 2.46                                                & 3.60                                              & 2.11                                                & 2.97                                 & 2.06 & 3.20 & {1.98} & \textbf{2.82} & 2.17 \\
            \textbf{NFE} ($\downarrow$) & 2                                  & 1                        & 1                      & 2                              & 2                             & 1                                                & 2                                                   & 1                                                 & 2                                                   & 1                                    & 2    & 1    & 2      & 1             & 2    \\ \bottomrule
        \end{tabular}}
    \vspace{-0.6em}
\end{table*}
As shown in \tabref{tab:cf10}, our \method achieves SOTA performance with just 1 NFE (Neural Function Evaluation) while maintaining competitive results for 2 NFEs.
These results underscore \method's robust compatibility across diverse datasets, network architectures, and transport types.

\newpage

\subsection{Detailed Comparison with SOTA Methods for Multi-step Generation}
\label{app:sota_multi}

\begin{table*}[h!]
    \centering
    \caption{\small{\textbf{System-level quality comparison for multi-step generation task on class-conditional ImageNet-1K.} Notation A\rp B denotes the result obtained by combining methods A and B. \dg{}/\up{} indicate a decrease/increase, respectively, in the metric compared to the baseline performance of the pre-trained models.}}
    \label{tab:in512_multistep}
    \resizebox{\textwidth}{!}{%
        \begin{tabular}{lccccccc|ccc}
            \toprule
            \textbf{METHOD}                         & \textbf{VAE/AE}                      & \textbf{Patch Size} & \textbf{Activation Size} & \textbf{NFE} ($\downarrow$) & \textbf{FID} ($\downarrow$) & \textbf{IS} ($\uparrow$) & \textbf{\#Params} & \textbf{\#Epochs} \\
            \midrule
            \multicolumn{9}{c}{\textbf{ $512 \times 512$}}                                                                                                                                                                                                                 \\
            \midrule
            \multicolumn{9}{c}{\textbf{Diffusion \& flow-matching models}}                                                                                                                                                                                                 \\
            \midrule
            ADM-G~\citep{dhariwal2021diffusion}     & -                                    & -                   & -                        & 250$\times$2                & 7.72                        & 172.71                   & 559M              & 388               \\
            U-ViT-H/4~\citep{bao2023all}            & SD-VAE~\citep{rombach2022high}       & 4                   & 16$\times$16             & 50$\times$2                 & 4.05                        & 263.79                   & 501M              & 400               \\
            DiT-XL/2~\citep{peebles2023scalable}    & SD-VAE~\citep{rombach2022high}       & 2                   & 32$\times$32             & 250$\times$2                & 3.04                        & 240.82                   & 675M              & 600               \\
            SiT-XL/2~\citep{ma2024sit}              & SD-VAE~\citep{rombach2022high}       & 2                   & 32$\times$32             & 250$\times$2                & 2.62                        & 252.21                   & 675M              & 600               \\
            MaskDiT~\citep{zheng2023fast}           & SD-VAE~\citep{rombach2022high}       & 2                   & 32$\times$32             & 79$\times$2                 & 2.50                        & 256.27                   & 736M              & -                 \\
            EDM2-S~\citep{karras2024analyzing}      & SD-VAE~\citep{rombach2022high}       & -                   & -                        & 63                          & 2.56                        & -                        & 280M              & 1678              \\
            EDM2-L~\citep{karras2024analyzing}      & SD-VAE~\citep{rombach2022high}       & -                   & -                        & 63                          & 2.06                        & -                        & 778M              & 1476              \\
            EDM2-XXL~\citep{karras2024analyzing}    & SD-VAE~\citep{rombach2022high}       & -                   & -                        & 63                          & 1.91                        & -                        & 1.5B              & 734               \\
            DiT-XL/1\rp\citep{chen2024deep}         & DC-AE~\citep{chen2024deep}           & 1                   & 16$\times$16             & 250$\times$2                & 2.41                        & 263.56                   & 675M              & 400               \\
            U-ViT-H/1\rp\citep{chen2024deep}        & DC-AE~\citep{chen2024deep}           & 1                   & 16$\times$16             & 30$\times$2                 & 2.53                        & 255.07                   & 501M              & 400               \\
            REPA-XL/2~\citep{yu2024representation}  & SD-VAE~\citep{rombach2022high}       & 2                   & 32$\times$32             & 250$\times$2                & 2.08                        & 274.6                    & 675M              & 200               \\
            DDT-XL/2~\citep{wang2025ddt}            & SD-VAE~\citep{rombach2022high}       & 2                   & 32$\times$32             & 250$\times$2                & \textbf{1.28}               & \textbf{305.1}           & 675M              & -                 \\
            \midrule
            \multicolumn{9}{c}{\textbf{GANs \& masked \& autoregressive models}}                                                                                                                                                                                           \\
            \midrule
            VQGAN\rp\citep{esser2021taming}         & -                                    & -                   & -                        & 256                         & 18.65                       & -                        & 227M              & -                 \\
            MAGVIT-v2~\citep{yu2023language}        & -                                    & -                   & -                        & 64$\times$2                 & 1.91                        & 324.3                    & 307M              & 1080              \\
            MAR-L~\citep{li2024autoregressive}      & -                                    & -                   & -                        & 256$\times$2                & \textbf{1.73}               & 279.9                    & 479M              & 800               \\
            VAR-$d$36-s~\citep{tian2024visual}      & -                                    & -                   & -                        & 10$\times$2                 & 2.63                        & 303.2                    & 2.3B              & 350               \\
            \midrule
            \multicolumn{9}{c}{\textbf{Ours: \method-S sampling with models trained by prior works}}                                                                                                                                                                       \\
            \midrule
            EDM2-S~\citep{karras2024analyzing}      & SD-VAE~\citep{rombach2022high}       & -                   & -                        & 40\dg{23}                   & 2.53\dg{0.03}               & -                        & 280M              & -                 \\
            EDM2-L~\citep{karras2024analyzing}      & SD-VAE~\citep{rombach2022high}       & -                   & -                        & 50\dg{13}                   & 2.04\dg{0.02}               & -                        & 778M              & -                 \\
            EDM2-XXL~\citep{karras2024analyzing}    & SD-VAE~\citep{rombach2022high}       & -                   & -                        & 40\dg{23}                   & 1.88\dg{0.03}               & -                        & 1.5B              & -                 \\
            DDT-XL/2~\citep{wang2025ddt}            & SD-VAE~\citep{rombach2022high}       & 2                   & 32$\times$32             & 200\dg{300}                 & \textbf{1.25}\dg{0.03}      & -                        & 675M              & -                 \\
            \midrule
            \multicolumn{9}{c}{\textbf{Ours: models trained and sampled using \methods~(setting $\lambda=0$)}}                                                                                                                                                             \\
            \midrule
            Ours-XL/1                               & DC-AE~\citep{chen2024deep}           & 1                   & 16$\times$16             & 40                          & \textbf{1.48}               & -                        & 675M              & 800               \\
            Ours-XL/1                               & DC-AE~\citep{chen2024deep}           & 1                   & 16$\times$16             & 20                          & 1.68                        & -                        & 675M              & 800               \\
            Ours-XL/4                               & SD-VAE~\citep{rombach2022high}       & 4                   & 16$\times$16             & 40                          & 1.67                        & -                        & 675M              & 320               \\
            Ours-XL/4                               & SD-VAE~\citep{rombach2022high}       & 4                   & 16$\times$16             & 20                          & 1.80                        & -                        & 675M              & 320               \\
            \midrule
            \multicolumn{9}{c}{\textbf{ $256 \times 256$}}                                                                                                                                                                                                                 \\
            \midrule
            \multicolumn{8}{c}{\textbf{Diffusion \& flow-matching models}}                                                                                                                                                                                                 \\
            \midrule
            ADM-G~\citep{dhariwal2021diffusion}     & -                                    & -                   & -                        & 250$\times$2                & 4.59                        & 186.70                   & 559M              & 396               \\
            U-ViT-H/2~\citep{bao2023all}            & SD-VAE~\citep{rombach2022high}       & 2                   & 16 $\times$ 16           & 50$\times$2                 & 2.29                        & 263.88                   & 501M              & 400               \\
            DiT-XL/2~\citep{peebles2023scalable}    & SD-VAE~\citep{rombach2022high}       & 2                   & 16 $\times$ 16           & 250$\times$2                & 2.27                        & 278.24                   & 675M              & 1400              \\
            SiT-XL/2~\citep{ma2024sit}              & SD-VAE~\citep{rombach2022high}       & 2                   & 16 $\times$ 16           & 250$\times$2                & 2.06                        & 277.50                   & 675M              & 1400              \\
            MDT~\citep{gao2023masked}               & SD-VAE~\citep{rombach2022high}       & 2                   & 16 $\times$ 16           & 250$\times$2                & 1.79                        & 283.01                   & 675M              & 1300              \\
            REPA-XL/2~\citep{yu2024representation}  & SD-VAE~\citep{rombach2022high}       & 2                   & 16 $\times$ 16           & 250$\times$2                & 1.96                        & 264.0                    & 675M              & 200               \\
            REPA-XL/2~\citep{yu2024representation}  & SD-VAE~\citep{rombach2022high}       & 2                   & 16 $\times$ 16           & 250$\times$2                & 1.42                        & 305.7                    & 675M              & 800               \\
            Light.DiT~\citep{yao2025reconstruction} & VA-VAE~\citep{yao2025reconstruction} & 1                   & 16 $\times$ 16           & 250$\times$2                & 2.11                        & -                        & 675M              & 64                \\
            Light.DiT~\citep{yao2025reconstruction} & VA-VAE~\citep{yao2025reconstruction} & 1                   & 16 $\times$ 16           & 250$\times$2                & 1.35                        & -                        & 675M              & 800               \\
            DDT-XL/2~\citep{wang2025ddt}            & SD-VAE~\citep{rombach2022high}       & 2                   & 16 $\times$ 16           & 250$\times$2                & 1.31                        & 308.1                    & 675M              & 256               \\
            DDT-XL/2~\citep{wang2025ddt}            & SD-VAE~\citep{rombach2022high}       & 2                   & 16 $\times$ 16           & 250$\times$2                & 1.26                        & 310.6                    & 675M              & 400               \\
            REPA-E-XL~\citep{leng2025repa}          & E2E-VAE\citep{leng2025repa}          & 1                   & 16 $\times$ 16           & 250$\times$2                & \textbf{1.26}               & 314.9                    & 675M              & 800               \\
            \midrule
            \multicolumn{8}{c}{\textbf{GANs \& masked \& autoregressive  models}}                                                                                                                                                                                          \\
            \midrule
            VQGAN\rp\citep{sun2024autoregressive}   & -                                    & -                   & -                        & -                           & 2.18                        & -                        & 3.1B              & 300               \\
            MAR-L~\citep{li2024autoregressive}      & -                                    & -                   & -                        & 256$\times$2                & 1.78                        & 296.0                    & 479M              & 800               \\
            MAR-H~\citep{li2024autoregressive}      & -                                    & -                   & -                        & 256$\times$2                & \textbf{1.55}               & 303.7                    & 943M              & 800               \\
            VAR-$d$30-re~\citep{tian2024visual}     & -                                    & -                   & -                        & 10$\times$2                 & 1.73                        & 350.2                    & 2.0B              & 350               \\
            \midrule
            \multicolumn{8}{c}{\textbf{Ours: \method-S sampling with models trained by prior works}}                                                                                                                                                                       \\
            \midrule
            DDT-XL/2~\citep{wang2025ddt}            & SD-VAE~\citep{rombach2022high}       & 2                   & 16 $\times$ 16           & 100\dg{400}                 & 1.27\up{0.01}               & -                        & 675M              & -                 \\
            Light.DiT~\citep{yao2025reconstruction} & VA-VAE~\citep{yao2025reconstruction} & 1                   & 16 $\times$ 16           & 100\dg{400}                 & 1.21\dg{0.14}               & -                        & 675M              & -                 \\
            REPA-E-XL~\citep{leng2025repa}          & E2E-VAE\citep{leng2025repa}          & 1                   & 16 $\times$ 16           & 80\dg{420}                  & \textbf{1.06}\dg{0.20}      & -                        & 675M              & -                 \\
            REPA-E-XL~\citep{leng2025repa}          & E2E-VAE\citep{leng2025repa}          & 1                   & 16 $\times$ 16           & 20\dg{480}                  & 2.00\up{0.74}               & -                        & 675M              & -                 \\
            \midrule
            \multicolumn{8}{c}{\textbf{Ours: models trained and sampled using \methods~(setting $\lambda=0$)}}                                                                                                                                                             \\
            \midrule
            Ours-XL/2                               & SD-VAE~\citep{rombach2022high}       & 2                   & 16 $\times$ 16           & 60                          & 1.41                        & -                        & 675M              & 400               \\
            Ours-XL/1                               & VA-VAE~\citep{yao2025reconstruction} & 1                   & 16 $\times$ 16           & 60                          & 1.21                        & -                        & 675M              & 400               \\
            Ours-XL/1                               & E2E-VAE~\citep{leng2025repa}         & 1                   & 16 $\times$ 16           & 40                          & \textbf{1.21}               & -                        & 675M              & 800               \\
            Ours-XL/1                               & E2E-VAE~\citep{leng2025repa}         & 1                   & 16 $\times$ 16           & 20                          & 1.30                        & -                        & 675M              & 800               \\
            \bottomrule
        \end{tabular}%
    }
\end{table*}

\subsection{Detailed Comparison with SOTA Methods for Few-step Generation}

\begin{table*}[h!]
    \centering
    \caption{\small{\textbf{System-level quality comparison for few-step generation task on class-conditional ImageNet-1K ($512\times512$).}}}
    \label{tab:in512_fewsteps}
    \resizebox{\textwidth}{!}{%
        \begin{tabular}{lccccccccc}
            \toprule
            \textbf{METHOD}                     & \textbf{VAE/AE}                      & \textbf{Patch Size} & \textbf{Activation Size} & \textbf{NFE} ($\downarrow$) & \textbf{FID} ($\downarrow$) & \textbf{IS} & \textbf{\#Params} & \textbf{\#Epochs} \\
            \midrule
            \multicolumn{9}{c}{\textbf{ $512 \times 512$}}                                                                                                                                                                                                \\
            \midrule
            \multicolumn{9}{c}{\textbf{Consistency training \& distillation}}                                                                                                                                                                             \\
            \midrule
            sCT-M~\citep{lu2024simplifying}     & -                                    & -                   & -                        & 1                           & 5.84                        & -           & 498M              & 1837              \\
            sCT-M~\citep{lu2024simplifying}     & -                                    & -                   & -                        & 2                           & 5.53                        & -           & 498M              & 1837              \\
            sCT-L~\citep{lu2024simplifying}     & -                                    & -                   & -                        & 1                           & 5.15                        & -           & 778M              & 1274              \\
            sCT-L~\citep{lu2024simplifying}     & -                                    & -                   & -                        & 2                           & 4.65                        & -           & 778M              & 1274              \\
            sCT-XXL~\citep{lu2024simplifying}   & -                                    & -                   & -                        & 1                           & 4.29                        & -           & 1.5B              & 762               \\
            sCT-XXL~\citep{lu2024simplifying}   & -                                    & -                   & -                        & 2                           & 3.76                        & -           & 1.5B              & 762               \\
            sCD-M~\citep{lu2024simplifying}     & -                                    & -                   & -                        & 1                           & 2.75                        & -           & 498M              & 1997              \\
            sCD-M~\citep{lu2024simplifying}     & -                                    & -                   & -                        & 2                           & 2.26                        & -           & 498M              & 1997              \\
            sCD-L~\citep{lu2024simplifying}     & -                                    & -                   & -                        & 1                           & 2.55                        & -           & 778M              & 1434              \\
            sCD-L~\citep{lu2024simplifying}     & -                                    & -                   & -                        & 2                           & 2.04                        & -           & 778M              & 1434              \\
            sCD-XXL~\citep{lu2024simplifying}   & -                                    & -                   & -                        & 1                           & 2.28                        & -           & 1.5B              & 921               \\
            sCD-XXL~\citep{lu2024simplifying}   & -                                    & -                   & -                        & 2                           & \textbf{1.88}               & -           & 1.5B              & 921               \\
            \midrule
            \multicolumn{9}{c}{\textbf{GANs \& masked \& autoregressive  models}}                                                                                                                                                                         \\
            \midrule
            BigGAN~\citep{brock2018large}       & -                                    & -                   & -                        & 1                           & 8.43                        & -           & 160M              & -                 \\
            StyleGAN~\citep{sauer2022stylegan}  & -                                    & -                   & -                        & 1$\times$2                  & \textbf{2.41}               & 267.75      & 168M              & -                 \\
            MAGVIT-v2~\citep{yu2023language}    & -                                    & -                   & -                        & 64$\times$2                 & 1.91                        & 324.3       & 307M              & 1080              \\
            VAR-$d$36-s~\citep{tian2024visual}  & -                                    & -                   & -                        & 10$\times$2                 & 2.63                        & 303.2       & 2.3B              & 350               \\
            \midrule
            \multicolumn{9}{c}{\textbf{Ours: models trained and sampled using \methods~(setting $\lambda=0$)}}                                                                                                                                            \\
            \midrule
            Ours-XL/1                           & DC-AE~\citep{chen2024deep}           & 1                   & 16$\times$16             & 32                          & \textbf{1.55}               & -           & 675M              & 800               \\
            Ours-XL/1                           & DC-AE~\citep{chen2024deep}           & 1                   & 16$\times$16             & 16                          & 1.81                        & -           & 675M              & 800               \\
            Ours-XL/1                           & DC-AE~\citep{chen2024deep}           & 1                   & 16$\times$16             & 8                           & 3.07                        & -           & 675M              & 800               \\
            Ours-XL/1                           & DC-AE~\citep{chen2024deep}           & 1                   & 16$\times$16             & 4                           & 74.0                        & -           & 675M              & 800               \\
            \midrule
            \multicolumn{9}{c}{\textbf{Ours: models trained and sampled using \methods~(setting $\lambda=1$)}}                                                                                                                                            \\
            \midrule
            Ours-XL/1                           & DC-AE~\citep{chen2024deep}           & 1                   & 16$\times$16             & 1                           & 2.42                        & -           & 675M              & 840               \\
            Ours-XL/1                           & DC-AE~\citep{chen2024deep}           & 1                   & 16$\times$16             & 2                           & \textbf{1.75}               & -           & 675M              & 840               \\
            Ours-XL/4                           & SD-VAE~\citep{rombach2022high}       & 4                   & 16$\times$16             & 1                           & 2.63                        & -           & 675M              & 360               \\
            Ours-XL/4                           & SD-VAE~\citep{rombach2022high}       & 4                   & 16$\times$16             & 2                           & 2.11                        & -           & 675M              & 360               \\
            \midrule
            \multicolumn{9}{c}{\textbf{ $256 \times 256$}}                                                                                                                                                                                                \\
            \midrule
            \multicolumn{9}{c}{\textbf{Consistency training \& distillation}}                                                                                                                                                                             \\
            \midrule
            iCT~\citep{song2023improved}        & -                                    & -                   & -                        & 2                           & 20.3                        & -           & 675M              & -                 \\
            Shortcut-XL/2~\citep{frans2024one}  & SD-VAE~\citep{rombach2022high}       & 2                   & 16$\times$16             & 1                           & 10.6                        & -           & 676M              & 250               \\
            Shortcut-XL/2~\citep{frans2024one}  & SD-VAE~\citep{rombach2022high}       & 2                   & 16$\times$16             & 4                           & 7.80                        & -           & 676M              & 250               \\
            Shortcut-XL/2~\citep{frans2024one}  & SD-VAE~\citep{rombach2022high}       & 2                   & 16$\times$16             & 128                         & 3.80                        & -           & 676M              & 250               \\
            IMM-XL/2~\citep{zhou2025inductive}  & SD-VAE~\citep{rombach2022high}       & 2                   & 16$\times$16             & 1$\times$2                  & 7.77                        & -           & 675M              & 3840              \\
            IMM-XL/2~\citep{zhou2025inductive}  & SD-VAE~\citep{rombach2022high}       & 2                   & 16$\times$16             & 2$\times$2                  & 5.33                        & -           & 675M              & 3840              \\
            IMM-XL/2~\citep{zhou2025inductive}  & SD-VAE~\citep{rombach2022high}       & 2                   & 16$\times$16             & 4$\times$2                  & 3.66                        & -           & 675M              & 3840              \\
            IMM-XL/2~\citep{zhou2025inductive}  & SD-VAE~\citep{rombach2022high}       & 2                   & 16$\times$16             & 8$\times$2                  & 2.77                        & -           & 675M              & 3840              \\
            IMM ($\omega=1.5$)                  & SD-VAE~\citep{rombach2022high}       & 2                   & 16$\times$16             & 1$\times$2                  & 8.05                        & -           & 675M              & 3840              \\
            IMM ($\omega=1.5$)                  & SD-VAE~\citep{rombach2022high}       & 2                   & 16$\times$16             & 2$\times$2                  & 3.99                        & -           & 675M              & 3840              \\
            IMM ($\omega=1.5$)                  & SD-VAE~\citep{rombach2022high}       & 2                   & 16$\times$16             & 4$\times$2                  & 2.51                        & -           & 675M              & 3840              \\
            IMM ($\omega=1.5$)                  & SD-VAE~\citep{rombach2022high}       & 2                   & 16$\times$16             & 8$\times$2                  & \textbf{1.99}               & -           & 675M              & 3840              \\
            \midrule
            \multicolumn{9}{c}{\textbf{GANs \& masked \& autoregressive  models}}                                                                                                                                                                         \\
            \midrule
            BigGAN~\citep{brock2018large}       & -                                    & -                   & -                        & 1                           & 6.95                        & -           & 112M              & -                 \\
            GigaGAN~\citep{kang2023scaling}     & -                                    & -                   & -                        & 1                           & 3.45                        & 225.52      & 569M              & -                 \\
            StyleGAN~\citep{sauer2022stylegan}  & -                                    & -                   & -                        & 1$\times$2                  & \textbf{2.30}               & 265.12      & 166M              & -                 \\
            VAR-$d$30-re~\citep{tian2024visual} & -                                    & -                   & -                        & 10$\times$2                 & 1.73                        & 350.2       & 2.0B              & 350               \\
            \midrule
            \multicolumn{9}{c}{\textbf{Ours: models trained and sampled using \methods~(setting $\lambda=0$)}}                                                                                                                                            \\
            \midrule
            Ours-XL/1                           & VA-VAE~\citep{yao2025reconstruction} & 1                   & 16$\times$16             & 16                          & 2.11                        & -           & 675M              & 400               \\
            Ours-XL/1                           & VA-VAE~\citep{yao2025reconstruction} & 1                   & 16$\times$16             & 8                           & 6.09                        & -           & 675M              & 400               \\
            Ours-XL/1                           & E2E-VAE~\citep{leng2025repa}         & 1                   & 16$\times$16             & 16                          & \textbf{1.40}               & -           & 675M              & 800               \\
            Ours-XL/1                           & E2E-VAE~\citep{leng2025repa}         & 1                   & 16$\times$16             & 8                           & 2.68                        & -           & 675M              & 800               \\
            \midrule
            \multicolumn{9}{c}{\textbf{Ours: models trained and sampled using \methods~(setting $\lambda=1$)}}                                                                                                                                            \\
            \midrule
            Ours-XL/1                           & VA-VAE~\citep{yao2025reconstruction} & 1                   & 16$\times$16             & 2                           & \textbf{1.42}               & -           & 675M              & 432               \\
            Ours-XL/1                           & VA-VAE~\citep{yao2025reconstruction} & 1                   & 16$\times$16             & 1                           & 2.19                        & -           & 675M              & 432               \\
            Ours-XL/2                           & SD-VAE~\citep{rombach2022high}       & 2                   & 16$\times$16             & 1                           & 2.10                        & -           & 675M              & 424               \\
            Ours-XL/1                           & E2E-VAE~\citep{leng2025repa}         & 1                   & 16$\times$16             & 1                           & 2.29                        & -           & 675M              & 264               \\
            \bottomrule
        \end{tabular}%
    }
\end{table*}

\newpage

\subsection{Case Studies}

In this section, we provide several case studies to intuitively illustrate the technical components proposed in this paper.

\subsubsection{Analysis of Consistency Ratio $\lambda$}

\begin{figure}[t]
    \centering
    \begin{subfigure}[b]{0.325\textwidth}
        \centering
        \includegraphics[width=\linewidth]{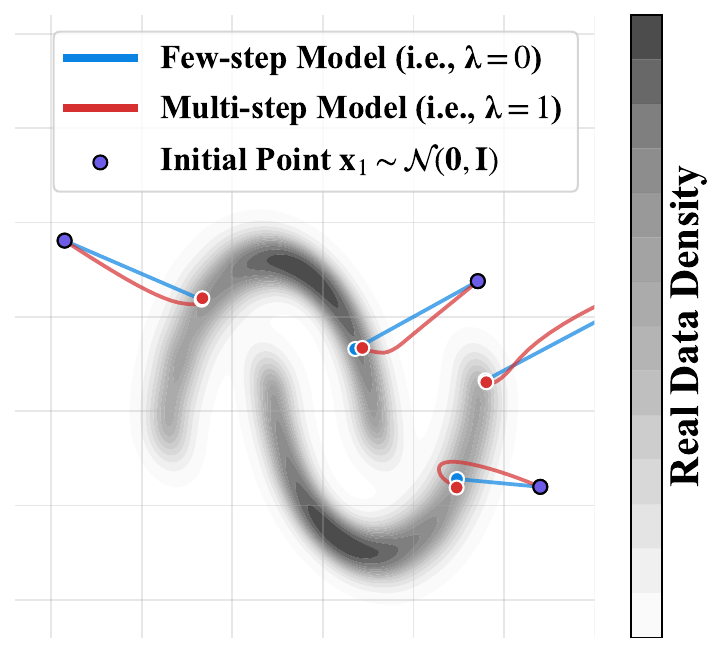}
        \caption{\textbf{Two Moons}}
        \label{fig:moons}
    \end{subfigure}
    \hfill
    \begin{subfigure}[b]{0.33\textwidth}
        \centering
        \includegraphics[width=\linewidth]{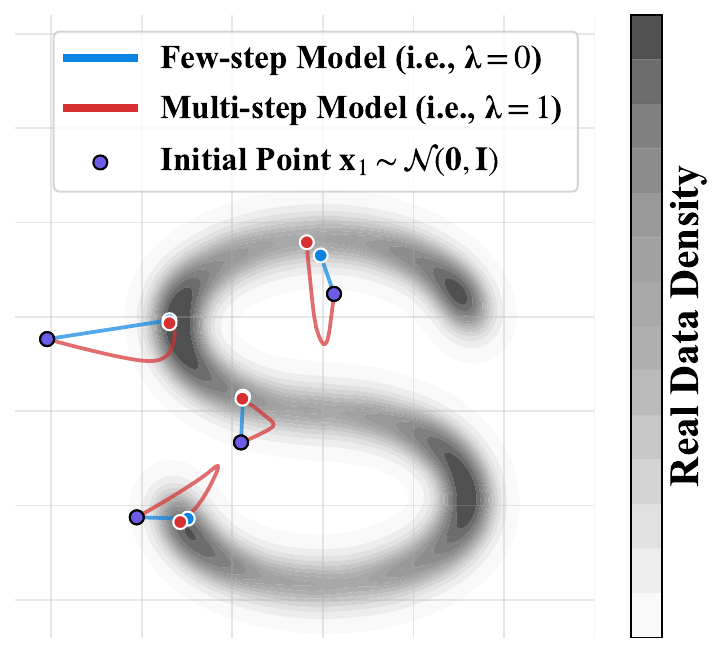}
        \caption{\textbf{S-Curve}}
        \label{fig:scurve}
    \end{subfigure}
    \hfill
    \begin{subfigure}[b]{0.325\textwidth}
        \centering
        \includegraphics[width=\linewidth]{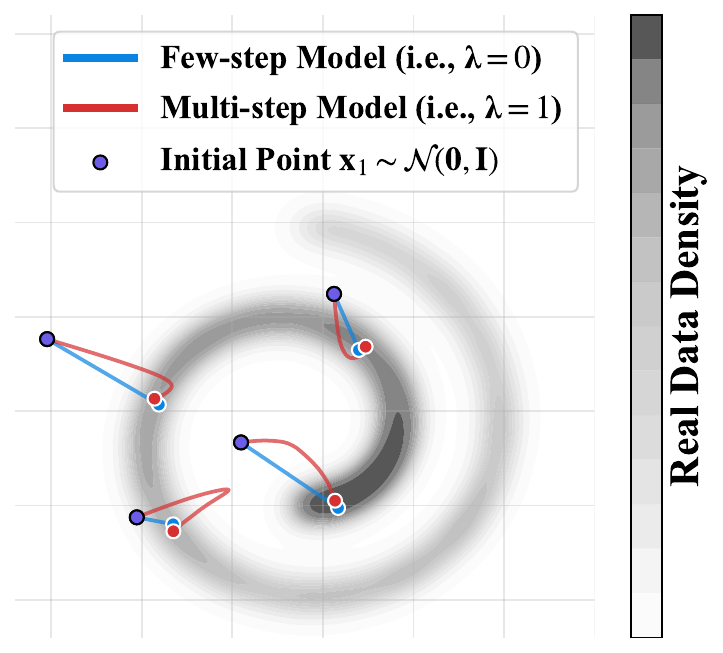}
        \caption{\textbf{Swiss Roll}}
        \label{fig:swissr}
    \end{subfigure}
    \caption{\small{
            \textbf{Case studies of \method~on three synthetic datasets.}
            These intuitive studies evaluate the ability of our \method~to capture the latent data structure for both few-step generation ($\lambda=1$) and multi-step generation ($\lambda=0$) tasks.
        }}
\end{figure}

\begin{figure*}[t!]
    \centering
    \includegraphics[width=\linewidth]{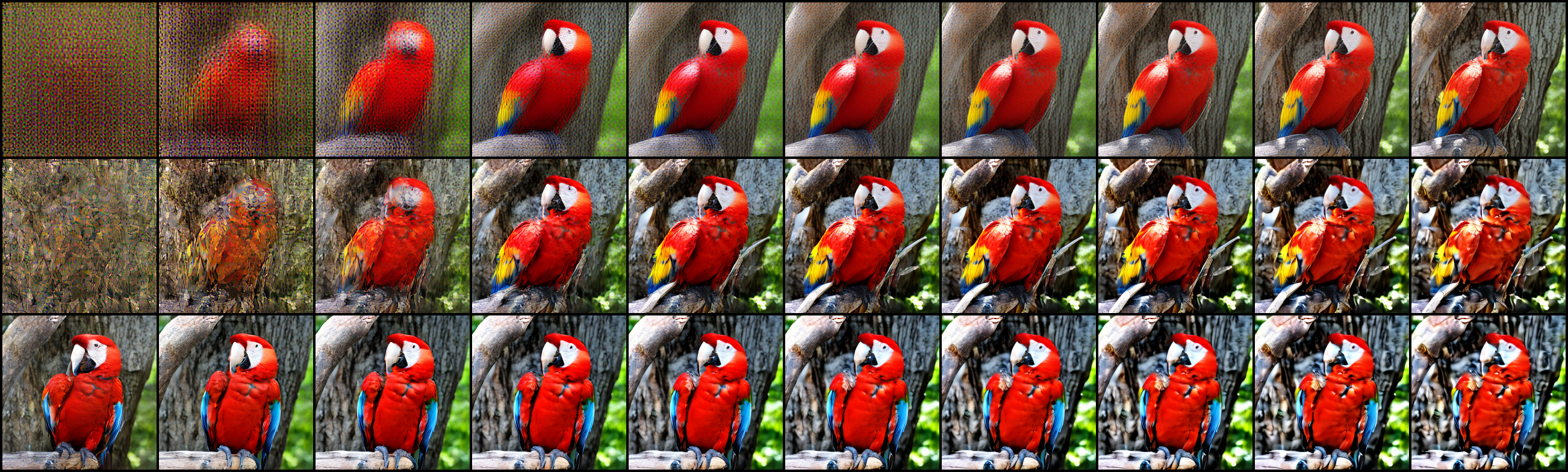}
    \caption{\small{
            \textbf{Intermediate images generated during $60$-step sampling from \method-S.}
            Columns display intermediate images $\hat{\mathbf{x}}_t$ produced at different timesteps $t$ during a single sampling trajectory, ordered from left to right by decreasing $t$.
            Rows correspond to models trained with $\lambda \in \{0.0, 0.5, 1.0\}$, ordered from top to bottom.
            Note that the initial noise for generating these images is the same.
        }}
    \label{fig:visualization}
\end{figure*}

We evaluate our approach on three synthetic benchmark datasets from \texttt{scikit-learn}~\citep{pedregosa2011scikit}: the Two Moons (non-linear separation, see~\figref{fig:moons}), S-Curve (manifold structure, see~\figref{fig:scurve}), and Swiss Roll (non-linear dimensionality reduction, see~\figref{fig:swissr}). These studies yield two primary observations:
\begin{enumerate}[label=(\alph*), nosep, leftmargin=16pt]
    \item Our \method~successfully captures the structure of the data distribution and maps initial points sampled from a Gaussian distribution to the target distribution, regardless of whether the task is few-step ($\lambda=1$) or multi-step ($\lambda=0$) generation.
    \item Models trained for multi-step ($\lambda=0$) and few-step ($\lambda=1$) generation map the same initial Gaussian noise to nearly identical target data points.
\end{enumerate}

To further validate these findings and explore additional properties of the consistency ratio $\lambda$, we conduct experiments on a real-world dataset (ImageNet-1K).
Specifically, we trained three models with three different settings of $\lambda \in \{0.0, 0.5, 1.0\}$.

The experimental results presented in \figref{fig:visualization} demonstrate the following:
\begin{enumerate}[label=(\alph*), nosep, leftmargin=16pt]
    \item For $\lambda=1.0$, high visual fidelity is achieved early in the sampling process. In contrast, for $\lambda=0.0$, high visual fidelity emerges in the mid to late stages. For $\lambda=0.5$, high-quality images appear in the mid-stage of sampling.
    \item Despite being trained with different settings of $\lambda$ values, the models produce remarkably similar generated images.
\end{enumerate}

In summary, we posit that while the setting of $\lambda$ affects the dynamics of the generation process, it does not substantially impact the final generated image quality. Detailed analysis of these phenomena is provided in~\appref{app:lambda_to_0},~\appref{app:lambda_to_1} and~\appref{app:optimal_solution}.

\subsubsection{Analysis of Transport Types}

\begin{figure*}[t!]
    \centering
    \includegraphics[width=\linewidth]{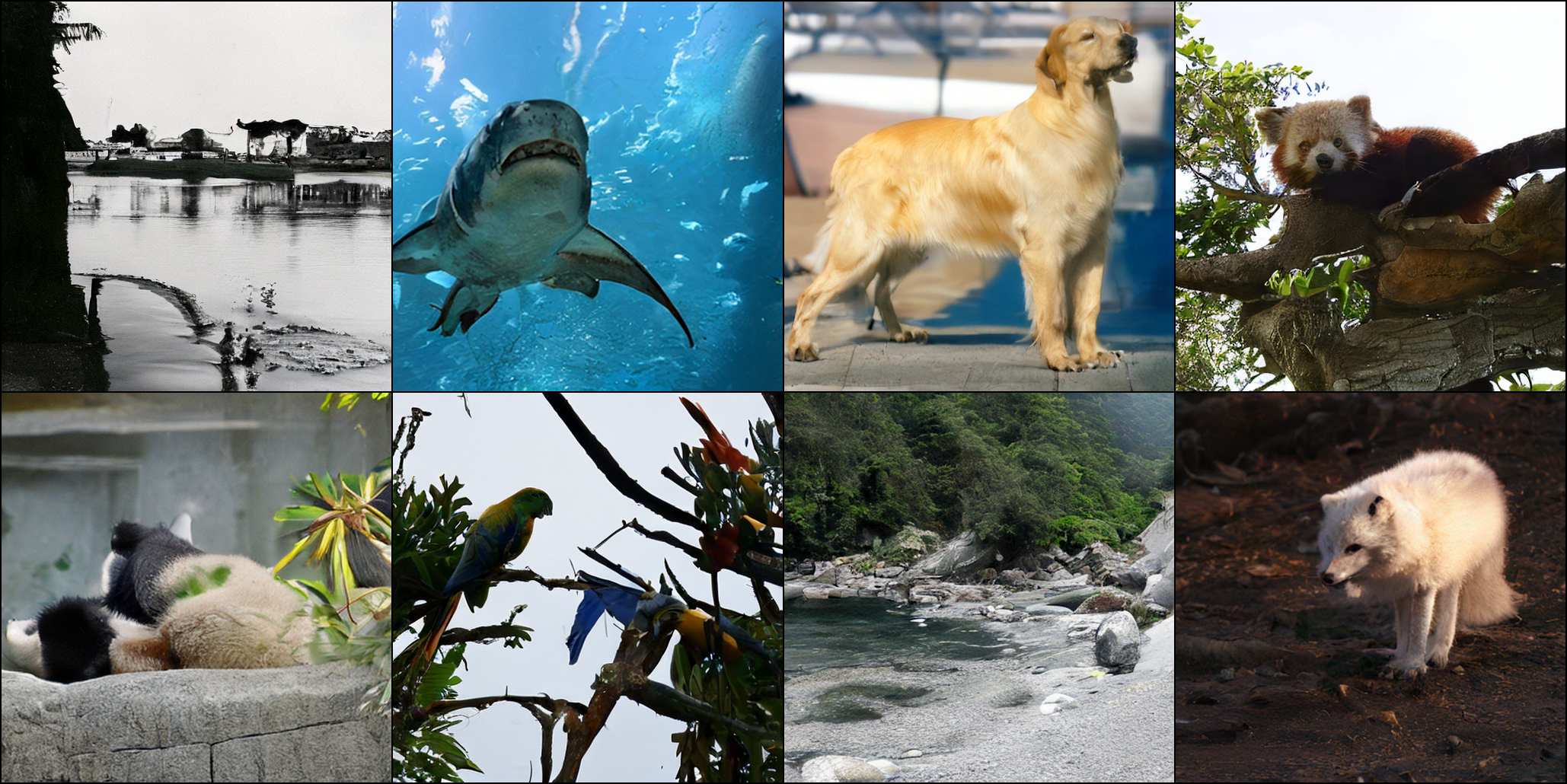}
    \caption{\small{
            \textbf{Visualization of generated images ($512\times512$) from pre-trained EDM2-S~\citep{karras2024analyzing}.}
        }}
    \label{fig:edm2_demo}
\end{figure*}

\begin{figure*}[t!]
    \centering
    \includegraphics[width=\linewidth]{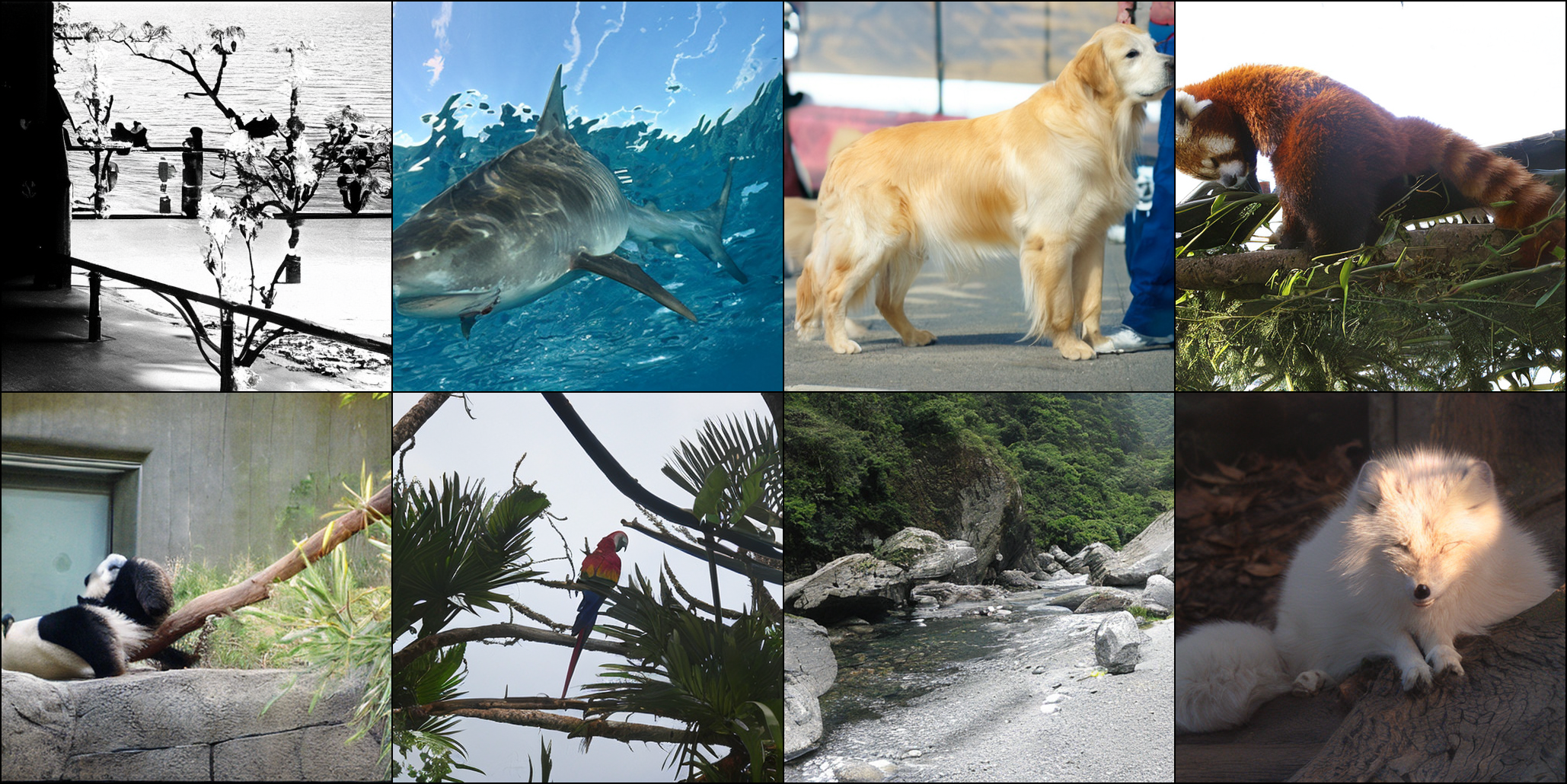}
    \caption{\small{
            \textbf{Visualization of generated images ($512\times512$) from pre-trained DDT-XL/2~\citep{wang2025ddt}.}
        }}
    \label{fig:ddt_demo}
\end{figure*}

Generated samples, obtained using \method-S with two distinct pre-trained models from prior works, are presented in \figref{fig:ddt_demo} and \figref{fig:edm2_demo}.
When using the identical initial Gaussian noise for both models, the generated images exhibit notable visual similarity.
This observation is unexpected, considering the models were trained independently~\citep{karras2024analyzing,wang2025ddt} using distinct algorithms, transport formulations, network architectures, and data augmentation strategies.
The similarity suggests that despite these differences, the learned probability flow ODEs may be converging to similar solutions.
See~\appref{app:close_lambda_0} for a comprehensive analysis of this phenomenon.

\section{Theoretical Analysis}
\label{app:theoretical_analysis}

\subsection{Main Results}

\subsubsection{Unified Training Objective}
\label{app:unified_training_objective}

\paragraph{Problem setup.}
Let \((\cV,\langle\cdot,\cdot\rangle)\) be a real inner-product space and
\(\mTheta\subseteq\mathbb{R}^p\) an open parameter domain.  We consider
\[
    \mmA:\mTheta\to \cV,
    \qquad
    \mmB\in\cV\quad(\text{constant w.r.t.\ }\mtheta\in\mTheta),
\]
and define the objective
\[
    \cJ(\mtheta)
    =\frac1\omega\,\bigl\|\mmA(\mtheta)-\mmB\bigr\|^2,
    \quad \omega>0.
\]
We denote by \(\nabla_{\mtheta}\mmA(\mtheta)\in\mathbb R^{p\times\dim\cV}\)
the Jacobian matrix of \(\mmA\).

\begin{lemma}[Gradient of a Squared Norm]
    \label{lem:grad-sqnorm}
    If \(\mv:\mTheta\to\cV\) is \(C^1\), then
    \[
        \nabla_{\mtheta}\bigl\|\mv(\mtheta)\bigr\|^2
        =2\,\bigl[\nabla_{\mtheta}\mv(\mtheta)\bigr]^\top\,\mv(\mtheta).
    \]
\end{lemma}

\begin{proof}
    Define \(\mf:\cV\to\mathbb{R}\) by \(\mf(\vv)=\langle \vv,\vv\rangle\).  Its Fr\'echet
    derivative is
    \[
        D\mf(\vv)[\hh]
        =\frac{\dm}{\dm\epsilon}\Bigl\|\vv+\epsilon \hh\Bigr\|^2\Big|_{\epsilon=0}
        =2\,\langle \vv,\hh\rangle.
    \]
    By the chain rule,
    \[
        \nabla_{\mtheta}\|\mv(\mtheta)\|^2
        =\bigl[\nabla_{\mtheta}\mv(\mtheta)\bigr]^\top
        D\mf\bigl(\mv(\mtheta)\bigr)
        =2\,\bigl[\nabla_{\mtheta}\mv(\mtheta)\bigr]^\top\,\mv(\mtheta).
    \]
\end{proof}

\begin{lemma}[Stop-Gradient Simplification]
    \label{lem:stopgrad}
    If \(\mmB\) does not depend on \(\mtheta\), then
    \[
        \nabla_{\mtheta}\bigl\|\mmA(\mtheta)-\mmB\bigr\|^2
        =2\,\bigl[\nabla_{\mtheta}\mmA(\mtheta)\bigr]^\top
        \bigl(\mmA(\mtheta)-\mmB\bigr).
    \]
\end{lemma}

\begin{proof}
    Set \(\mv(\mtheta)=\mmA(\mtheta)-\mmB\).  Since \(\nabla_{\mtheta}\mv
    =\nabla_{\mtheta}\mmA\), \lemref{lem:grad-sqnorm} applies directly.
\end{proof}

\begin{lemma}[Finite-Difference Definition]
    \label{lem:fd}
    Let \(t>0\), \(\lambda\in(0,1)\), and \(\mmA_0:\{\lambda t,t\}\to\cV\).
    Define
    \[
        \Delta\mmA
        :=\frac{\mmA_0(t)-\mmA_0(\lambda t)}{t-\lambda t}.
    \]
    Then
    \[
        \mmA_0(t)-\mmA_0(\lambda t)
        = (t-\lambda t)\,\Delta\mmA.
    \]
\end{lemma}

\begin{proof}
    Immediate from the definition.
\end{proof}

\begin{theorem}[Gradient Approximation via Finite Difference]
    \label{thm:main}
    Under the above hypotheses, let
    \[
        \cJ(\mtheta)
        =\frac1\omega\,\bigl\|\mmA(\mtheta)-\mmA_0(\lambda t)\bigr\|^2,
    \]
    and assume \(\mmA(\mtheta)\approx \mmA_0(t)\).  Then
    \[
        \nabla_{\mtheta}\cJ(\mtheta)
        =\frac{2}{\omega}\,
        \bigl[\nabla_{\mtheta}\mmA(\mtheta)\bigr]^\top
        \bigl(\mmA(\mtheta)-\mmA_0(\lambda t)\bigr)
        \approx
        \frac{2(t-\lambda t)}{\omega}\,
        \bigl[\nabla_{\mtheta}\mmA(\mtheta)\bigr]^\top
        \Delta\mmA
        \;\propto\;
        \bigl\langle\nabla_{\mtheta}\mmA(\mtheta),\,\Delta\mmA\bigr\rangle.
    \]
\end{theorem}

\begin{proof}
    Combine~\lemref{lem:stopgrad} and \lemref{lem:fd}, then absorb the scalar
    \(\tfrac{2(t-\lambda t)}\omega\) into the learning rate.  The only non-rigorous
    step is the approximation \(\mmA(\mtheta)\approx \mmA_0(t)\).
\end{proof}

\begin{lemma}
    \label{lem:inner_vs_sq}
    Let \(\mmF_{\mtheta}:\cX\to \cV\) be \(C^1\) in \(\mtheta\), let \(\yy\in\cV\),
    and let \(\mmF^-:\cX\to\cV\) be independent of \(\mtheta\).  Define
    \[
        \cL(\mtheta)
        =\E_{\xx}\bigl\|\mmF_{\mtheta}(\xx)-\mmF^-(\xx)+\yy\bigr\|^2,
        \quad
        \mmG(\mtheta)
        =\E_{\xx}\bigl\langle\mmF_{\mtheta}(\xx),\,\yy\bigr\rangle.
    \]
    Then
    \[
        \nabla_{\mtheta}\mmG(\mtheta)
        =\frac12\,\nabla_{\mtheta}\cL(\mtheta)
        \;-\;
        \E_{\xx}\bigl[\nabla_{\mtheta}\mmF_{\mtheta}(\xx)\bigr]^\top
        \bigl(\mmF_{\mtheta}(\xx)-\mmF^-(\xx)\bigr).
    \]
    In particular, if \(\mmF_{\mtheta}(\xx)\approx\mmF^-(\xx)\) then
    \[
        \nabla_{\mtheta}\mmG(\mtheta)
        \approx\frac12\,\nabla_{\mtheta}\cL(\mtheta).
    \]
\end{lemma}

\begin{proof}
    By~\lemref{lem:grad-sqnorm},
    \[
        \nabla_{\mtheta}\bigl\|\mmF_{\mtheta}-\mmF^-+\yy\bigr\|^2
        =2\,\bigl[\nabla_{\mtheta}\mmF_{\mtheta}\bigr]^\top
        \bigl(\mmF_{\mtheta}-\mmF^-+\yy\bigr).
    \]
    Taking expectation and dividing by \(2\) gives
    \[
        \tfrac12\nabla_{\mtheta}\cL
        =\E\!\bigl[(\nabla_{\mtheta}\mmF_{\mtheta})^\top(\mmF_{\mtheta}-\mmF^-)\bigr]
        +
        \E\!\bigl[(\nabla_{\mtheta}\mmF_{\mtheta})^\top\yy\bigr].
    \]
    On the other hand,
    \[
        \nabla_{\mtheta}\mmG
        =\nabla_{\mtheta}\E_\langle\mmF_{\mtheta},\yy\rangle
        =\E\!\bigl[(\nabla_{\mtheta}\mmF_{\mtheta})^\top\yy\bigr].
    \]
    Rearranging yields the stated identity.
\end{proof}

\begin{proof}[Derivation of the training objective]
    We begin with the original training objective:
    \[
        \cL(\mtheta)
        =\E_{(\zz,\xx)\sim p(\zz,\xx),\,t}
        \Bigl[
            \frac1{\hat\omega(t)}
            \bigl\|\,
            \mf^{\mathrm{\xx}}\bigl(\mmF_{\mtheta}(\xx_t,t),\xx_t,t\bigr)
            -\,
            \mf^{\mathrm{\xx}}\bigl(\mmF_{\mtheta^-}(\xx_{\lambda t},\lambda t),\xx_{\lambda t},\lambda t\bigr)
            \bigr\|_2^2
            \Bigr],
    \]
    where $\mtheta^-$ denotes the stop-gradient copy of~$\mtheta$.

    \medskip

    \noindent\textbf{Step 1.}  By~\lemref{lem:stopgrad},
    \[
        \nabla_{\mtheta}\cL
        =\E_{(\zz,\xx),\,t}
        \Bigl[
            \frac{2}{\hat\omega(t)}\,
            \bigl[\nabla_{\mtheta}\,\mf^{\mathrm{\xx}}\bigl(\mmF_{\mtheta}(\xx_t,t),\xx_t,t\bigr)\bigr]^{\top}
            \Bigl(
            \mf^{\mathrm{\xx}}(\mmF_{\mtheta}(\xx_t,t),\xx_t,t)
            -\,
            \mf^{\mathrm{\xx}}(\mmF_{\mtheta^-}(\xx_{\lambda t},\lambda t),\xx_{\lambda t},\lambda t)
            \Bigr)
            \Bigr].
    \]

    \medskip

    \noindent\textbf{Step 2.}  Define
    \[
        \mmA_0(s)
        :=\mf^{\mathrm{\xx}}\bigl(\mmF_{\mtheta^-}(\xx_s,s),\xx_s,s\bigr),
        \qquad
        \Delta \mmA
        :=\frac{\mmA_0(t)-\mmA_0(\lambda t)}{t-\lambda t}.
    \]
    Subsequently, based on~\lemref{lem:fd} and~\thmref{thm:main}, we obtain:
    \[
        \nabla_{\mtheta}\cL
        \;=\;
        \E_{t}
        \Bigl[
            \frac{2\,(t-\lambda t)}{\hat\omega(t)}
            \bigl[\nabla_{\mtheta}\,\mf^{\mathrm{\xx}}\bigr]^{\!T}
            \,\Delta \mmA
            \Bigr]
        \;\propto\;
        \E_{t}\bigl\langle\nabla_{\mtheta}\,\mf^{\mathrm{\xx}},\,\Delta \mmA\bigr\rangle.
    \]

    \medskip

    \noindent\textbf{Step 3.}  Since
    \[
        \mf^{\mathrm{\xx}}(\hh,\xx,s)
        =\frac{\alpha(s)\,\hh \;-\;\hat\alpha(s)\,\xx}
        {\alpha(s)\,\hat\gamma(s)-\hat\alpha(s)\,\gamma(s)},
        \quad
        \hat\omega(t)=\frac{\tan(t)}{4},
    \]
    one checks
    \[
        \nabla_{\hh}\mf^{\mathrm{\xx}}
        =\frac{\alpha(t)}
        {\alpha(t)\,\hat\gamma(t)-\hat\alpha(t)\,\gamma(t)},
        \quad
        \frac{1}{\hat\omega(t)}=\frac{4\cos(t)}{\sin(t)}.
    \]
    Hence
    \[
        \nabla_{\mtheta}\,
        \mf^{\mathrm{\xx}}\bigl(\mmF_{\mtheta}(\xx_t,t),\xx_t,t\bigr)
        =
        \frac{\alpha(t)}
        {\alpha(t)\,\hat\gamma(t)-\hat\alpha(t)\,\gamma(t)}
        \;\nabla_{\mtheta}\,\mmF_{\mtheta}(\xx_t,t),
    \]
    and therefore
    \[
        \nabla_{\mtheta}\cL
        \;\propto\;
        \E_{t}
        \Bigl\langle
        \nabla_{\mtheta}\,\mmF_{\mtheta}(\xx_t,t),
        \;
        \yy
        \Bigr\rangle,
        \quad
        \yy
        =\frac{4\,\alpha(t)\,\cos(t)}
        {\sin(t)\,\bigl(\alpha(t)\,\hat\gamma(t)-\hat\alpha(t)\,\gamma(t)\bigr)}
        \,\Delta \mmA.
    \]

    \medskip

    \noindent\textbf{Step 4.}  Finally apply~\lemref{lem:inner_vs_sq} with
    $
        \mmF_{\mtheta}=\mmF_{\mtheta}(\xx_t,t)
    $ and
    $
        \mmF^-=\mmF_{\mtheta^-}(\xx_t,t)
    $.
    We have
    \[
        \cL(\mtheta)
        =\E_{(\zz,\xx)\sim p(\zz,\xx),\,t}
        \Bigl\|
        \mmF_{\mtheta}(\xx_t,t)
        -\mmF_{\mtheta^-}(\xx_t,t)
        +\yy
        \Bigr\|_2^2.
    \]
    Pulling the overall $\cos(t)$ inside the norm yields the final training objective
    \[
        \cL(\mtheta)
        =\E_{(\zz,\xx)\sim p(\zz,\xx),\,t}
        \Bigl[
            \cos(t)\,
            \Bigl\|
            \mmF_{\mtheta}(\xx_t,t)
            -\mmF_{\mtheta^-}(\xx_t,t)
            +\frac{4\,\alpha(t)\,\Delta\mf^{\mathrm{\xx}}_t}
            {\sin(t)\,\bigl(\alpha(t)\,\hat\gamma(t)-\hat\alpha(t)\,\gamma(t)\bigr)}
            \Bigr\|_2^2
            \Bigr].
    \]
    This completes the derivation.
\end{proof}

\subsubsection{Learning Objective when $\lambda = 0$}
\label{app:lambda_to_0}

Recall that \((\zz,\xx)\sim p(\zz,\xx)\) is a pair of latent and data variables (typically independent), and let \(t\in[0,1]\).
We have four differentiable scalar functions
$
    \alpha,\gamma,\hat\alpha,\hat\gamma\colon[0,1]\to\mathbb{R}
$
, the \emph{noisy interpolant}
$
    \xx_t \;=\; \alpha(t)\,\zz \;+\; \gamma(t)\,\xx
$
and
$
    \mmF_t \;=\; \mmF_{\mtheta}(\xx_t,t)
$.
We define the \(\xx\)- and \(\zz\)-prediction functions by
\[
    \mf^{\xx}(\mmF_t,\xx_t,t) =
    \frac{\alpha(t)\,\mmF_t \;-\; \hat\alpha(t)\,\xx_t}
    {\alpha(t)\,\hat\gamma(t) \;-\; \hat\alpha(t)\,\gamma(t)}\,,
    \quad \text{and} \quad
    \mf^{\zz}(\mmF_t,\xx_t,t) =
    \frac{\hat\gamma(t)\,\xx_t \;-\; \gamma(t)\,\mmF_t}
    {\alpha(t)\,\hat\gamma(t) \;-\; \hat\alpha(t)\,\gamma(t)}\,.
\]
Finally, let \(\hat\omega(t)>0\) be a weight function.
We consider the \(\xx\)- and \(\zz\)-prediction losses
\[
    \mathcal{L}_{\xx}(\mtheta)
    =
    \E_{(\zz,\xx)\sim p(\zz,\xx),\,t}
    \Bigl[\frac{1}{\hat{\omega}(t)}\,
        \bigl\|\mf^{\xx}(\mmF_t,\xx_t,t)-\xx\bigr\|_2^2
        \Bigr]\,,
\]
\[
    \mathcal{L}_{\zz}(\mtheta)
    =
    \E_{(\zz,\xx)\sim p(\zz,\xx),\,t}
    \Bigl[\frac{1}{\hat{\omega}(t)}\,
        \bigl\|\mf^{\zz}(\mmF_t,\xx_t,t)-\zz\bigr\|_2^2
        \Bigr]\,.
\]
Recall that our unified loss function is defined by:
\[
    \cL(\mtheta) = \EE{(\zz,\xx)\sim p(\zz,\xx),t}{\frac{1}{\hat{\omega}(t)} \norm{ \mf^{\mathrm{\xx}}(\mmF_{\mtheta}(\xx_t,t), \xx_t, t) - \mf^{\mathrm{\xx}}(\mmF_{\mtheta^-}(\xx_{\lambda t},\lambda t), \xx_{\lambda t}, \lambda t) }_2^2 }\,.
\]
We have \(\cL(\mtheta) = \mathcal{L}_{\xx}(\mtheta)\) when \(\lambda=0\), since \(\mf^{\mathrm{\xx}}(\mmF_0, \xx_0, 0) = 0\).
Then, we define the direct-field loss
\[
    \mathcal{L}_{\mmF}(\mtheta)
    =
    \E_{(\zz,\xx),\,t}
    \Bigl[
        w(t)\,\bigl\|\mmF_t -(\hat\alpha(t)\,\zz\;+\;\hat\gamma(t)\,\xx)\bigr\|_2^2
        \Bigr]\,,
    \quad
    w(t)>0\,.
\]

\begin{lemma}[Equivalence of \(\xx\)-prediction and direct-field loss]
    \label{lem:x-equivalence}
    For all \(\mtheta\),
    \[
        \mf^{\xx}(\mmF_t,\xx_t,t)-\xx
        \;=\;
        \frac{\alpha(t)}{\alpha(t)\,\hat\gamma(t)-\hat\alpha(t)\,\gamma(t)}
        \,\bigl[\mmF_t -(\hat\alpha(t)\,\zz\;+\;\hat\gamma(t)\,\xx)\bigr].
    \]
    Hence
    \[
        \mathcal{L}_{\xx}(\mtheta)
        =
        \E_{(\zz,\xx),\,t}
        \Bigl[
            \frac{\alpha(t)^2}{\hat\omega(t)\,\bigl(\alpha(t)\,\hat\gamma(t)-\hat\alpha(t)\,\gamma(t)\bigr)^2}
            \,\bigl\|\mmF_t -(\hat\alpha(t)\,\zz\;+\;\hat\gamma(t)\,\xx)\bigr\|_2^2
            \Bigr],
    \]
    so \(\mathcal{L}_{\xx}\) is equivalent to \(\mathcal{L}_{\mmF}\) with
    \[
        w(t)=\frac{\alpha(t)^2}{\hat\omega(t)\,\bigl(\alpha(t)\,\hat\gamma(t)-\hat\alpha(t)\,\gamma(t)\bigr)^2}.
    \]
\end{lemma}

\begin{proof}
    Compute
    \[
        \mf^{\xx}(\mmF_t,\xx_t,t)-\xx
        =
        \frac{\alpha(t)\,\mmF_t-\hat\alpha(t)\,\xx_t}
        {\alpha(t)\,\hat\gamma(t)-\hat\alpha(t)\,\gamma(t)}
        \;-\;\xx.
    \]
    Since \(\xx_t=\alpha(t)\zz+\gamma(t)\xx\), the numerator becomes
    \[
        \alpha\,\mmF_t
        -\hat\alpha\bigl(\alpha\zz+\gamma\xx\bigr)
        -\bigl(\alpha\,\hat\gamma-\hat\alpha\,\gamma\bigr)\xx
        =
        \alpha(t)\Bigl[\mmF_t - \bigl(\hat\alpha(t)\zz+\hat\gamma(t)\xx\bigr)\Bigr].
    \]
    Dividing by \(\alpha\,\hat\gamma-\hat\alpha\,\gamma\) yields the desired factorization.  Substituting into \(\mathcal{L}_{\xx}\) gives the weight \(w(t)\) as above.
\end{proof}

\begin{lemma}[Equivalence of \(\zz\)-Prediction and Direct-Field Loss]
    \label{lem:z-equivalence}
    For all \(\mtheta\),
    \[
        \mf^{\zz}(\mmF_t,\xx_t,t)-\zz
        \;=\;
        \frac{\gamma(t)}{\alpha(t)\,\hat\gamma(t)-\hat\alpha(t)\,\gamma(t)}
        \,\bigl[T(t,\zz,\xx)-\mmF_t\bigr].
    \]
    Hence
    \[
        \mathcal{L}_{\zz}(\mtheta)
        =
        \E_{(\zz,\xx),\,t}
        \Bigl[
            \frac{\gamma(t)^2}{\hat\omega(t)\,\bigl(\alpha(t)\,\hat\gamma(t)-\hat\alpha(t)\,\gamma(t)\bigr)^2}
            \,\bigl\|\mmF_t -( \hat\alpha(t)\,\zz\;+\;\hat\gamma(t)\,\xx)\bigr\|_2^2
            \Bigr],
    \]
    so \(\mathcal{L}_{\zz}\) is equivalent to \(\mathcal{L}_{\mmF}\) with
    \[
        w(t)=\frac{\gamma(t)^2}{\hat\omega(t)\,\bigl(\alpha(t)\,\hat\gamma(t)-\hat\alpha(t)\,\gamma(t)\bigr)^2}.
    \]
\end{lemma}

\begin{proof}
    Compute
    \[
        \mf^{\zz}(\mmF_t,\xx_t,t)-\zz
        =
        \frac{\hat\gamma(t)\,\xx_t-\gamma(t)\,\mmF_t}
        {\alpha(t)\,\hat\gamma(t)-\hat\alpha(t)\,\gamma(t)}
        \;-\;\zz.
    \]
    Using \(\xx_t=\alpha\zz+\gamma\xx\), the numerator is
    \[
        \hat\gamma(\alpha\zz+\gamma\xx)
        -\gamma\,\mmF_t
        -\bigl(\alpha\,\hat\gamma-\hat\alpha\,\gamma\bigr)\zz
        =
        \gamma(t)\Bigl[\hat\alpha(t)\zz+\hat\gamma(t)\xx-\mmF_t\Bigr].
    \]
    Dividing by \(\alpha\,\hat\gamma-\hat\alpha\,\gamma\) gives the factorization.  Substitution into \(\mathcal{L}_{\zz}\) yields the stated equivalence.
\end{proof}

Then,
when $\lambda=0$, we aim to derive the Probability Flow Ordinary Differential Equation (PF-ODE)~\citep{song2020score} corresponding to a defined forward process from time $0$ to $1$.

\begin{lemma}[Probability Flow ODE for the linear Gaussian forward process]
    \label{lem:pf-ode}
    Let $p(\mathbf{x})$ be a data distribution on $\mathbb{R}^d$, and let
    $\mathbf{z}\sim\mathcal{N}(\mathbf{0},\mathbf{I}_d)$ be independent of $\mathbf{x}$.  Let
    $\alpha,\gamma:[0,1]\to\mathbb{R}$ be continuously differentiable scalar functions satisfying
    \[
        \alpha(0)=0,\quad \alpha(1)=1,
        \qquad
        \gamma(0)=1,\quad \gamma(1)=0,
    \]
    and assume $\gamma(t)\neq0$ for $t\in(0,1)$.  Define the \emph{forward process}
    \[
        \mathbf{x}_t \;=\; \alpha(t)\,\mathbf{z}\;+\;\gamma(t)\,\mathbf{x},
        \qquad t\in[0,1],
    \]
    so that $\mathbf{x}_0=\mathbf{x}\sim p(\mathbf{x})$ and $\mathbf{x}_1=\mathbf{z}\sim\mathcal{N}(0,I)$.
    Let $p_t(\mathbf{x}_t)$ denote the marginal density of $\mathbf{x}_t$.  Then the \emph{Probability
        Flow ODE} for this process,
    \[
        \frac{\mathrm{d}\,\mathbf{x}_t}{\mathrm{d}t}
        \;=\;\mathbf{f}(\mathbf{x}_t,t)
        \;-\;\tfrac12\,g(t)^2\,
        \nabla_{\mathbf{x}_t}\log p_t(\mathbf{x}_t),
    \]
    takes the explicit form
    \begin{equation}\label{eq:pf-ode-final}
        \frac{\mathrm{d}\,\mathbf{x}_t}{\mathrm{d}t}
        \;=\;
        \frac{\gamma'(t)}{\gamma(t)}\,\mathbf{x}_t
        \;-\;
        \Bigl[\,
            \alpha(t)\,\alpha'(t)
            \;-\;
            \frac{\gamma'(t)}{\gamma(t)}\,\alpha(t)^2
            \Bigr]\,
        \nabla_{\mathbf{x}_t}\,\log p_t(\mathbf{x}_t) \,.
    \end{equation}
\end{lemma}

\begin{proof}
    We first represent the forward process $\mathbf{x}_t$ as the solution of the linear SDE
    \[
        \mathrm{d}\,\mathbf{x}_t
        \;=\;
        \mathbf{f}(\mathbf{x}_t,t)\,\mathrm{d}t
        \;+\;
        g(t)\,\mathrm{d}\mathbf{w}_t,
    \]
    where $\mathbf{w}_t$ is a standard $d$‐dimensional Wiener process, and where
    $\mathbf{f}(\cdot,t)$ and $g(t)$ are to be determined so that
    $\mathbf{x}_t=\alpha(t)\,\mathbf{z}+\gamma(t)\,\mathbf{x}$ in law.

    1.  Drift term via the conditional mean.
    Since $\mathbf{z}$ and $\mathbf{x}$ are independent,
    \[
        \mathbb{E}[\mathbf{x}_t\mid\mathbf{x}_0=\mathbf{x}]
        \;=\;
        \gamma(t)\,\mathbf{x}.
    \]
    Differentiating in $t$ gives
    \[
        \frac{\mathrm{d}}{\mathrm{d}t}
        \mathbb{E}[\mathbf{x}_t\mid\mathbf{x}_0]
        \;=\;
        \gamma'(t)\,\mathbf{x}.
    \]
    On the other hand, if $\mathbf{f}(\mathbf{x}_t,t)=H(t)\,\mathbf{x}_t$
    for some matrix $H(t)$, then
    \[
        \frac{\mathrm{d}}{\mathrm{d}t}
        \mathbb{E}[\mathbf{x}_t\mid\mathbf{x}_0]
        \;=\;
        H(t)\,\mathbb{E}[\mathbf{x}_t\mid\mathbf{x}_0]
        \;=\;
        H(t)\,\gamma(t)\,\mathbf{x}.
    \]
    Comparison yields $H(t)=\gamma'(t)/\gamma(t)\,\mathbf{I}_d$, so
    \[
        \mathbf{f}(\mathbf{x}_t,t)
        = \frac{\gamma'(t)}{\gamma(t)}\,\mathbf{x}_t.
    \]

    2.  Diffusion term via the conditional variance.
    The covariance of $\mathbf{x}_t$ given $\mathbf{x}_0$ is
    \[
        \operatorname{Var}(\mathbf{x}_t\mid\mathbf{x}_0)
        = \alpha(t)^2\,\mathbf{I}_d.
    \]
    For a linear SDE with drift matrix $H(t)$ and scalar diffusion $g(t)$,
    the covariance $\Sigma(t)$ satisfies the Lyapunov equation
    \[
        \frac{\mathrm{d}\,\Sigma(t)}{\mathrm{d}t}
        = H(t)\,\Sigma(t)
        + \Sigma(t)\,H(t)^\top
        + g(t)^2\,\mathbf{I}_d.
    \]
    Substitute $\Sigma(t)=\alpha(t)^2\mathbf{I}_d$ and
    $H(t)=\tfrac{\gamma'(t)}{\gamma(t)}\mathbf{I}_d$.  Since
    $\tfrac{\mathrm{d}}{\mathrm{d}t}\bigl(\alpha(t)^2\bigr)=2\,\alpha(t)\,\alpha'(t)$, we get
    \[
        2\,\alpha(t)\,\alpha'(t)\,\mathbf{I}_d
        =
        2\,\frac{\gamma'(t)}{\gamma(t)}\,\alpha(t)^2\,\mathbf{I}_d
        \;+\;
        g(t)^2\,\mathbf{I}_d.
    \]
    Rearranging yields
    \[
        g(t)^2
        = 2\,\alpha(t)\,\alpha'(t)
        \;-\;
        2\,\frac{\gamma'(t)}{\gamma(t)}\,\alpha(t)^2.
    \]

    3.  Probability Flow ODE.
    By general theory (see, e.g., \emph{de Bortoli et al.}), the probability flow ODE associated
    with the SDE $\mathrm{d}\mathbf{x}_t=\mathbf{f}(\mathbf{x}_t,t)\,\mathrm{d}t+g(t)\,\mathrm{d}\mathbf{w}_t$
    is
    \[
        \frac{\mathrm{d}\,\mathbf{x}_t}{\mathrm{d}t}
        =
        \mathbf{f}(\mathbf{x}_t,t)
        \;-\;
        \tfrac12\,g(t)^2\,\nabla_{\mathbf{x}_t}\log p_t(\mathbf{x}_t).
    \]
    Substituting the expressions for $\mathbf{f}$ and $g^2$ above gives
    \[
        \frac{\mathrm{d}\,\mathbf{x}_t}{\mathrm{d}t}
        =
        \frac{\gamma'(t)}{\gamma(t)}\,\mathbf{x}_t
        \;-\;
        \Bigl[\,
            \alpha(t)\,\alpha'(t)
            \;-\;
            \tfrac{\gamma'(t)}{\gamma(t)}\,\alpha(t)^2
            \Bigr]\,
        \nabla_{\mathbf{x}_t}\log p_t(\mathbf{x}_t),
    \]
    i.e.,
    \[
        \mathbf{f}(\mathbf{x}_t,t)
        = \frac{\gamma'(t)}{\gamma(t)}\,\mathbf{x}_t,
        \qquad
        g(t)^2
        = 2\,\alpha(t)\,\alpha'(t)
        \;-\;
        2\,\frac{\gamma'(t)}{\gamma(t)}\,\alpha(t)^2.
    \]
    which is exactly the claimed formula \eqref{eq:pf-ode-final}.
\end{proof}

\begin{lemma}[Tweedie formula~\citep{song2020score} for the linear Gaussian model]
    \label{lem:tweedie_formula}
    Under the linear Gaussian interpolation model
    $
        \xx_t\mid\xx \;\sim\; \mathcal{N}\bigl(\gamma(t)\,\xx,\;\alpha^2(t)\,\mathbf{I}\bigr),
    $
    the conditional expectation of \(\xx\) given \(\xx_t\) is
    \[
        \mathbb{E}[\xx\mid\xx_t]
        \;=\;
        \frac{\xx_t + \alpha^2(t)\,\nabla_{\xx_t}\log p_t(\xx_t)}{\gamma(t)} \,.
    \]
\end{lemma}

\begin{proof}
    We write the conditional expectation by Bayes’ rule:
    \[
        \mathbb{E}[\xx\mid\xx_t]
        = \int \xx\,p(\xx\mid\xx_t)\,d\xx
        = \frac{1}{p_t(\xx_t)}
        \int \xx\,p_t(\xx_t\mid\xx)\,p(\xx)\,d\xx,
    \]
    where \(p_t(\xx_t)=\int p_t(\xx_t\mid\xx)\,p(\xx)\,d\xx\).

    Since
    \(\;p_t(\xx_t\mid\xx)
    = (2\pi\alpha^2(t))^{-d/2}
    \exp\!\bigl(-\tfrac{1}{2\alpha^2(t)}\|\xx_t-\gamma(t)\xx\|^2\bigr)\),
    we have
    \[
        \nabla_{\xx_t}p_t(\xx_t\mid\xx)
        = -\,\frac{1}{\alpha^2(t)}\,(\xx_t-\gamma(t)\xx)\,p_t(\xx_t\mid\xx).
    \]
    Differentiating the marginal,
    \[
        \nabla_{\xx_t}p_t(\xx_t)
        = \int \nabla_{\xx_t}p_t(\xx_t\mid\xx)\,p(\xx)\,d\xx
        = -\frac{1}{\alpha^2(t)}
        \int (\xx_t-\gamma(t)\xx)\,p_t(\xx_t\mid\xx)\,p(\xx)\,d\xx.
    \]
    Multiply by \(-\alpha^2(t)\) and split:
    \[
        -\alpha^2(t)\,\nabla_{\xx_t}p_t(\xx_t)
        = \xx_t\,p_t(\xx_t)
        -\gamma(t)\int \xx\,p_t(\xx_t\mid\xx)\,p(\xx)\,d\xx.
    \]
    Rearrange and divide by \(\gamma(t)p_t(\xx_t)\):
    \[
        \frac{1}{p_t(\xx_t)}
        \int \xx\,p_t(\xx_t\mid\xx)\,p(\xx)\,d\xx
        = \frac{\xx_t + \alpha^2(t)\,\nabla_{\xx_t}p_t(\xx_t)/p_t(\xx_t)}{\gamma(t)}
        = \frac{\xx_t + \alpha^2(t)\,\nabla_{\xx_t}\log p_t(\xx_t)}{\gamma(t)}.
    \]
    Hence
    \(
    \mathbb{E}[\xx\mid\xx_t]
    = (\xx_t + \alpha^2(t)\,\nabla_{\xx_t}\log p_t(\xx_t))\,/\,\gamma(t),
    \)
    as claimed.
\end{proof}

\begin{lemma}[Optimal predictors as conditional expectations]
    \label{lem:opt_preders}
    For each fixed \(t\) and observed \(\xx_t\), the pointwise minimizers
    \(\mf^{\xx}_\star\) and \(\mf^{\zz}_\star\) for the objective function $\cL(\mtheta)$ satisfy
    \[
        \mf^{\xx}_\star(\mmF_t,\xx_t,t)
        = \E[\xx\mid \xx_t] \, ,
        \quad
        \mf^{\zz}_\star(\mmF_t,\xx_t,t)
        = \E[\zz\mid \xx_t] \, .
    \]
\end{lemma}

\begin{proof}
    Fix \(t\) and \(\xx_t\).
    By \lemref{lem:x-equivalence} and \lemref{lem:z-equivalence}, we conclude that the minimizers of $\cL(\mtheta)$ are equivalent to those of $\cL_{\xx}$ and $\cL_{\zz}$.

    Then,
    up to an additive constant independent of
    \(\mf^\xx\), the contribution of \((t,\xx_t)\) to \(\cL_{\xx}\) is
    \[
        \cJ_{\xx}\bigl(\mf^\xx(\mmF_t,\xx_t,t)\bigr)
        = \E\bigl[\|\mf^\xx(\mmF_t,\xx_t,t)-\xx\|_2^2
            \mid \xx_t\bigr].
    \]
    For any random vector \(X\), the function \(w\mapsto\E\|w - X\|^2\)
    is uniquely minimized at \(w = \E[X]\).  Therefore
    \[
        \mf^{\xx}_\star(\mmF_t,\xx_t,t)
        = \arg\min_{w}\;\E\bigl[\|w-\xx\|^2\mid\xx_t\bigr]
        = \E[\xx\mid\xx_t].
    \]
    The same argument applies to
    \[
        \cJ_{\zz}\bigl(\mf^\zz(\mmF_t,\xx_t,t)\bigr)
        = \E\bigl[\|\mf^\zz(\mmF_t,\xx_t,t)-\zz\|_2^2
            \mid \xx_t\bigr],
    \]
    yielding
    \[
        \mf^{\zz}_\star(\mmF_t,\xx_t,t)
        = \E[\zz\mid\xx_t].
    \]
\end{proof}

\begin{theorem}
    Under the linear Gaussian interpolation model
    $
        \xx_t \;=\;\alpha(t)\,\zz \;+\;\gamma(t)\,\xx,
    $
    with \(\zz\sim\cN(0,\mathbf{I})\) independent of \(\xx\), we have
    \[
        \mf^{\xx}_\star(\mmF_t,\xx_t,t)
        =\frac{\xx_t \;+\;\alpha^2(t)\,\nabla_{\xx_t}\log p_t(\xx_t)}
        {\gamma(t)}\,,
        \quad
        \mf^{\zz}_\star(\mmF_t,\xx_t,t)
        =\alpha(t)\,\nabla_{\xx_t}\log p_t(\xx_t)\,.
    \]
    Then for every \(t\),
    \[
        \alpha'(t)\,\mf^{\zz}_\star(\mmF_t,\xx_t,t)
        +\gamma'(t)\,\mf^{\xx}_\star(\mmF_t,\xx_t,t)
        =\frac{\gamma'(t)}{\gamma(t)}\,\xx_t
        \;-\;
        \Bigl[
            \alpha(t)\,\alpha'(t)
            \;-\;
            \frac{\gamma'(t)}{\gamma(t)}\,\alpha^2(t)
            \Bigr]\,
        \nabla_{\xx_t}\log p_t(\xx_t).
    \]
    As a result, by \lemref{lem:pf-ode}, we conclude:
    \[
        \alpha'(t)\,\mf^{\zz}_\star(\mmF_t,\xx_t,t)
        +\gamma'(t)\,\mf^{\xx}_\star(\mmF_t,\xx_t,t)
        = \frac{\mathrm{d}\,\mathbf{x}_t}{\mathrm{d}t}
    \]
\end{theorem}

\begin{proof}
    \noindent\textbf{Tweedie formula for $\mf^{\xx}_\star(\mmF_t,\xx_t,t)$.}
    According to \lemref{lem:opt_preders} and \lemref{lem:tweedie_formula}, we have
    \[
        \mf^{\xx}_\star(\mmF_t,\xx_t,t) = \mathbb{E}[\xx\mid\xx_t]
        = \frac{\xx_t \;+\;\alpha^2(t)\,\nabla_{\xx_t}\log p_t(\xx_t)}
        {\gamma(t)} \,.
    \]

    \medskip
    \noindent\textbf{Derivation of \(\mathbb{E}[\zz\mid\xx_t]\) for $\mf^{\zz}_\star(\mmF_t,\xx_t,t)$.}
    From \(\xx_t=\alpha(t)\,\zz+\gamma(t)\,\xx\) we solve
    \(\zz=(\xx_t-\gamma(t)\,\xx)/\alpha(t)\).  Taking conditional
    expectation and substituting the above,
    \[
        \begin{aligned}
            \mathbb{E}[\zz\mid\xx_t]
             & = \frac{1}{\alpha(t)}\Bigl(\xx_t
            -\gamma(t)\,\mathbb{E}[\xx\mid\xx_t]\Bigr) \\
             & = \frac{1}{\alpha(t)}\Bigl(\xx_t
            -\gamma(t)\,\frac{\xx_t+\alpha^2(t)\,\nabla_{\xx_t}\log p_t(\xx_t)}
                {\gamma(t)}
            \Bigr)
            = -\,\alpha(t)\,\nabla_{\xx_t}\log p_t(\xx_t) \,.
        \end{aligned}
    \]
    Thus, according to \lemref{lem:opt_preders}, we can obtain
    \[
        \mf^{\zz}_\star(\mmF_t,\xx_t,t)
        = -\,\alpha(t)\,\nabla_{\xx_t}\log p_t(\xx_t) \,.
    \]

    \medskip
    \noindent\textbf{Combine to obtain the claimed identity.}
    \[
        \begin{aligned}
             & \;\alpha'(t)\,\mf^{\zz}_\star(\mmF_t,\xx_t,t)
            +\gamma'(t)\,\mf^{\xx}_\star(\mmF_t,\xx_t,t)                       \\[4pt]
             & = \alpha'(t)\bigl[-\alpha(t)\nabla_{\xx_t}\log p_t(\xx_t)\bigr]
            + \gamma'(t)\,
            \frac{\xx_t + \alpha^2(t)\,\nabla_{\xx_t}\log p_t(\xx_t)}
            {\gamma(t)}                                                        \\[4pt]
             & = -\,\alpha(t)\,\alpha'(t)\,\nabla_{\xx_t}\log p_t(\xx_t)
            + \frac{\gamma'(t)}{\gamma(t)}\,\xx_t
            + \frac{\gamma'(t)}{\gamma(t)}\,\alpha^2(t)\,
            \nabla_{\xx_t}\log p_t(\xx_t)                                      \\[4pt]
             & = \frac{\gamma'(t)}{\gamma(t)}\,\xx_t
            \;-\;
            \Bigl[
                \alpha(t)\,\alpha'(t)
                \;-\;
                \frac{\gamma'(t)}{\gamma(t)}\,\alpha^2(t)
                \Bigr]\,
            \nabla_{\xx_t}\log p_t(\xx_t) \,.
        \end{aligned}
    \]
    This matches the claimed formula.
\end{proof}

\begin{remark}[Velocity field of the flow ODE]
    Given $\xx$ and $\zz$, the field $\vv^{(\zz,\xx)}(\yy,t)=\alpha'(t) \zz + \gamma'(t) \xx $ could transport $\zz$ to $\xx$,
    so the velocity field of the flow ODE can be computed as
    \begin{align*}
        \vv^*(\xx_t,t) & = \E_{(\zz,\xx)|\xx_t}{\left[\vv^{(\zz,\xx)}(\xx_t,t) |\xx_t\right]}                                                                 \\
                       & = \E_{(\zz,\xx)|\xx_t}{\left[\alpha'(t) \zz + \gamma'(t) \xx |\xx_t\right]}                                                          \\
                       & = \alpha'(t) \cdot \E{}{\left[\zz |\xx_t\right]} + \gamma'(t) \cdot \E{}{\left[\xx |\xx_t\right]}                                    \\
                       & = \alpha'(t) \cdot \mf^{\mathrm{\zz}}_{\star}(\mmF_t, \xx_t, t) + \gamma'(t) \cdot \mf^{\mathrm{\xx}}_{\star}(\mmF_t, \xx_t, t) \, .
    \end{align*}
\end{remark}

\subsubsection{Closed-form Solution Analysis when $\lambda = 0$}
\label{app:close_lambda_0}

\begin{corollary}[Closed‐form PF–ODE for an arbitrary Gaussian mixture in $\mathbb R^d$]
    \label{cor:pf-ode-gmm}
    Let
    \[
        p(\mathbf x)
        \;=\;
        \sum_{j=1}^K w_j\,
        \mathcal N\bigl(\mathbf x;\,\boldsymbol m_j,\,
        \boldsymbol\Sigma_j\bigr),
        \quad
        w_j>0,\;\sum_j w_j=1,
    \]
    be a Gaussian‐mixture density on $\mathbb R^d$.  Let $\alpha,\gamma$ satisfy the hypotheses of \lemref{lem:pf-ode}, and define the forward map
    \[
        \mathbf x_t
        = \alpha(t)\,\mathbf z + \gamma(t)\,\mathbf x,
        \quad
        \mathbf x\sim p(\mathbf x),\;\mathbf z\sim\mathcal N(\mathbf0,\mathbf I).
    \]
    For each component $j$ set
    \[
        \boldsymbol\mu_j(t)
        = \gamma(t)\,\boldsymbol m_j,
        \qquad
        \boldsymbol\Sigma_j(t)
        = \gamma(t)^2\,\boldsymbol\Sigma_j \;+\;\alpha(t)^2\,\mathbf I,
        \qquad
        \phi_j(\mathbf x_t)
        =\mathcal N\bigl(\mathbf x_t;\,\boldsymbol\mu_j(t),\boldsymbol\Sigma_j(t)\bigr)
    \]
    so that
    \[
        p_t(\mathbf x_t)
        = \sum_{j=1}^K w_j\,
        \mathcal N\bigl(\mathbf x_t;\,\boldsymbol\mu_j(t),\,
        \boldsymbol\Sigma_j(t)\bigr).
    \]
    Then the Probability‐Flow ODE \eqref{eq:pf-ode-final} admits the closed‐form
    drift
    \[
        \frac{\mathrm d\mathbf x_t}{\mathrm d t}
        = \frac{\gamma'(t)}{\gamma(t)}\,\mathbf x_t
        \;+\;
        \Bigl[\,
            \alpha(t)\,\alpha'(t)
            \;-\;
            \frac{\gamma'(t)}{\gamma(t)}\,\alpha(t)^2
            \Bigr]
        \sum_{j=1}^K
        \frac{w_j\,
            \phi_j(\mathbf x_t)}
        {\,p_t(\mathbf x_t)\,}
        \;\boldsymbol\Sigma_j(t)^{-1}\,
        \bigl(\mathbf x_t - \boldsymbol\mu_j(t)\bigr).
    \]
\end{corollary}
\begin{proof}
    Step 1. \emph{Affine transform of a Gaussian mixture.}
    Conditioned on the $j$-th component, $\mathbf x\sim\mathcal N(\boldsymbol m_j,\boldsymbol\Sigma_j)$, and hence
    \[
        \mathbf x_t = \alpha(t)\,\mathbf z + \gamma(t)\,\mathbf x
        \;\Bigl|\;(j)\Bigr.
        \sim
        \mathcal N\bigl(\gamma(t)\,\boldsymbol m_j,\;\alpha(t)^2\mathbf I+\gamma(t)^2\boldsymbol\Sigma_j\bigr).
    \]
    Defining
    \[
        \boldsymbol\mu_j(t)=\gamma(t)\,\boldsymbol m_j,
        \quad
        \boldsymbol\Sigma_j(t)=\gamma(t)^2\,\boldsymbol\Sigma_j+\alpha(t)^2\,\mathbf I,
    \]
    we conclude that the marginal of $\mathbf x_t$ is
    \[
        p_t(\mathbf x_t)
        =\sum_{j=1}^K w_j\,
        \mathcal N\bigl(\mathbf x_t;\,\boldsymbol\mu_j(t),\boldsymbol\Sigma_j(t)\bigr).
    \]

    Step 2. \emph{Score of the mixture.}
    Set
    \[
        \phi_j(\mathbf x_t)
        =\mathcal N\bigl(\mathbf x_t;\,\boldsymbol\mu_j(t),\boldsymbol\Sigma_j(t)\bigr),
        \quad
        p_t(\mathbf x_t)=\sum_{j=1}^K w_j\,\phi_j(\mathbf x_t).
    \]
    Then by the usual mixture‐rule,
    \[
        \nabla_{\mathbf x_t}\log p_t
        =\frac{1}{p_t(\mathbf x_t)}
        \sum_{j=1}^K w_j\,\phi_j(\mathbf x_t)\,
        \nabla_{\mathbf x_t}\log\phi_j(\mathbf x_t).
    \]
    Since for each Gaussian component
    \[
        \nabla_{\mathbf x_t}\log\phi_j(\mathbf x_t)
        = -\,\boldsymbol\Sigma_j(t)^{-1}
        \bigl(\mathbf x_t-\boldsymbol\mu_j(t)\bigr),
    \]
    we obtain the closed‐form score
    \[
        \nabla_{\mathbf x_t}\log p_t(\mathbf x_t)
        = -\frac{1}{p_t(\mathbf x_t)}
        \sum_{j=1}^K w_j\,
        \mathcal N\bigl(\mathbf x_t;\,\boldsymbol\mu_j(t),\boldsymbol\Sigma_j(t)\bigr)\,
        \boldsymbol\Sigma_j(t)^{-1}\,
        \bigl(\mathbf x_t-\boldsymbol\mu_j(t)\bigr).
    \]

    Step 3. \emph{Substitution into the PF–ODE.}
    By \lemref{lem:pf-ode}, the Probability–Flow ODE reads
    \[
        \frac{\mathrm d\mathbf x_t}{\mathrm d t}
        = \frac{\gamma'(t)}{\gamma(t)}\,\mathbf x_t
        \;-\;
        \Bigl[\,
            \alpha(t)\,\alpha'(t)
            \;-\;
            \frac{\gamma'(t)}{\gamma(t)}\,\alpha(t)^2
            \Bigr]\,
        \nabla_{\mathbf x_t}\log p_t(\mathbf x_t).
    \]
    Substituting the expression for $\nabla\log p_t$ above (and observing that the two “$-$” signs cancel) yields
    \[
        \frac{\mathrm d\mathbf x_t}{\mathrm d t}
        = \frac{\gamma'(t)}{\gamma(t)}\,\mathbf x_t
        \;+\;
        \Bigl[\,
            \alpha(t)\,\alpha'(t)
            -\frac{\gamma'(t)}{\gamma(t)}\,\alpha(t)^2
            \Bigr]
        \sum_{j=1}^K
        \frac{w_j\,
            \mathcal N\bigl(\mathbf x_t;\,\boldsymbol\mu_j(t),\boldsymbol\Sigma_j(t)\bigr)}
        {p_t(\mathbf x_t)}
        \;\boldsymbol\Sigma_j(t)^{-1}\,
        \bigl(\mathbf x_t-\boldsymbol\mu_j(t)\bigr),
    \]
    which is exactly the claimed closed‐form drift.
\end{proof}
\begin{corollary}[Closed‐form PF–ODE for a symmetric two‐peak Gaussian mixture]
    \label{cor:pf-ode-bimodal}
    Let \(p(x)\) be the one‐dimensional, symmetric, two‐peak Gaussian mixture
    \[
        p(x)
        =\tfrac12\,\mathcal N\bigl(x; -m,\sigma^2\bigr)
        +\tfrac12\,\mathcal N\bigl(x; +m,\sigma^2\bigr),
    \]
    and let \(\alpha,\gamma\) be as in \lemref{lem:pf-ode}.  Define
    \[
        x_t = \alpha(t)\,z + \gamma(t)\,x,
        \qquad
        \Sigma_t = \gamma(t)^2\,\sigma^2 + \alpha(t)^2,
        \qquad
        \mu_\pm(t)=\pm\,\gamma(t)\,m.
    \]
    Then the marginal density of \(x_t\) is
    \[
        p_t(x_t)
        = \tfrac12\,\mathcal N\bigl(x_t;\mu_-(t),\Sigma_t\bigr)
        + \tfrac12\,\mathcal N\bigl(x_t;\mu_+(t),\Sigma_t\bigr),
    \]
    and the Probability‐Flow ODE \eqref{eq:pf-ode-final} admits the closed‐form drift
    \[
        \frac{\mathrm d x_t}{\mathrm d t}
        = \frac{\gamma'(t)}{\gamma(t)}\,x_t
        + \Bigl[\,
            \alpha(t)\,\alpha'(t)
            \;-\;
            \frac{\gamma'(t)}{\gamma(t)}\,\alpha(t)^2
            \Bigr]
        \frac{1}{\Sigma_t}
        \Bigl[
            x_t
            \;-\;
            \gamma(t)\,m\,
            \tanh\!\Bigl(\frac{\gamma(t)\,m}{\Sigma_t}\,x_t\Bigr)
            \Bigr].
    \]
\end{corollary}

\begin{proof}
    Step 1. \emph{Marginal law under the affine map.}  Conditional on
    \(x=\pm m\), one has
    \[
        x_t=\alpha z+\gamma x \;\Big|\;(x=\pm m)\;\sim\;
        \mathcal N\bigl(\pm\gamma m,\;\alpha^2+\gamma^2\sigma^2\bigr)
        =\mathcal N\bigl(\mu_\pm(t),\Sigma_t\bigr).
    \]
    Since each peak has weight \(\tfrac12\), the marginal of \(x_t\) is
    \(\tfrac12\mathcal N(\mu_-,\Sigma_t)+\tfrac12\mathcal N(\mu_+,\Sigma_t)\).

    Step 2. \emph{Score of the bimodal mixture.}  Write
    \(\phi_\pm(x_t)=\mathcal N(x_t;\mu_\pm(t),\Sigma_t)\), so
    \(p_t=\tfrac12(\phi_-+\phi_+)\).  Then
    \[
        \frac{\dm}{\dm x_t}\,\log p_t
        =\frac{1}{p_t}\,
        \tfrac12\bigl(\phi_-\,\nabla\log\phi_- + \phi_+\,\nabla\log\phi_+\bigr),
        \qquad
        \nabla\log\phi_\pm
        = -\,\frac{x_t-\mu_\pm(t)}{\Sigma_t}.
    \]
    Hence
    \[
        \frac{\dm}{\dm x_t}\log p_t
        = -\,\frac{1}{2\,p_t\,\Sigma_t}
        \bigl[\phi_-(x_t-\mu_-)+\phi_+(x_t-\mu_+)\bigr].
    \]
    Define
    \[
        r_\pm(x_t)
        = \frac{\phi_\pm(x_t)}{\phi_-(x_t)+\phi_+(x_t)},
        \quad
        \phi_-+\phi_+=2\,p_t.
    \]
    Then
    \[
        \frac{\dm}{\dm x_t}\log p_t
        = -\,\frac{1}{\Sigma_t}\,
        \bigl[r_-(x_t-\mu_-)+r_+(x_t-\mu_+)\bigr].
    \]
    A direct computation shows
    \[
        r_+ - r_- = \tanh\!\Bigl(\frac{\gamma m}{\Sigma_t}\,x_t\Bigr),
        \quad
        r_-(x_t+\gamma m)+r_+(x_t-\gamma m)
        = x_t - \gamma m\,\tanh\!\Bigl(\frac{\gamma m}{\Sigma_t}\,x_t\Bigr).
    \]
    Therefore
    \[
        \frac{\dm}{\dm x_t}\log p_t
        = -\,\frac{1}{\Sigma_t}
        \Bigl[
            x_t
            -\gamma m\,\tanh\!\Bigl(\tfrac{\gamma m}{\Sigma_t}\,x_t\Bigr)
            \Bigr].
    \]

    Step 3. \emph{Substitution into the PF–ODE.}  By \lemref{lem:pf-ode},
    \[
        \frac{\dm x_t}{\dm t}
        = \frac{\gamma'}{\gamma}\,x_t
        -\Bigl[\,
            \alpha\,\alpha'
            -\frac{\gamma'}{\gamma}\,\alpha^2
            \Bigr]
        \frac{\dm}{\dm x_t}\log p_t.
    \]
    Since \(\tfrac{\dm}{\dm x_t}\log p_t\) carries a “\(-)\)” sign, the two
    negatives cancel, yielding exactly
    \[
        \frac{\dm x_t}{\dm t}
        = \frac{\gamma'}{\gamma}\,x_t
        +\Bigl[\,
            \alpha\,\alpha'
            -\frac{\gamma'}{\gamma}\,\alpha^2
            \Bigr]
        \frac{1}{\Sigma_t}
        \Bigl[
            x_t
            -\gamma m\,\tanh\!\Bigl(\tfrac{\gamma m}{\Sigma_t}\,x_t\Bigr)
            \Bigr],
    \]
    as claimed.
\end{proof}

\begin{remark}[OU‐type schedule for the symmetric bimodal case]
    \label{rmk:ou-schedule-general}
    Specialize ~\corref{cor:pf-ode-bimodal} to the Ornstein–Uhlenbeck‐type schedule with
    \[
        \gamma(t)=e^{-s t},
        \qquad
        \alpha(t)=\sqrt{1-e^{-2s t}},
    \]
    and noise variance \(\sigma^2\) in each mixture component.  Then the marginal variance is
    \[
        \Sigma_t
        = \gamma(t)^2\,\sigma^2 \;+\;\alpha(t)^2
        = \sigma^2 e^{-2s t} + (1-e^{-2s t}),
    \]
    and one obtains the closed‐form drift of the Probability‐Flow ODE:
    \[
        \boxed{
        \frac{\dm x_t}{\dm t}
        = -\,s\,x_t
        + \frac{s}{\Sigma_t}
        \Bigl[
        x_t
        - m\,e^{-s t}\,
        \tanh\!\Bigl(\frac{m\,e^{-s t}}{\Sigma_t}\,x_t\Bigr)
        \Bigr].
        }
    \]
\end{remark}

\begin{proof}
    We start from the general drift in~\corref{cor:pf-ode-bimodal}:
    \[
        \frac{\dm x_t}{\dm t}
        = \frac{\gamma'}{\gamma}\,x_t
        + \Bigl[
            \alpha\,\alpha'
            -\frac{\gamma'}{\gamma}\,\alpha^2
            \Bigr]
        \frac{1}{\Sigma_t}
        \Bigl[
            x_t
            -\gamma\,m\,
            \tanh\!\Bigl(\frac{\gamma\,m}{\Sigma_t}\,x_t\Bigr)
            \Bigr].
    \]
    We now substitute \(\gamma(t)=e^{-s t}\), \(\alpha(t)=\sqrt{1-e^{-2s t}}\) and compute each piece in detail:

    Derivative of \(\gamma\):
    \[
        \gamma'(t)
        = -s\,e^{-s t},
        \quad
        \Longrightarrow
        \quad
        \frac{\gamma'(t)}{\gamma(t)}
        = -s.
    \]

    Marginal variance \(\Sigma_t\):
    \[
        \Sigma_t
        = \gamma(t)^2\,\sigma^2 + \alpha(t)^2
        = \sigma^2\,e^{-2s t} + (1-e^{-2s t}).
    \]

    Square of \(\alpha\) and its derivative:
    \[
        \alpha(t)^2
        = 1 - e^{-2s t},
        \qquad
        \frac{\dm}{\dm t}\bigl[\alpha(t)^2\bigr]
        = 2s\,e^{-2s t}
        \;\Longrightarrow\;
        2\,\alpha\,\alpha'
        = 2s\,e^{-2s t}
        \;\Longrightarrow\;
        \alpha\,\alpha' = s\,e^{-2s t}.
    \]

    Combination term
    \[
        \alpha\,\alpha'
        - \frac{\gamma'}{\gamma}\,\alpha^2
        = s\,e^{-2s t} - (-s)\,(1-e^{-2s t})
        = s\bigl[e^{-2s t} + 1 - e^{-2s t}\bigr]
        = s.
    \]

    Substitution into the general drift formula gives
    \[
        \frac{\dm x_t}{\dm t}
        = -s\,x_t
        + s\;\frac{1}{\Sigma_t}
        \Bigl[
            x_t
            - e^{-s t}\,m
            \,\tanh\!\Bigl(\tfrac{e^{-s t}\,m}{\Sigma_t}\,x_t\Bigr)
            \Bigr].
    \]

    Hence the final, closed‐form Probability‐Flow ODE is
    \[
        \frac{\dm x_t}{\dm t}
        = -\,s\,x_t
        + \frac{s}{\Sigma_t}
        \Bigl[
        x_t
        - m\,e^{-s t}
        \tanh\!\Bigl(\frac{m\,e^{-s t}}{\Sigma_t}\,x_t\Bigr)
        \Bigr],
    \]
    where
    \(\Sigma_t = \sigma^2 e^{-2s t} + (1 - e^{-2s t})\).
\end{proof}

\begin{remark}[Triangular schedule for the symmetric bimodal case]
    \label{rmk:sin-cos-schedule-bimodal}
    Specialize~\corref{cor:pf-ode-bimodal} to the trigonometric schedule
    \[
        \gamma(t)=\cos\!\Bigl(\tfrac\pi2\,t\Bigr),
        \qquad
        \alpha(t)=\sin\!\Bigl(\tfrac\pi2\,t\Bigr),
    \]
    with noise variance \(\sigma^2\) in each mixture component.  Then
    \[
        \Sigma_t
        = \gamma(t)^2\,\sigma^2 + \alpha(t)^2
        = \sigma^2\cos^2\!\Bigl(\tfrac\pi2\,t\Bigr)
        + \sin^2\!\Bigl(\tfrac\pi2\,t\Bigr),
    \]
    and the closed‐form drift of the Probability‐Flow ODE is
    \[
        \boxed{
            \frac{\dm x_t}{\dm t}
            = -\frac\pi2\,\tan\!\Bigl(\tfrac\pi2\,t\Bigr)\,x_t
            + \frac{\tfrac\pi2\,\tan\!\bigl(\tfrac\pi2\,t\bigr)}{\Sigma_t}
            \Bigl[
                x_t
                - \cos\!\Bigl(\tfrac\pi2\,t\Bigr)\,m\,
                \tanh\!\Bigl(\frac{\cos(\frac\pi2 t)\,m}{\Sigma_t}\,x_t\Bigr)
                \Bigr].
        }
    \]
\end{remark}

\begin{proof}
    We begin with the general drift in~\corref{cor:pf-ode-bimodal}:
    \[
        \frac{\dm x_t}{\dm t}
        = \frac{\gamma'}{\gamma}\,x_t
        + \Bigl[\,
            \alpha\,\alpha'
            - \frac{\gamma'}{\gamma}\,\alpha^2
            \Bigr]
        \frac{1}{\Sigma_t}
        \Bigl[
            x_t
            - \gamma\,m
            \,\tanh\!\Bigl(\frac{\gamma\,m}{\Sigma_t}\,x_t\Bigr)
            \Bigr].
    \]
    For \(\gamma(t)=\cos(\tfrac\pi2 t)\), \(\alpha(t)=\sin(\tfrac\pi2 t)\),
    \[
        \gamma'(t)=-\tfrac\pi2\,\sin\!\Bigl(\tfrac\pi2\,t\Bigr)=-\tfrac\pi2\,\alpha(t),
        \quad
        \frac{\gamma'}{\gamma}=-\tfrac\pi2\,\tan\!\Bigl(\tfrac\pi2\,t\Bigr).
    \]
    And
    \[
        \alpha'(t)=\tfrac\pi2\,\cos\!\Bigl(\tfrac\pi2\,t\Bigr)=\tfrac\pi2\,\gamma(t),
    \]
    so that
    \[
        \alpha\,\alpha'
        - \frac{\gamma'}{\gamma}\,\alpha^2
        = \frac\pi2\,\alpha\,\gamma
        + \frac\pi2\,\frac{\alpha^3}{\gamma}
        = \frac\pi2\,\frac{\alpha}{\gamma}
        \bigl(\alpha^2 + \gamma^2\bigr)
        = \frac\pi2\,\tan\!\Bigl(\tfrac\pi2\,t\Bigr).
    \]
    Substituting into the general formula immediately yields the boxed drift.
\end{proof}

\begin{remark}[Linear schedule for the symmetric bimodal case]
    \label{rmk:linear-schedule}
    Specialize \corref{cor:pf-ode-bimodal} to the "Linear" schedule
    \[
        \gamma(t)=1-t,
        \qquad
        \alpha(t)=t,
        \quad t\in[0,1].
    \]
    Then the marginal variance is
    \[
        \Sigma_t
        =\gamma(t)^2\,\sigma^2 + \alpha(t)^2
        =(1-t)^2\,\sigma^2 + t^2,
    \]
    and one obtains the closed‐form drift of the Probability‐Flow ODE:
    \[
        \boxed{
            \frac{\mathrm d x_t}{\mathrm d t}
            = -\frac{x_t}{1-t}
            + \frac{t}{(1-t)\,\Sigma_t}
            \Bigl[
                x_t
                - m\,(1-t)\,
                \tanh\!\Bigl(\frac{m\,(1-t)}{\Sigma_t}\,x_t\Bigr)
                \Bigr].
        }
    \]
\end{remark}

\begin{proof}
    We begin with the general drift formula from \corref{cor:pf-ode-bimodal}:
    \[
        \frac{\mathrm d x_t}{\mathrm d t}
        = \frac{\gamma'(t)}{\gamma(t)}\,x_t
        + \Bigl[\,\alpha(t)\,\alpha'(t)
            - \frac{\gamma'(t)}{\gamma(t)}\,\alpha(t)^2
            \Bigr]
        \frac{1}{\Sigma_t}
        \Bigl[
            x_t
            - \gamma(t)\,m\,
            \tanh\!\Bigl(\frac{\gamma(t)\,m}{\Sigma_t}\,x_t\Bigr)
            \Bigr].
    \]
    We substitute \(\gamma(t)=1-t\) and \(\alpha(t)=t\) and compute each piece:

    1. Derivative of \(\gamma\):
    \[
        \gamma'(t) = -1,
        \quad\Longrightarrow\quad
        \frac{\gamma'(t)}{\gamma(t)} = -\frac{1}{1-t}.
    \]

    2. Marginal variance:
    \[
        \Sigma_t
        = (1-t)^2\,\sigma^2 + t^2.
    \]

    3. Square of \(\alpha\) and its derivative:
    \[
        \alpha(t)^2 = t^2,
        \quad
        \frac{\mathrm d}{\mathrm d t}\bigl[\alpha(t)^2\bigr]
        = 2t
        \;\Longrightarrow\;
        2\,\alpha\,\alpha' = 2t
        \;\Longrightarrow\;
        \alpha(t)\,\alpha'(t) = t.
    \]

    4. Combination term:
    \[
        \alpha\,\alpha'
        - \frac{\gamma'}{\gamma}\,\alpha^2
        = t - \Bigl(-\frac{1}{1-t}\Bigr)\,t^2
        = t + \frac{t^2}{1-t}
        = \frac{t}{1-t}.
    \]

    Substituting these into the general drift gives
    \[
        \frac{\mathrm d x_t}{\mathrm d t}
        = -\frac{x_t}{1-t}
        + \frac{t}{(1-t)\,\Sigma_t}
        \Bigl[
            x_t
            - m\,(1-t)\,
            \tanh\!\Bigl(\frac{m\,(1-t)}{\Sigma_t}\,x_t\Bigr)
            \Bigr],
    \]
    which is the claimed closed‐form Probability‐Flow ODE.
\end{proof}

\begin{remark}[OU‐type schedule for the Hermite–Gaussian \(n=1\) case]
    Apply \lemref{lem:pf-ode} to the one‐dimensional Hermite–Gaussian initial density
    \[
        p_1(x)\;\propto\;x\,e^{-x^2/2},\quad x>0,
    \]
    and the OU‐type schedule
    \[
        \gamma(t)=e^{-s t},\qquad \alpha(t)=\sqrt{1-e^{-2s t}}.
    \]
    Then the Probability–Flow ODE \eqref{eq:pf-ode-final} reduces to the scalar form
    \[
        \boxed{
            \frac{\dm x_t}{\dm t} \;=\; -\,\frac{s}{x_t},
            \quad t\in[0,1],
        }
    \]
    and integrating from \(t=1\) (with \(x(1)=x_1\)) to any \(t\in[0,1]\) yields the explicit solution
    \[
        \boxed{
            x_t \;=\;\sqrt{x_1^2 \;+\; 2\,s\,(1-t)}\,.}
    \]
\end{remark}

\begin{proof}
    By \lemref{lem:pf-ode}, the drift of the Probability–Flow ODE is
    \[
        \frac{\dm x_t}{\dm t}
        = \frac{\gamma'(t)}{\gamma(t)}\,x_t
        \;-\;
        \Bigl[\,
            \alpha(t)\,\alpha'(t)
            -\tfrac{\gamma'(t)}{\gamma(t)}\,\alpha(t)^2
            \Bigr]\,
        \partial_{x_t}\ln p_t(x_t).
    \]
    Under \(\gamma(t)=e^{-st}\) and \(\alpha(t)=\sqrt{1-e^{-2st}}\) one computes
    \[
        \frac{\gamma'}{\gamma}=-s,\quad
        2\,\alpha\,\alpha'=2s\,e^{-2st}
        \;\Longrightarrow\;
        \alpha\,\alpha'=s\,e^{-2st},
        \quad
        -\frac{\gamma'}{\gamma}\,\alpha^2
        =s\,(1-e^{-2st}),
    \]
    hence
    \[
        \alpha\,\alpha' \;-\;\frac{\gamma'}{\gamma}\,\alpha^2
        =s\,e^{-2st} + s\,(1-e^{-2st}) = s.
    \]
    Moreover, one checks that the marginal density remains
    \(\;p_t(x)\propto x\,e^{-x^2/2}\),
    so
    \(\partial_x\ln p_t(x)=\tfrac1{x}-x\).  Therefore
    \[
        \frac{\dm x_t}{\dm t}
        = -\,s\,x_t \;-\;s\,\Bigl(\tfrac1{x_t}-x_t\Bigr)
        = -\,\frac{s}{x_t}.
    \]
    Separating variables,
    \[
        \frac{\dm x}{\dm t}=-\frac{s}{x}
        \quad\Longrightarrow\quad
        \int_{x_1}^{\,x_t} x\,dx
        = -s\!\int_{1}^{\,t} ds
        \;\Longrightarrow\;
        \frac{x_t^2 - x_1^2}{2} = -s\,(t-1),
    \]
    whence
    \[
        x_t^2 = x_1^2 + 2\,s\,(1-t),
        \quad
        x_t = \sqrt{x_1^2 + 2\,s\,(1-t)},
    \]
    taking the positive root on \(x>0\).
\end{proof}

\begin{lemma}[Picard–Lindelöf existence and uniqueness]
    \label{lem:PL}
    Let $v\colon\R\times[0,1]\to\R$ be continuous in $t$ and satisfy the \emph{uniform Lipschitz} condition
    \[
        |v(x,t)-v(y,t)|\;\le\;L\,|x-y|,
        \quad\forall\,x,y\in\R,\;t\in[0,1],
    \]
    for some constant $L<\infty$.  Then for any $t_0\in[0,1]$ and any initial value $x(t_0)=x_0$, there exists $\delta>0$ and a unique function
    \[
        x\in C^1\bigl([t_0-\delta,t_0+\delta]\cap[0,1]\bigr)
    \]
    solving the ODE
    \[
        \frac{\dm x}{\dm t}(t)\;=\;v\bigl(x(t),t\bigr),
        \quad
        x(t_0)=x_0.
    \]
\end{lemma}

\begin{proof}
    Fix $t_0\in[0,1]$ and $x_0\in\R$.  Choose $\delta>0$ so small that $(t_0-\delta,t_0+\delta)\subset[0,1]$ and $L\delta<1$.  Define the closed ball
    \[
        B_R=\bigl\{x\in C([t_0-\delta,t_0+\delta],\R):\|x-x_0\|_\infty\le R\bigr\}
    \]
    with $R>0$ to be chosen.  Consider the operator
    \[
        (\Gamma x)(t)
        =x_0+\int_{t_0}^t v\bigl(x(s),s\bigr)\,\dm s.
    \]
    Since $v$ is continuous on the compact set $B_R\times[t_0-\delta,t_0+\delta]$, it is bounded by some $M<\infty$.  If we choose $R=M\delta$, then $\Gamma$ maps $B_R$ into itself:
    \[
        \|\Gamma x - x_0\|_\infty
        \le \sup_{t}\int_{t_0}^t|v(x(s),s)|\,\dm s
        \le M\,\delta = R.
    \]
    Moreover, for any $x,y\in B_R$ and any $t$ in the interval,
    \[
        |(\Gamma x)(t)-(\Gamma y)(t)|
        \le \int_{t_0}^t |v(x(s),s)-v(y(s),s)|\,\dm s
        \le L\,\delta\,\|x-y\|_\infty
        < \|x-y\|_\infty,
    \]
    so $\Gamma$ is a contraction.  By the Banach fixed‐point theorem, $\Gamma$ has a unique fixed point in $B_R$, which is precisely the unique $C^1$ solution of the ODE on $[t_0-\delta,t_0+\delta]\cap[0,1]$.
\end{proof}

\begin{lemma}[Gronwall’s inequality and no blow‐up]
    \label{lem:Gronwall}
    Let \(x\in C^1([0,1])\) satisfy
    \[
        |x'(t)| \;\le\; K\bigl(1 + |x(t)|\bigr),
        \quad t\in[0,1],
    \]
    for some constant \(K\ge0\).  Then
    \[
        |x(t)| \;\le\; \bigl(|x(1)| + 1\bigr)\,e^{K(1-t)} \;-\; 1,
        \quad \forall\,t\in[0,1],
    \]
    and in particular \(x\) does not blow up in finite time on \([0,1]\).
\end{lemma}

\begin{proof}
    Define
    \[
        y(t) \;=\; |x(t)| + 1 \;\ge\; 1.
    \]
    Since \(y(t)\) is Lipschitz, for almost every \(t\) we have
    \[
        y'(t)
        = \frac{\dm}{\dm t}\bigl(|x(t)|+1\bigr)
        = \mathrm{sgn}(x(t))\,x'(t),
    \]
    and hence
    \[
        y'(t)
        \;\ge\; -\,|x'(t)|
        \;\ge\;-\,K\bigl(1+|x(t)|\bigr)
        \;=\;-K\,y(t).
    \]
    Equivalently,
    \[
        y'(t) + K\,y(t)\;\ge\;0.
    \]
    Multiply both sides by the integrating factor \(e^{Kt}\):
    \[
        \frac{\dm}{\dm t}\bigl(e^{Kt}y(t)\bigr)
        = e^{Kt}\bigl(y'(t)+K\,y(t)\bigr)
        \;\ge\;0.
    \]
    Thus the function \(t\mapsto e^{Kt}y(t)\) is non‐decreasing on \([0,1]\).  For any \(t\le1\) we then have
    \[
        e^{Kt}y(t)\;\le\;e^{K\cdot1}y(1)
        \quad\Longrightarrow\quad
        y(t)\;\le\;y(1)\,e^{K(1-t)}
        =\bigl(|x(1)|+1\bigr)\,e^{K(1-t)}.
    \]
    Rewriting \(y(t)=|x(t)|+1\) gives
    \[
        |x(t)| \;\le\; \bigl(|x(1)|+1\bigr)e^{K(1-t)} \;-\; 1,
    \]
    as claimed.  In particular \(|x(t)|<\infty\) for all \(t\in[0,1]\), so no finite‐time blow‐up occurs.
\end{proof}

\begin{lemma}[Gaussian convolution preserves linear‐growth bound]
    \label{lem:gauss-conv}
    Let $p_0\in C^1(\mathbb R)$ be a probability density satisfying
    \[
        \bigl|\partial_x\log p_0(x)\bigr|\;\le\;A + B\,|x|,
        \quad A,B<\infty,\ \forall x\in\R,
    \]
    and assume furthermore that
    $
        \|p_0\|_\infty \;=\;\sup_{x\in\R}p_0(x)\;\le\;M<\infty.
    $
    For each $\sigma>0$, define the Gaussian kernel
    $
        \phi_\sigma(u)
        =\frac1{\sqrt{2\pi}\,\sigma}
        \exp\!\Bigl(-\tfrac{u^2}{2\sigma^2}\Bigr),
    $
    and set
    $
        p_\sigma(x)\;=\;(p_0*\phi_\sigma)(x)
        =\int_{\!\R}p_0(y)\,\phi_\sigma(x-y)\,\mathrm{d}y.
    $
    Then $p_\sigma\in C^\infty(\R)$ and there exist
    \[
        A(\sigma)=A + B\,M\,\sigma\sqrt{\tfrac{2}{\pi}},
        \quad
        B(\sigma)=B,
    \]
    such that
    \[
        \bigl|\partial_x\log p_\sigma(x)\bigr|
        \;\le\;A(\sigma)\;+\;B(\sigma)\,|x|,
        \quad\forall x\in\R.
    \]
\end{lemma}

\begin{proof}
    \emph{Smoothness and differentiation under the integral.}
    Since $\phi_\sigma\in C^\infty(\R)$ decays rapidly and $p_0\in L^\infty(\R)$,
    by dominated convergence we may differentiate under the integral to get
    \[
        p_\sigma'(x)
        =\int_{\R}p_0(y)\,\partial_x\phi_\sigma(x-y)\,\mathrm{d}y
        =\int_{\R}p_0(y)\,\phi_\sigma'(x-y)\,\mathrm{d}y.
    \]
    Noting $\partial_y\phi_\sigma(x-y)=-\phi_\sigma'(x-y)$, we rewrite
    \[
        p_\sigma'(x)
        = -\int_{\R}p_0(y)\,\partial_y\phi_\sigma(x-y)\,\mathrm{d}y.
    \]

    \emph{Integration by parts.}
    Integrating the above in $y$ and using that
    $p_0(y)\phi_\sigma(x-y)\to0$ as $|y|\to\infty$, we obtain
    \[
        p_\sigma'(x)
        =\int_{\R}p_0'(y)\,\phi_\sigma(x-y)\,\mathrm{d}y
        =\int_{\R}(\partial_y\log p_0)(y)\,p_0(y)\,\phi_\sigma(x-y)\,\mathrm{d}y.
    \]

    \emph{Bounding $\partial_x\log p_\sigma$.}
    Hence
    \begin{align*}
        \bigl|\partial_x\log p_\sigma(x)\bigr|
         & =\frac{|p_\sigma'(x)|}{p_\sigma(x)}
        =\frac{\bigl|\int(\partial_y\log p_0)(y)\,p_0(y)\,\phi_\sigma(x-y)\,\dm y\bigr|}{p_\sigma(x)} \\
         & \le\frac{{\int|\partial_y\log p_0(y)|\,p_0(y)\,\phi_\sigma(x-y)\,\dm y}}{p_\sigma(x)}
        \le\frac{{\int\bigl(A+B|y|\bigr)\,p_0(y)\,\phi_\sigma(x-y)\,\dm y}}{p_\sigma(x)}              \\
         & =A + B\,\frac{\int |y|\,p_0(y)\,\phi_\sigma(x-y)\,\dm y}{p_\sigma(x)}.
    \end{align*}

    \emph{Change of variables.}
    Set $u=y-x$. Then
    \[
        \int|y|\,p_0(y)\,\phi_\sigma(x-y)\,\dm y
        =\int|x+u|\,p_0(x+u)\,\phi_\sigma(u)\,\dm u
        \le|x|\,p_\sigma(x)+\int|u|\,p_0(x+u)\,\phi_\sigma(u)\,\dm u.
    \]
    Hence
    \[
        \frac{\int|y|\,p_0(y)\,\phi_\sigma(x-y)\,\dm y}{p_\sigma(x)}
        \le|x| + \frac{\int|u|\,p_0(x+u)\,\phi_\sigma(u)\,\dm u}{p_\sigma(x)}.
    \]

    \emph{Using the $L^\infty$‐bound on $p_0$.}
    Since $p_0(x+u)\le M$,
    \[
        \int|u|\,p_0(x+u)\,\phi_\sigma(u)\,\dm u
        \le M\int|u|\,\phi_\sigma(u)\,\dm u
        =M\,\sigma\sqrt{\tfrac{2}{\pi}}.
    \]

    \emph{Conclusion.}
    Combining the above estimates yields
    \[
        \bigl|\partial_x\log p_\sigma(x)\bigr|
        \le A + B\Bigl(|x|+M\,\sigma\sqrt{\tfrac{2}{\pi}}\Bigr)
        =\bigl[A+B\,M\,\sigma\sqrt{\tfrac{2}{\pi}}\bigr]+B\,|x|.
    \]
    Thus one may set
    \[
        A(\sigma)=A+B\,M\,\sigma\sqrt{\tfrac{2}{\pi}},\quad
        B(\sigma)=B,
    \]
    and the lemma follows.
\end{proof}

\begin{figure}[t]
    \centering
    \begin{subfigure}[b]{0.48\textwidth}
        \centering
        \includegraphics[width=\linewidth]{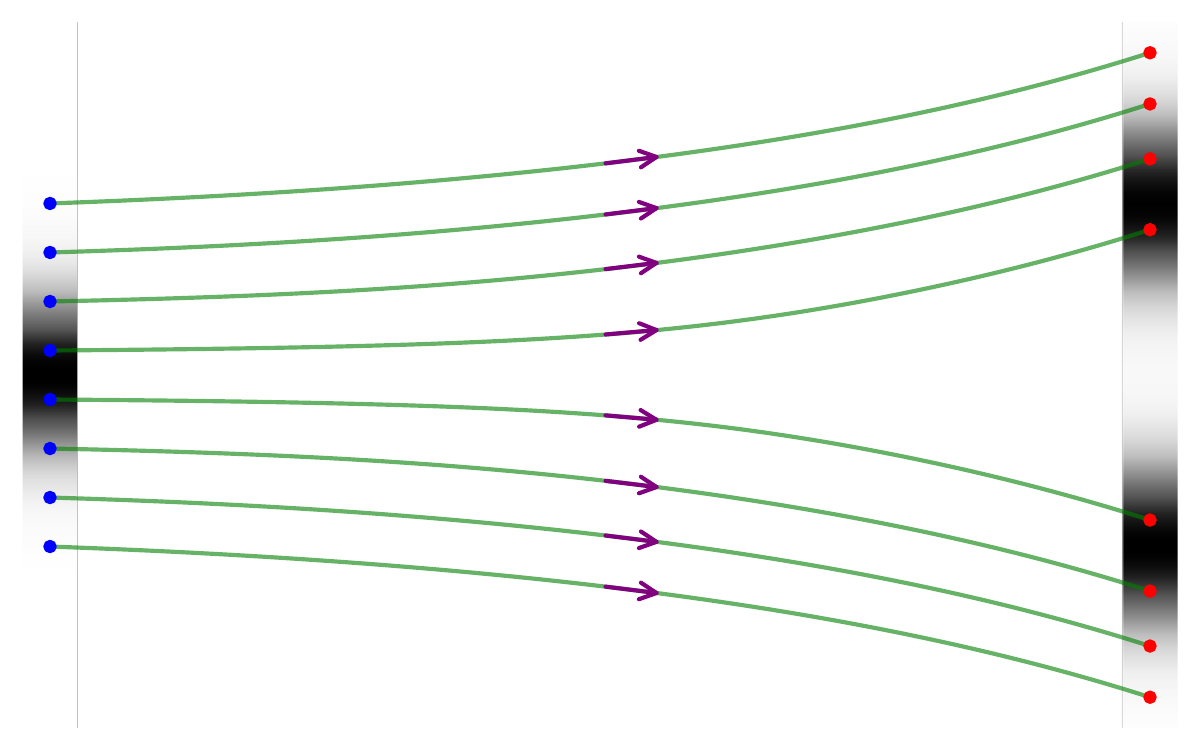}
        \caption{\textbf{OU-type.}}
    \end{subfigure}
    \hfill
    \begin{subfigure}[b]{0.48\textwidth}
        \centering
        \includegraphics[width=\linewidth]{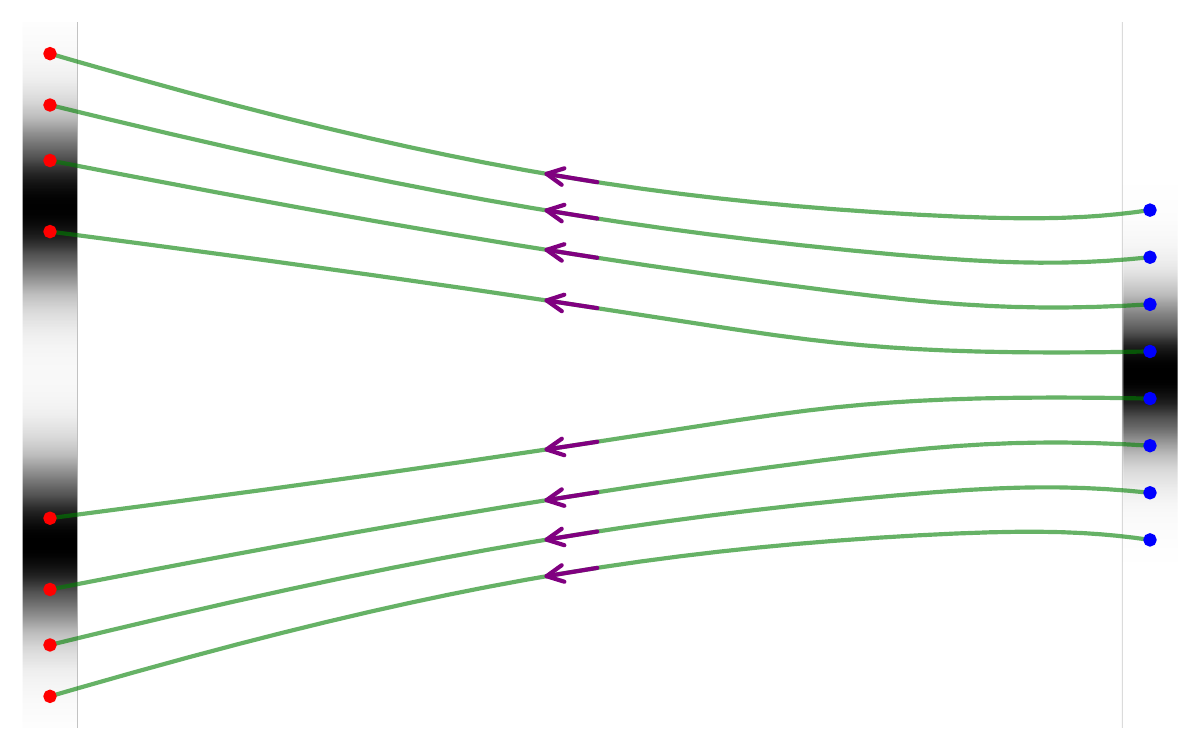}
        \caption{\textbf{Linear.}}
    \end{subfigure}
    \caption{\small{
            \textbf{Comparison of two optimal Probability-Flow ODE trajectories on 1D data.}
            Starting from identical initial noise distributions and noise points, we apply two distinct transport types—OU-type and Linear—to analyze their trajectories. The results show that both types successfully converge to the same target distribution (a bimodal Gaussian) and accurately match the \textit{same} target data points, despite following different ODE paths.
        }}
    \label{fig:visualize_1d_traj}
\end{figure}

\begin{theorem}[Monotonicity and uniqueness of the 1D probability‐flow map]
    Let \(p_0(x)\) be a probability density on \(\R\) satisfying the linear‐growth bound
    \[
        \bigl|\partial_x\log p_0(x)\bigr|
        \;\le\;
        A + B\,|x|,
        \qquad
        A,B<\infty,\quad\forall x\in\R.
    \]
    Let \(z\sim\mathcal N(0,1)\) be independent of \(x_0\), and let
    \(\alpha,\gamma:[0,1]\to\R\) be \(C^1\) functions with
    \[
        \alpha(0)=0,\;\alpha(1)=1,
        \quad
        \gamma(0)=1,\;\gamma(1)=0,
        \quad
        \gamma(t)\neq0\ \forall t\in(0,1).
    \]
    Define the forward process
    \[
        x_t \;=\;\alpha(t)\,z\;+\;\gamma(t)\,x_0,
        \qquad t\in[0,1],
    \]
    so that \(x_0\sim p_0\) and \(x_1\sim\mathcal N(0,1)\).  Let \(p_t\) denote
    the density of \(x_t\).  By~\lemref{lem:pf-ode}, the velocity field:
    \[
        v(x,t)
        = \frac{\gamma'(t)}{\gamma(t)}\,x
        \;-\;
        \Bigl[\,
            \alpha(t)\,\alpha'(t)
            \;-\;
            \frac{\gamma'(t)}{\gamma(t)}\,\alpha(t)^2
            \Bigr]\,
        \partial_x\log p_t(x).
    \]
    Consider the backward ODE
    $
        \frac{\dm}{\dm t}\,x_t = v\bigl(x_t,t\bigr)
    $,
    Then for each \(x_1\in\R\) there is a unique \(C^1\) solution \(t\mapsto x_t(x_1)\) on \([0,1]\), and the map
    \[
        g(x_1)\;=\;x_0(x_1)
        \;=\;
        F_0^{-1}\!\bigl(F_1(x_1)\bigr)
    \]
    is strictly increasing on \(\R\) and is the unique increasing transport
    pushing \(p_1\) onto \(p_0\).
\end{theorem}

\begin{proof}
    {(1) Global existence and uniqueness.}
    Since
    \[
        x_t=\alpha(t)\,z+\gamma(t)\,x_0,
        \quad
        p_t=p_0*\mathcal N\bigl(0,\alpha(t)^2\bigr),
    \]
    standard Gaussian‐convolution estimates imply
    \(\bigl|\partial_x\log p_t(x)\bigr|\le A_t+B_t|x|\) for some continuous
    \(A_t,B_t\) (cf., \lemref{lem:gauss-conv}).  Hence there exists \(K<\infty\) such that
    \[
        |v(x,t)|\;\le\;K\,(1+|x|),
        \qquad
        \bigl|\partial_x v(x,t)\bigr|\;\le\;K,
        \quad
        \forall\,x\in\R,\;t\in[0,1].
    \]
    In particular \(v\) is globally Lipschitz in \(x\) (uniformly in \(t\)) and
    of linear growth.
    By the \lemref{lem:PL} together with \lemref{lem:Gronwall} to prevent finite‐time blow‐up, the backward ODE
    admits for each \(x_1\) a unique \(C^1\) solution on \([0,1]\).

    \smallskip
    {(2) Conservation of the CDF.}
    Let
    \[
        F_t(x)=\int_{-\infty}^x p_t(u)\,\dm u
        \quad(\text{the CDF of }p_t).
    \]
    Since \(p_t\) satisfies the continuity equation
    \(\partial_t p_t+\partial_x(v\,p_t)=0\), along any characteristic
    \(t\mapsto x_t\) one computes
    \[
        \frac{\dm}{\dm t}F_t(x_t)
        = \int_{-\infty}^{x_t}\partial_t p_t(u)\,\dm u
        +p_t(x_t)\,\frac{\dm x_t}{\dm t}
        = -\bigl[v\,p_t\bigr]_{-\infty}^{x_t}
        +p_t(x_t)\,v(x_t,t)
        = 0,
    \]
    using \(\lim_{u\to-\infty}p_t(u)=0\).  Hence
    \(\,F_t(x_t)=F_1(x_1)\) for all \(t\in[0,1]\).

    \smallskip
    {(3) Quantile representation.}
    Evaluating at \(t=0\) gives
    \[
        F_0\bigl(x_0(x_1)\bigr)
        = F_1(x_1).
    \]
    Since \(F_0\colon\R\to(0,1)\) is strictly increasing and onto, it has an
    inverse \(F_0^{-1}\), and thus
    \[
        x_0(x_1)
        = F_0^{-1}\!\bigl(F_1(x_1)\bigr).
    \]

    \smallskip
    {(4) Monotonicity and uniqueness.}
    If \(x_1<y_1\) then \(F_1(x_1)<F_1(y_1)\), so
    \[
        g(x_1)
        = F_0^{-1}\bigl(F_1(x_1)\bigr)
        < F_0^{-1}\bigl(F_1(y_1)\bigr)
        = g(y_1),
    \]
    showing \(g\) is strictly increasing.  In one dimension the strictly
    increasing transport between two given laws is unique, so \(g\) is the
    unique increasing map pushing \(p_1\) onto \(p_0\).
    A case study presented in \figref{fig:visualize_1d_traj} validates this theorem, considering the specific schedules discussed in \remarkref{rmk:linear-schedule} and \remarkref{rmk:ou-schedule-general}.
\end{proof}

\begin{lemma}[Monotone transport from Gaussian to \(P\)]
    Let \(Z\sim N(0,1)\) be a standard normal random variable and let \(X\) be a random variable with distribution \(P\) on \(\mathbb{R}\), having cumulative distribution function (CDF) \(F_P\).  Define
    \[
        \Phi(z)=\Pr[Z\le z],
        \qquad
        F_P^{-1}(u)=\inf\{x:F_P(x)\ge u\},\;u\in(0,1).
    \]
    Then there exists a non‐decreasing continuous function
    $
        g(z)=F_P^{-1}\bigl(\Phi(z)\bigr)
    $
    such that \(g(Z)\overset{d}{=}X\) if and only if \(P\) has no atoms (i.e.\ \(F_P\) is continuous).  Moreover, if \(F_P\) is strictly increasing then \(g\) is unique.
\end{lemma}

\begin{proof}
    \emph{Existence.}
    Since \(\Phi:\mathbb{R}\to(0,1)\) is continuous and strictly increasing, the random variable
    \[
        U = \Phi(Z)
    \]
    is distributed uniformly on \((0,1)\).  Hence for any \(x\in\mathbb{R}\),
    \[
        \Pr\bigl(F_P^{-1}(U)\le x\bigr)
        = \Pr\bigl(U \le F_P(x)\bigr)
        = F_P(x),
    \]
    so \(F_P^{-1}(U)\) has distribution \(P\).  The quantile function \(F_P^{-1}\) is non‐decreasing and, by standard results on generalized inverses (see e.g.\ Billingsley, \emph{Probability and Measure}), is continuous on \((0,1)\) if and only if \(F_P\) is continuous.  Therefore
    \[
        g(z)=F_P^{-1}\bigl(\Phi(z)\bigr)
    \]
    is non‐decreasing and continuous exactly when \(F_P\) is continuous, and in that case \(g(Z)\overset{d}{=}X\).

    \emph{Necessity.}
    Suppose \(P\) has an atom at \(x_0\), i.e.\ \(\Pr[X=x_0]=p>0\).  If there were a continuous non‐decreasing \(g\) with \(g(Z)\overset{d}{=}X\), then to produce a point‐mass \(p\) at \(x_0\) it would have to be constant on a set of positive \(\Pr\)-mass in the continuous law of \(Z\).  But continuity of \(g\) then forces it to be constant on a strictly larger interval, yielding a mass \(>p\) at \(x_0\), a contradiction.  Thus \(F_P\) must be continuous.

    \emph{Uniqueness.}
    Let \(g_1,g_2\) be two continuous non‐decreasing functions with \(g_i(Z)\overset{d}{=}P\).  Define for \(u\in(0,1)\)
    \[
        h_i(u)\;=\;g_i\bigl(\Phi^{-1}(u)\bigr),\qquad i=1,2.
    \]
    Each \(h_i\) is continuous, non‐decreasing, and pushes \(\mathrm{Unif}(0,1)\) onto \(P\).  When \(F_P\) is strictly increasing, its quantile \(F_P^{-1}\) is the unique such map (classical uniqueness of quantile functions for atomless laws).  Hence \(h_1\equiv h_2\equiv F_P^{-1}\) on \((0,1)\), and therefore \(g_1\equiv g_2\) on \(\mathbb{R}\).
\end{proof}

\subsubsection{Learning Objective as $\lambda \to 1$}
\label{app:lambda_to_1}
\begin{lemma}[$L^p$‐estimate for the difference of two absolutely continuous functions]
    \label{lem:lp_estimate_diff}
    Let \(I=[a,b]\) be a compact interval and \((E,\|\cdot\|)\) a Banach space.  Suppose
    $
        f,g : I \to E
    $
    are absolutely continuous with Bochner–integrable derivatives \(f',g'\).  Fix \(1\le p\le\infty\).  Then
    \[
        \|\,f - g\|_{L^p(I;E)}
        \;\le\;
        (b-a)^{1/p}\,\bigl\|f(a)-g(a)\bigr\|
        \;+\;
        \int_a^b (b-s)^{1/p}\,\bigl\|f'(s)-g'(s)\bigr\|\;\mathrm{d}s,
    \]
    where for \(p=\infty\) one interprets \((b-s)^{1/p}=1\).  Moreover, if \(1<p<\infty\) and \(p'\) denotes the conjugate exponent \(1/p+1/p'=1\), then by Hölder’s inequality one further deduces
    \[
        \|\,f - g\|_{L^p(I;E)}
        \;\le\;
        (b-a)^{1/p}\,\bigl\|f(a)-g(a)\bigr\|
        \;+\;
        \Bigl(\tfrac{p-1}{p}\Bigr)^{1/p'}\,(b-a)\,\|\,f'-g'\|_{L^p(I;E)}.
    \]
\end{lemma}

\begin{proof}
    Since \(f\) and \(g\) are absolutely continuous on \([a,b]\), the Fundamental Theorem of Calculus in the Bochner setting gives, for each \(t\in[a,b]\),
    \[
        f(t)-g(t)
        \;=\;
        \bigl(f(a)-g(a)\bigr)
        +
        \int_a^t \bigl(f'(s)-g'(s)\bigr)\,\mathrm{d}s.
    \]
    Set \(X(s)=f'(s)-g'(s)\).  Then for every \(t\in[a,b]\),
    \[
        \bigl\|f(t)-g(t)\bigr\|
        \;\le\;
        \bigl\|f(a)-g(a)\bigr\|
        \;+\;
        \Bigl\|\int_a^t X(s)\,\mathrm{d}s\Bigr\|.
    \]
    We now distinguish two cases.

    \medskip
    \noindent\emph{Case 1: \(1\le p<\infty\).}
    Taking the \(L^p\)–norm in the variable \(t\) over \([a,b]\) and applying Minkowski’s integral inequality for Bochner integrals yields
    \begin{align*}
        \|\,f-g\|_{L^p_t}
         & \;\le\;
        \bigl\|f(a)-g(a)\bigr\|\;\bigl\|1\bigr\|_{L^p([a,b])}
        \;+\;
        \Bigl\|\int_a^t X(s)\,\mathrm{d}s\Bigr\|_{L^p_t}                                   \\
         & =
        (b-a)^{1/p}\,\bigl\|f(a)-g(a)\bigr\|
        \;+\;
        \Bigl(\int_a^b \Bigl\|\int_a^t X(s)\,\mathrm{d}s\Bigr\|^p\,\mathrm{d}t\Bigr)^{1/p} \\
         & \;\le\;
        (b-a)^{1/p}\,\bigl\|f(a)-g(a)\bigr\|
        \;+\;
        \int_a^b
        \bigl\|1_{[s,b]}(\cdot)\,X(s)\bigr\|_{L^p_t}
        \,\mathrm{d}s.
    \end{align*}
    Here we have written
    \(\int_a^t X(s)\,ds=\int_a^b1_{[a,t]}(s)\,X(s)\,ds\)
    and used the fact that
    \[
        \bigl\|1_{[s,b]}(t)\bigr\|_{L^p_t}
        =\Bigl(\int_a^b1_{[s,b]}(t)\,dt\Bigr)^{1/p}
        =(b-s)^{1/p}.
    \]
    Hence
    \[
        \|\,f-g\|_{L^p(I;E)}
        \;\le\;
        (b-a)^{1/p}\,\bigl\|f(a)-g(a)\bigr\|
        \;+\;
        \int_a^b(b-s)^{1/p}\,\bigl\|X(s)\bigr\|\,\mathrm{d}s,
    \]
    which is the claimed \(L^p\)–estimate.

    \medskip
    \noindent\emph{Case 2: \(p=\infty\).}
    Taking the essential supremum in \(t\in[a,b]\) in the pointwise bound
    \(\|f(t)-g(t)\|\le\|f(a)-g(a)\|+\int_a^t\|X(s)\|\,ds\)
    gives immediately
    \[
        \|\,f-g\|_{L^\infty(I;E)}
        \;\le\;
        \|f(a)-g(a)\|
        \;+\;
        \int_a^b\|X(s)\|\,\mathrm{d}s,
    \]
    which agrees with the above formula when \((b-s)^{1/p}=1\).

    \medskip
    \noindent\emph{Refinement for \(1<p<\infty\).}
    Let \(p'\) be the conjugate exponent, \(1/p+1/p'=1\).  Applying Hölder’s inequality to the integral
    \(\int_a^b(b-s)^{1/p}\,\|X(s)\|\;ds\) gives
    \[
        \int_a^b(b-s)^{1/p}\,\|X(s)\|\,ds
        \;\le\;
        \Bigl(\int_a^b (b-s)^{p'/p}\,ds\Bigr)^{1/p'}\,
        \Bigl(\int_a^b\|X(s)\|^p\,ds\Bigr)^{1/p}.
    \]
    Since \(p'/p=1/(p-1)\), a direct computation yields
    \[
        \int_a^b (b-s)^{p'/p}\,ds
        \;=\;
        \int_0^{\,b-a}u^{1/(p-1)}\,du
        \;=\;
        \frac{p-1}{p}\,(b-a)^{p'}.
    \]
    Hence
    \[
        \Bigl(\int_a^b (b-s)^{p'/p}\,ds\Bigr)^{1/p'}
        =\Bigl(\tfrac{p-1}{p}\Bigr)^{1/p'}\,(b-a),
    \]
    and we arrive at
    \[
        \int_a^b(b-s)^{1/p}\,\|X(s)\|\,ds
        \;\le\;
        \Bigl(\tfrac{p-1}{p}\Bigr)^{1/p'}\,(b-a)\,\|X\|_{L^p(I;E)}.
    \]
    Combining this with the previous display completes the proof of the refined estimate.
\end{proof}

\begin{lemma}[Uniqueness of absolutely continuous functions]\label{lem:uniqueness}
    Let \(I=[a,b]\) be a compact interval and \((E,\|\cdot\|)\) a Banach space.  Suppose
    $
        f,g : I \to E
    $
    are absolutely continuous with Bochner–integrable derivatives \(f',g'\).  If
    \[
        f(a)=g(a)
        \quad\text{and}\quad
        f'(t)=g'(t)\quad\text{for almost every }t\in I,
    \]
    then \(f(t)=g(t)\) for all \(t\in I\).
\end{lemma}

\begin{proof}
    Apply \lemref{lem:lp_estimate_diff} (the \(L^p\)–estimate for differences) in the case \(p=\infty\).  Since in this case one has
    \[
        \bigl(b-s\bigr)^{1/p}=1,
        \quad
        \|f(a)-g(a)\|=0,
        \quad
        \|f'(s)-g'(s)\|=0\;\text{a.e.},
    \]
    the conclusion of \lemref{lem:lp_estimate_diff} reads
    \[
        \|\,f-g\|_{L^\infty(I;E)}
        \;\le\;
        \|f(a)-g(a)\|
        +
        \int_a^b\|f'(s)-g'(s)\|\,ds
        =0.
    \]
    Hence
    \(\|f-g\|_{L^\infty(I;E)}=0\), which means
    \[
        \sup_{t\in I}\|f(t)-g(t)\|=0,
    \]
    so \(f(t)=g(t)\) for every \(t\in I\).
\end{proof}

\begin{theorem}[Pathwise consistency via zero total derivative]
    \label{thm:pathwise_consistency}
    Let \(p(\mathbf{x})\) be a data distribution on \(\mathbb{R}^d\), and let
    \(\mathbf{z}\sim\mathcal{N}(\mathbf{0},\mI_d)\) be independent of \(\mathbf{x}\).  Let
    \(\alpha,\gamma:[0,1]\to\mathbb{R}\) be \(C^1\) scalar functions satisfying
    \[
        \alpha(0)=0,\ \alpha(1)=1,
        \quad
        \gamma(0)=1,\ \gamma(1)=0,
        \quad
        \gamma(t)\neq0\ \forall t\in(0,1).
    \]
    Define the forward process
    \[
        \mathbf{x}_t \;=\; \alpha(t)\,\mathbf{z} \;+\;\gamma(t)\,\mathbf{x},
        \quad t\in[0,1],
    \]
    so that \(\mathbf{x}_0=\mathbf{x}\sim p(\mathbf{x})\) and
    \(\mathbf{x}_1=\mathbf{z}\sim\mathcal{N}(0,I)\).  Let \(p_t\) be the law of \(\mathbf{x}_t\).
    By~\lemref{lem:pf-ode} the corresponding Probability Flow ODE is
    \[
        \mathbf{v}(\mathbf{x}_t,t)
        \;=\;
        \frac{\dm}{\dm t}\,\mathbf{x}_t
        \;=\;
        \frac{\gamma'(t)}{\gamma(t)}\,\mathbf{x}_t
        \;-\;
        \Bigl[\,
            \alpha(t)\,\alpha'(t)
            \;-\;
            \frac{\gamma'(t)}{\gamma(t)}\,\alpha(t)^2
            \Bigr]\,
        \nabla_{\mathbf{x}_t}\,\log p_t(\mathbf{x}_t).
    \]
    Given any point \(\mathbf{x}_t\), define
    \[
        \mg(\mathbf{x}_t,t)
        \;=\;
        \mathbf{x}_0
        \;=\;
        \mathbf{x}_t \;+\;\int_{t}^{0}\mathbf{v}(\mathbf{x}_u,u)\,\mathrm{d}u.
    \]
    Let \((\zz,\xx)\sim p(\xx)\otimes\mathcal{N}(0,I)\) and \(t\sim\mathrm{Unif}[0,1]\) be all mutually independent.  Write
    \(\E_{(\zz,\xx)}\) for expectation over \((\zz,\xx)\) and
    \(\E_{(\zz,\xx),t}\) for expectation over \((\zz,\xx)\) and \(t\).  Suppose
    \[
        \E_{(\zz,\xx)}\bigl\Vert \mf(\xx_0,0)-\mg(\xx_0,0)\bigr\Vert \;=\;0,
        \quad
        \E_{(\zz,\xx),t}\Bigl\Vert \frac{\dm}{\dm t}\,\mf(\xx_t,t)\Bigr\Vert \;=\;0.
    \]
    Then
    \[
        \E_{(\zz,\xx),t}\bigl\Vert \mf(\xx_t,t)-\mg(\xx_t,t)\bigr\Vert \;=\;0.
    \]
\end{theorem}

\begin{proof}
    Fix a draw \((\zz,\xx)\).  Along its forward trajectory
    \(\mathbf{x}_t=\alpha(t)\zz+\gamma(t)\xx\), define the two curves
    \[
        f(t)=\mf\bigl(\mathbf{x}_t,t\bigr),
        \quad
        g(t)=\mg\bigl(\mathbf{x}_t,t\bigr).
    \]
    We check the hypotheses of~\lemref{lem:uniqueness} for \(f,g:[0,1]\to\mathbb{R}^d\).

    \emph{Absolute continuity.}
    Since \(\mf\) is \(C^1\) in \((\mathbf{x},t)\) and \(t\mapsto\mathbf{x}_t\) is \(C^1\),
    the composition \(f(t)=\mf(\mathbf{x}_t,t)\) is absolutely continuous, with
    \[
        f'(t)
        = \frac{\dm}{\dm t}\,\mf\bigl(\mathbf{x}_t,t\bigr),
        \quad\text{existing a.e.}
    \]
    Also
    \[
        g(t)
        = \mathbf{x}_t + \int_t^0 \mathbf{v}(\mathbf{x}_u,u)\,\mathrm{d}u
        = \mathbf{x}_0 - \int_0^t \mathbf{v}(\mathbf{x}_u,u)\,\mathrm{d}u
    \]
    is the sum of a \(C^1\) function and an absolutely continuous integral, hence itself absolutely continuous.

    \emph{Coincidence of initial values.}
    From
    \(\E_{(\zz,\xx)}\Vert \mf(\xx_0,0)-\mg(\xx_0,0)\Vert=0\) we get
    \(\mf(\xx_0,0)=\mg(\xx_0,0)\) almost surely, so \(f(0)=g(0)\) for almost every \((\zz,\xx)\).

    \emph{Coincidence of derivatives a.e.}
    By Tonelli–Fubini,
    \[
        0
        =\E_{(\zz,\xx),t}\Bigl\Vert \tfrac{\dm}{\dm t}\mf(\xx_t,t)\Bigr\Vert
        =\int\!\left(\int_0^1
        \Bigl\Vert \tfrac{\dm}{\dm t}\mf(\xx_t,t)\Bigr\Vert
        \,\mathrm{d}t\right)\,
        \mathrm{d}\mathbb{P}(\zz,\xx).
    \]
    Hence for almost every \((\zz,\xx)\),
    \(\int_0^1\|\partial_t\mf(\xx_t,t)\|\,\mathrm{d}t=0\), which forces
    \(\partial_t\mf(\xx_t,t)=0\) for almost all \(t\).  Thus
    \[
        f'(t)=0
        \quad\text{for a.e. }t\in[0,1].
    \]
    On the other hand
    \[
        g'(t)
        = \frac{\dm \mathbf{x}_t}{\dm t}
        - \mathbf{v}(\mathbf{x}_t,t)
        = \mathbf{v}(\mathbf{x}_t,t)
        - \mathbf{v}(\mathbf{x}_t,t)
        = 0,
        \quad\forall t\in[0,1].
    \]

    \emph{Conclusion by uniqueness.}
    We have shown \(f,g\) are absolutely continuous, \(f(0)=g(0)\), and
    \(f'(t)=g'(t)\) for almost every \(t\).  By~\lemref{lem:uniqueness},
    \(f(t)=g(t)\) for all \(t\in[0,1]\) (almost surely in \((\zz,\xx)\)).  Hence
    \(\mf(\xx_t,t)=\mg(\xx_t,t)\) a.s., and taking expectation yields
    $
        \E_{(\zz,\xx),t}\bigl\Vert \mf(\xx_t,t)-\mg(\xx_t,t)\bigr\Vert=0.
    $
\end{proof}

\begin{remark}[Consistency‐training loss]
    By~\thmref{thm:pathwise_consistency}, to enforce
    $
        \mf(\mathbf{x}_t,t)\approx \mg(\mathbf{x}_t,t)=\mathbf{x}_0
    $ along the PF–ODE flow,
    we suggests two equivalent training objectives:

    1. \emph{Continuous PDE‐residual loss}
    \[
        \mathcal{L}_{\rm PDE}
        = \mathbb{E}_{t,\mathbf{x}_t}
        \Bigl\|\partial_t \mf(\mathbf{x}_t,t)
        + v(\mathbf{x}_t,t)\!\cdot\!\nabla_{\mathbf{x}_t}\mf(\mathbf{x}_t,t)
        \Bigr\|^2.
    \]

    2. \emph{Finite‐difference consistency loss}
    \[
        \mathcal{L}_{\rm cons}
        = \mathbb{E}_{t,\mathbf{x}_0,\mathbf{z}}
        \Bigl\|
        \mf\bigl(\mathbf{x}_{t+\Delta t},\,t+\Delta t\bigr)
        - \mf\bigl(\mathbf{x}_t,t\bigr)
        \Bigr\|^2,
    \]
    where
    \(\mathbf{x}_t=\alpha(t)\mathbf{z}+\gamma(t)\mathbf{x}_0\)
    and similarly for \(\mathbf{x}_{t+\Delta t}\).
\end{remark}

\begin{proof}
    We begin from the requirement that \(\mf(\mathbf{x}_t,t)\) remain constant along the flow:
    \[
        \frac{\mathrm{d}}{\mathrm{d}t}\,\mf(\mathbf{x}_t,t)
        =\bigl(\partial_t + \vv(\mathbf{x}_t,t)\cdot\nabla_{\mathbf{x}_t}\bigr)\,\mf(\mathbf{x}_t,t)
        = \partial_t \mf(\mathbf{x}_t,t)
        + \underbrace{\frac{\mathrm{d}\mathbf{x}_t}{\mathrm{d}t}}_{=\,\vv(\mathbf{x}_t,t)}
        \cdot\nabla_{\mathbf{x}_t}\mf(\mathbf{x}_t,t)
        = 0.
    \]
    This is exactly the linear transport PDE
    \[
        (\partial_t + \vv\cdot\nabla)\,\mf(\mathbf{x},t)=0.
    \]
    To train a network \(\mf\) to satisfy it, one may minimize the $L^2$‐residual over the joint law of \(t\) and \(\mathbf{x}_t\), yielding
    \[
        \mathcal{L}_{\rm PDE}
        = \mathbb{E}_{t,\mathbf{x}_t}
        \Bigl\|
        \partial_t \mf(\mathbf{x}_t,t)
        + \vv(\mathbf{x}_t,t)\!\cdot\!\nabla_{\mathbf{x}_t}\mf(\mathbf{x}_t,t)
        \Bigr\|^2.
    \]
    In practice, computing the spatial gradient \(\nabla_{\mathbf{x}_t}\mf\) can be expensive.  Instead, we use a small time increment \(\Delta t\) and the finite‐difference approximation
    \[
        \mf(\mathbf{x}_{t+\Delta t},\,t+\Delta t)
        -\mf(\mathbf{x}_t,t)
        \;\approx\;
        \Delta t\,
        \bigl[\partial_t \mf
            + \vv\cdot\nabla \mf\bigr](\mathbf{x}_t,t).
    \]
    Squaring and taking expectations over \(t,\mathbf{x}_0,\mathbf{z}\) then yields the discrete consistency loss
    \[
        \mathcal{L}_{\rm cons}
        = \mathbb{E}_{t,\mathbf{x}_0,\mathbf{z}}
        \Bigl\|
        \mf\bigl(\mathbf{x}_{t+\Delta t},\,t+\Delta t\bigr)
        - \mf\bigl(\mathbf{x}_t,t\bigr)
        \Bigr\|^2.
    \]
    This completes the derivation of both forms of the consistency‐training objective.
\end{proof}

Recall that \((\zz,\xx)\sim p(\zz,\xx)\) is a pair of latent and data variables (typically independent), and let \(t\in[0,1]\).
We have four differentiable scalar functions
$
    \alpha,\gamma,\hat\alpha,\hat\gamma\colon[0,1]\to\mathbb{R}
$
, the \emph{noisy interpolant}
$
    \xx_t \;=\; \alpha(t)\,\zz \;+\; \gamma(t)\,\xx
$
and
$
    \mmF_t \;=\; \mmF_{\mtheta}(\xx_t,t)
$.
We define the \(\xx\)- and \(\zz\)-prediction functions by
\[
    \mf^{\xx}(\mmF_t,\xx_t,t) =
    \frac{\alpha(t)\,\mmF_t \;-\; \hat\alpha(t)\,\xx_t}
    {\alpha(t)\,\hat\gamma(t) \;-\; \hat\alpha(t)\,\gamma(t)}\,,
    \quad \text{and} \quad
    \mf^{\zz}(\mmF_t,\xx_t,t) =
    \frac{\hat\gamma(t)\,\xx_t \;-\; \gamma(t)\,\mmF_t}
    {\alpha(t)\,\hat\gamma(t) \;-\; \hat\alpha(t)\,\gamma(t)}\,.
\]
Since
\begin{align*}
    \mf^{\mathrm{\xx}}(\mmF_0, \xx_0, 0) & = \frac{\alpha(0)\cdot\mmF_{\mtheta}(\xx_0,0) - \hat{\alpha}(0)\cdot\xx_0}{\alpha(0)\cdot \hat{\gamma}(0) - \hat{\alpha}(0) \cdot \gamma(0)} \\
                                         & = \frac{0 \cdot\mmF_{\mtheta}(\xx_0,0) - \hat{\alpha}(0)\cdot\xx_0}{0 \cdot \hat{\gamma}(0) - \hat{\alpha}(0) \cdot 1}                       \\
                                         & = \frac{\boldsymbol{0} - \hat{\alpha}(0)\cdot\xx_0}{0 - \hat{\alpha}(0) }                                                                    \\
                                         & = \xx_0
\end{align*}
$\mf^{\mathrm{\xx}}$ satisfies the boundary condition of consistency models~\citep{song2023consistency} and \thmref{thm:pathwise_consistency}.
To better understand the unified loss, let's analyze a bit further.
For simplicity we use the notation $\mf_{\mtheta}(\xx_t, t) :=\mf^{\mathrm{\xx}}(\mmF_{\mtheta}(\xx_t, t), \xx_t, t)$,
the training objective is then equal to
\[
    \cL(\mtheta)=\E_{t,(\zz,\xx)}{\left[ \frac{1}{\hat{\omega}(t)}\| \mf_{\mtheta}(\xx_t, t) - \mf_{\mtheta^-}(\xx_{\lambda t}, \lambda t) \|_2^2\right]} \, .
\]
Let $\phi_t(\xx)$ be the solution of the PF-ODE determined by the velocity field $\vv^*(\xx_t,t)=\E_{(\zz,\xx)|\xx_t}{\left[\vv^{(\zz,\xx)}(\xx_t,t) |\xx_t\right]}$
(where $\vv^{(\zz,\xx)}(\yy,t)=\alpha'(t) \zz + \gamma'(t) \xx $) and an initial value $\xx$ at time $t=0$.
Define $\mg_{\mtheta}(\xx,t):=\mf_{\mtheta}(\phi_t(\xx), t)$ that moves along the solution trajectory.
When $\lambda \rightarrow 1$, the gradient of the loss tends to
\begin{align*}
    \lim_{\lambda \rightarrow 1} \nabla_{\mtheta} \frac{\cL(\mtheta)}{2(1-\lambda)}
     & = \E_{t}{\left[\frac{t}{\hat{\omega}(t)} \cdot \E_{(\zz,\xx)}{\lim_{\lambda \rightarrow 1} \langle \frac{\mf_{\mtheta}(\xx_t, t) - \mf_{\mtheta}(\xx_{\lambda t}, \lambda t)}{t-\lambda t} ,\nabla_{\mtheta} \mf_{\mtheta}(\xx_t, t)\rangle } \right]} \\
     & = \E_{t}{\left[\frac{t}{\hat{\omega}(t)} \cdot \E_{(\zz,\xx)}{ \langle \frac{\dm \mf_{\mtheta}(\xx_t, t) }{\dm t} , \nabla_{\mtheta} \mg_{\mtheta}(\phi_t^{-1}(\xx_t),t)\rangle } \right]}
\end{align*}
The inner expectation can be computed as:
\begin{align*}
     & \E_{(\zz,\xx),\xx_t}{ \langle \frac{\dm \mf_{\mtheta}(\xx_t, t) }{\dm t} , \nabla_{\mtheta} \mg_{\mtheta}(\phi_t^{-1}(\xx_t),t)\rangle }                                                                                                      \\
     & = \E_{(\zz,\xx),\xx_t}{ \langle \partial_1 \mf_{\mtheta}(\xx_t,  t) \cdot \vv^{(\zz,\xx)}(\xx_t,t)+\partial_2 \mf_{\mtheta}(\xx_t,  t) , \nabla_{\mtheta} \mg_{\mtheta}(\phi_t^{-1}(\xx_t),t)\rangle }                                        \\
     & = \E_{(\zz,\xx),\xx_t}{ \langle \partial_1 \mf_{\mtheta}(\xx_t,  t) \cdot (\alpha'(t) \zz + \gamma'(t) \xx)+\partial_2 \mf_{\mtheta}(\xx_t,  t) , \nabla_{\mtheta} \mg_{\mtheta}(\phi_t^{-1}(\xx_t),t)\rangle }                               \\
     & = \E_{\xx_t}{\left[ \E_{(\zz,\xx)|\xx_t}{\langle \partial_1 \mf_{\mtheta}(\xx_t,  t) \cdot (\alpha'(t) \zz + \gamma'(t) \xx)+\partial_2 \mf_{\mtheta}(\xx_t,  t) , \nabla_{\mtheta} \mg_{\mtheta}(\phi_t^{-1}(\xx_t),t)\rangle } \right]}     \\
     & = \E_{\xx_t}{ \langle \partial_1 \mf_{\mtheta}(\xx_t,  t) \cdot \E_{(\zz,\xx)|\xx_t}{\left[\alpha'(t) \zz + \gamma'(t) \xx |\xx_t\right]}+\partial_2 \mf_{\mtheta}(\xx_t,  t) , \nabla_{\mtheta} \mg_{\mtheta}(\phi_t^{-1}(\xx_t),t)\rangle } \\
     & = \E_{\xx_t}{ \langle \partial_1 \mf_{\mtheta}(\xx_t,  t) \cdot \vv^*(\xx_t,t)+\partial_2 \mf_{\mtheta}(\xx_t,  t) , \nabla_{\mtheta} \mg_{\mtheta}(\phi_t^{-1}(\xx_t),t) \rangle }                                                           \\
     & = \E_{\xx_t}{ \langle \partial_2 \mg_{\mtheta}(\phi_t^{-1}(\xx_t),t), \nabla_{\mtheta} \mg_{\mtheta}(\phi_t^{-1}(\xx_t),t) \rangle }                                                                                                          \\
     & = \nabla_{\mtheta} \E_{\phi_t^{-1}(\xx_t)}{\frac{1}{2} \|\mg_{\mtheta}(\phi_t^{-1}(\xx_t),t)-\mg_{\mtheta^-}(\phi_t^{-1}(\xx_t),t)+\partial_2 \mg_{\mtheta}(\phi_t^{-1}(\xx_t),t)\|^2_2}
\end{align*}
Thus from the perspective of gradient, when $\lambda \rightarrow 1$ the training objective is equivalent to
\[
    \E_{\phi_t^{-1}(\xx_t),t}{\left[\frac{t}{\hat{\omega}(t)} \cdot \|\mg_{\mtheta}(\phi_t^{-1}(\xx_t),t)-\mg_{\mtheta^-}(\phi_t^{-1}(\xx_t),t)+\partial_2 \mg_{\mtheta}(\phi_t^{-1}(\xx_t),t)\|^2_2\right]}
\]
which naturally leads to the solution $\mg_{\mtheta}(\xx,t)=\xx$ (since $\mg_{\mtheta}(\xx,0)\equiv\xx$), or equivalently $\mf^{\mathrm{\xx}}(\mmF_{\mtheta^*}(\xx_t,t),\xx_t,t)=\mf_{\mtheta^*}(\xx_t,  t)=\phi_t^{-1}(\xx_t)$, that is the definition of consistency function.

\subsubsection{Analysis on the Optimal Solution for $\lambda \in [0,1]$}
\label{app:optimal_solution}

Below we provide some examples to illustrate the property of the optimal solution for the unified loss by considering some simple cases of data distribution.

(for simplicity define $\mf_{\mtheta}(\xx_t, t)=\mf^{\mathrm{\xx}}(\mmF_{\mtheta}(\xx_t, t), \xx_t, t)$)

Assume $\xx\sim \mathcal N(\mmu,\Sigma)$.  For $r<t$ the conditional mean
\[
    \EE{}{\left[ \xx_r|\xx_t \right]} = \gamma(r)\mmu + \left( \gamma(r)\gamma(t)\Sigma + \alpha(r)\alpha(t)\mI \right) \left( \gamma(t)^2\Sigma + \alpha(t)^2 \mI \right)^{-1} \left( \xx_t - \gamma(t)\mmu \right) \, ,
\]
denote
\[
    \mK(r,t) := \left( \gamma(r)\gamma(t)\Sigma + \alpha(r)\alpha(t)\mI \right) \left( \gamma(t)^2\Sigma + \alpha(t)^2 \mI \right)^{-1} \, ,
\]
using above equations we can get the optimal solution for diffusion model:
\[
    \mf_{\mtheta^*}^{\mathrm{DM}}(\xx_t, t)= \EE{}{\left[ \xx|\xx_t \right]} = \mmu  + \mK(0,t) (\xx_t-\gamma_t\mmu) \, .
\]

Now consider a series of $t$ together: $t=t_T>t_{T-1}>\ldots >t_1>t_{0} \approx 0$. This series could be obtained by $t_{j-1}=\lambda \cdot t_j,j=T,\ldots,0$, for instance.
With an abuse of notation, denote $\xx_{t_{j}}$ as $\xx_j$ and $\alpha(t_j)$ as $\alpha_j$, $\gamma(t_j)$ as $\gamma_j$.
Since $t_0\approx 0, \xx_0 \approx \xx$,
we could conclude the trained model $\mf_{\mtheta^*}(\xx_1, t_1)=\EE{\xx|\xx_1}{\left[\xx|\xx_1\right]}$, and concequently
\[
    \mf_{\mtheta^*}(\xx_{j+1}, t_{j+1})=\EE{\xx_{j}|\xx_{j+1}}{\left[\mf_{\mtheta^*}(\xx_j,t_j)|\xx_{j+1}\right]},\, j=1,\ldots,T-1 \, .
\]
Using the property of the conditional expectation, we have $\EE{\xx_{j}}{\left[\mf_{\mtheta^*}(\xx_j,t_j)\right]}=\EE{\xx}{\left[\xx\right]},\forall j$.
Using the expressions above we have
\[
    \mf_{\mtheta^*}(\xx_1, t_1)= \mmu + \mK(t_0,t_1) (\xx_1-\gamma_1\mmu)
\]
and
\[
    \mf_{\mtheta^*}(\xx_{j}, t_{j})= \mmu + \left[\prod_{k=1}^j \mK(t_{k-1},t_k)\right]\cdot (\xx_t-\gamma_t\mmu),\, j=2,\ldots,T
\]
Further denote $c_j= \prod_{k=1}^j \alpha_{k-1}\alpha_k+\gamma_{k-1}\gamma_k$ and assume $\Sigma=\mI,\alpha=\sin(t),\gamma(t)=\cos(t)$.
For appropriate choice of the partition scheme (e.g. even or geometric), the coefficient $c_j$ can converge as $T$ grows.
For instance, when evenly partitioning the interval $[0,t]$, we have:
\[
    \lim_{T\rightarrow \infty} c(t) = \lim_{T\rightarrow \infty} \prod_{k=1}^T \alpha_{k-1}\alpha_k+\gamma_{k-1}\gamma_k = \lim_{T\rightarrow \infty} (\cos(\frac{t}{T}))^T = 1 \, .
\]
Thus the trained model can be viewed as an interpolant between the consistency model($\lambda \rightarrow 1$ or $T \rightarrow \infty$) and the diffusion model($\lambda \rightarrow 0$ or $T \rightarrow 1$):
\[
    \mf_{\mtheta^*}(\xx_t, t)= \mmu + c(t) (\xx_t-\gamma(t) \mmu) \,,
\]
\[
    \mf_{\mtheta^*}^{\mathrm{CM}}(\xx_t, t)= \mmu + (\xx_t-\gamma(t) \mmu) \, ,
\]
\[
    \mf_{\mtheta^*}^{\mathrm{DM}}(\xx_t, t)= \mmu + \gamma(t) (\xx_t-\gamma(t) \mmu) \, .
\]

The expression of $\mf_{\mtheta^*}^{\mathrm{CM}}$ can be obtained by first compute the velocity field $\vv^*(\xx_t,t)=\EE{}{\left[\alpha'(t) \zz + \gamma'(t) \xx |\xx_t\right]}=\gamma'(t) \mmu$ then solve the initial value problem of ODE to get $\xx(0)$.

The above optimal solution can be possibly obtained by training.
For example if we set the parameterizition as $\mf_{\mtheta}(\xx_t, t)=(1-\gamma_t c_t)\mtheta+ c_t \xx_t$,
the gradient of the loss can be computed as (let $r=\lambda \cdot t$):

\begin{equation*}
    \nabla_{\mtheta} { \| \mf_{\mtheta}(\xx_t, t) - \mf_{\mtheta^-}(\xx_r, r) \|_2^2}
    = 2(1-\gamma_t c_t)\left[(\alpha_t\gamma_t-\alpha_r\gamma_r)\zz+(\gamma_r c_r-\gamma_t c_t)(\mtheta-\xx)\right] \, ,
\end{equation*}

\begin{equation*}
    \nabla_{\mtheta} \EE{\zz,\xx}{ \| \mf_{\mtheta}(\xx_t, t) - \mf_{\mtheta^-}(\xx_r, r) \|_2^2}
    = 2(1-\gamma_t c_t)(\gamma_r c_r-\gamma_t c_t)(\mtheta-\mmu) \, ,
\end{equation*}

\begin{align*}
    \nabla_{\mtheta} \cL(\mtheta) & = \EE{t}{\frac{2(1-\gamma_t c_t)(\gamma_r c_r-\gamma_t c_t)}{\hat{\omega}(t)}(\mtheta-\mmu)}              \\
                                  & = C(\mtheta-\mmu),\, C= \EE{t}{\frac{2(1-\gamma_t c_t)(\gamma_r c_r-\gamma_t c_t)}{\hat{\omega}(t)}} \, .
\end{align*}

Use gradient descent to update $\mtheta$ during training:

\begin{equation*}
    \frac{d\mtheta(s)}{ds}= -\nabla_{\mtheta} \cL(\mtheta) = -C(\mtheta-\mmu) \, .
\end{equation*}

The generalization loss thus evolves as:

\begin{align*}
    \frac{d\|\mtheta(s)-\mmu\|^2}{ds} & = \langle \mtheta(s)-\mmu,\frac{d\mtheta(s)}{ds}\rangle \\
                                      & = \langle \mtheta(s)-\mmu,-C(\mtheta(s)-\mmu)\rangle    \\
                                      & = -C\|\mtheta(s)-\mmu\|^2 \, ,
\end{align*}
\begin{equation*}
    \Longrightarrow \|\mtheta(s)-\mmu\|^2 = \|\mtheta(0)-\mmu\|^2 e^{-Cs} \, .
\end{equation*}

\subsubsection{Enhanced Target Score Function}
\label{app:enhanced_target_score}
Recall that CFG proposes to modify the sampling distribution as
\[
    \tilde{p}_\theta(\xx_t|\cc) \propto p_\theta(\xx_t|\cc)p_\theta(\cc|\xx_t)^\zeta \, ,
\]
Bayesian rule gives
\[
    p_\theta(\cc|\xx_t) = \frac{p_\theta(\xx_t|\cc)p_\theta(\cc)}{p_\theta(\xx_t)} \, ,
\]
so we can futher deduce
\begin{align*}
    \tilde{p}_\theta(\xx_t|\cc) & \propto p_\theta(\xx_t|\cc)p_\theta(\cc|\xx_t)^\zeta                                  \\
                                & = p_\theta(\xx_t|\cc)(\frac{p_\theta(\xx_t|\cc)p_\theta(\cc)}{p_\theta(\xx_t)})^\zeta \\
                                & \propto p_\theta(\xx_t|\cc)(\frac{p_\theta(\xx_t|\cc)}{p_\theta(\xx_t)})^\zeta \, .
\end{align*}
When $t \in [0, s]$ ($s = 0.75$), inspired by above expression and a recent work~\cite{tang2025diffusion}, we choose to use below as the target score function for training 
\[
    \nabla_{\xx_t}\log\left(p_t(\xx_t|\cc) \left(\frac{p_{t,\mtheta}(\xx_t|\cc)}{p_{t,\mtheta}(\xx_t)}\right)^\zeta\right)
\]
which equals to
\[
    \nabla_{\xx_t}\log p_t(\xx_t|\cc)+ \zeta \left(\nabla_{\xx_t} \log p_{t,\mtheta}(\xx_t|\cc)- \nabla_{\xx_t} \log p_{t,\mtheta}(\xx_t)\right) \, .
\]
For $\mf^{\mathrm{\zz}}_{\star}$ we originally want to learn:
\[
    \mf^{\mathrm{\zz}}_{\star}(\mmF_t, \xx_t, t)= -\alpha(t)\nabla_{\xx_t} \log p_t(\xx_t) \, ,
\]
now it turns to
\begin{align*}
    \mf^{\mathrm{\zz}}_{\star}(\mmF_t, \xx_t, t) &= -\alpha(t)\nabla_{\xx_t}\log\left(p_t(\xx_t|\cc) \left(\frac{p_{t,\mtheta}(\xx_t|\cc)}{p_{t,\mtheta}(\xx_t)}\right)^\zeta\right)\\
    &= -\alpha(t) \left[\nabla_{\xx_t}\log p_t(\xx_t|\cc)+ \zeta \left(\nabla_{\xx_t} \log p_{t,\mtheta}(\xx_t|\cc)- \nabla_{\xx_t} \log p_{t,\mtheta}(\xx_t)\right) \right] \\
    &= -\alpha(t) \nabla_{\xx_t}\log p_t(\xx_t|\cc) + \zeta  \left(-\alpha(t)\nabla_{\xx_t} \log p_{t,\mtheta}(\xx_t|\cc)+\alpha(t) \nabla_{\xx_t} \log p_{t,\mtheta}(\xx_t)\right) \\
    &= -\alpha(t) \nabla_{\xx_t}\log p_t(\xx_t|\cc) + \zeta  \left(\mf^{\mathrm{\zz}}(\mmF_t, \xx_t, t) - \mf^{\mathrm{\zz}}(\mmF^\varnothing_t, \xx_t, t)\right) \, ,
\end{align*}
thus in training we set the objective for $\mf^{\mathrm{\zz}}$ as:
\[
    \zz^\star \gets \zz + \zeta \cdot \left(\mf^{\mathrm{\zz}}(\mmF_t, \xx_t, t) - \mf^{\mathrm{\zz}}(\mmF^\varnothing_t, \xx_t, t)\right) \, .
\]
Similarly, since $\mf^{\mathrm{\xx}}_{\star}=\frac{\xx_t +\alpha^2(t)\nabla_{\xx_t} \log p_t(\xx_t)}{\gamma(t)}$ is also linear in the score function, we can use the same strategy to modify the training objective for $\mf^{\mathrm{\xx}}$:
\[
    \xx^\star \gets \xx + \zeta \cdot \left(\mf^{\mathrm{\xx}}(\mmF_t, \xx_t, t) - \mf^{\mathrm{\xx}}(\mmF^\varnothing_t, \xx_t, t)\right) \, .
\]
When $t \in (s, 1]$ ($s = 0.75$), we further slightly modify the target score function to
\[
    \nabla_{\xx_t}\log p_t(\xx_t|\cc)+ \zeta \left(\nabla_{\xx_t} \log p_{t,\mtheta}(\xx_t|\cc)- \nabla_{\xx_t} \log p_{t}(\xx_t)\right), \, \zeta=0.5 
\]
which corresponds to the following training objective:
\[
    \xx^\star \gets \xx + \frac{1}{2} \left(\mf^{\mathrm{\xx}}(\mmF_t, \xx_t, t) - \xx\right) \, ,
    \zz^\star \gets \zz + \frac{1}{2} \left(\mf^{\mathrm{\zz}}(\mmF_t, \xx_t, t) - \zz\right) \, .
\]

\subsubsection{Unified Sampling Process}
\label{app:unified_sampling}

\paragraph{Deterministic sampling.}
When the stochastic ratio $\rho=0$, let's analyze a apecial case where the coefficients satisfying $\hat{\alpha}(t)=\frac{\dm\alpha(t)}{\dm t},\hat{\gamma}(t)=\frac{\dm\gamma(t)}{\dm t}$.
Let $\Delta t=t_{i+1}-t_i$, for the core updating rule we have:
\begin{align*}
    \xx^{\prime} & = \alpha(t_{i+1}) \cdot  \hat{\zz}  + \gamma(t_{i+1}) \cdot \hat{\xx}                                                                                                                                                                                                                                                                                                                                                         \\
                 & = (\alpha(t_{i})+\alpha'(t_{i})\Delta t + o(\Delta t)) \cdot  \hat{\zz}  + (\gamma(t_{i})+\gamma'(t_{i})\Delta t + o(\Delta t)) \cdot \hat{\xx}                                                                                                                                                                                                                                                                               \\
                 & = (\alpha(t_{i})\hat{\zz}+\gamma(t_{i})\hat{\xx})+(\hat{\alpha}(t_{i})   \hat{\zz}  + \hat{\gamma}(t_{i}) \hat{\xx} )\cdot \Delta t  + o(\Delta t)                                                                                                                                                                                                                                                                            \\
                 & = (\alpha(t_{i})\mf^{\mathrm{\zz}}(\mmF, \tilde{\xx}, t_{i})+\gamma(t_{i})\mf^{\mathrm{\xx}}(\mmF, \tilde{\xx}, t_{i}))+(\hat{\alpha}(t_{i})   \mf^{\mathrm{\zz}}(\mmF, \tilde{\xx}, t_{i})  + \hat{\gamma}(t_{i}) \mf^{\mathrm{\xx}}(\mmF, \tilde{\xx}, t_{i}) )\cdot \Delta t  + o(\Delta t)                                                                                                                                \\
                 & = (\alpha(t_{i})\frac{\hat{\gamma}(t_{i})\cdot\tilde{\xx} - \gamma(t_{i})\cdot\mmF(\tilde{\xx}, t_{i})}{\alpha(t_{i})\cdot \hat{\gamma}(t_{i}) - \hat{\alpha}(t_{i}) \cdot \gamma(t_{i})}+\gamma(t_{i})\frac{\alpha(t_{i})\cdot\mmF(\tilde{\xx}, t_{i}) - \hat{\alpha}(t_{i})\cdot\xx_t}{\alpha(t_{i})\cdot \hat{\gamma}(t_{i}) - \hat{\alpha}(t_{i}) \cdot \gamma(t_{i})})                                                   \\
                 & + (\hat{\alpha}(t_{i})   \frac{\hat{\gamma}(t_{i})\cdot\tilde{\xx} - \gamma(t_{i})\cdot\mmF(\tilde{\xx}, t_{i})}{\alpha(t_{i})\cdot \hat{\gamma}(t_{i}) - \hat{\alpha}(t_{i}) \cdot \gamma(t_{i})}  +  \hat{\gamma}(t_{i}) \frac{\alpha(t_{i})\cdot\mmF(\tilde{\xx}, t_{i}) - \hat{\alpha}(t_{i})\cdot\xx_t}{\alpha(t_{i})\cdot \hat{\gamma}(t_{i}) - \hat{\alpha}(t_{i}) \cdot \gamma(t_{i})} )\cdot \Delta t  + o(\Delta t) \\
                 & = \tilde{\xx}+\mmF(\tilde{\xx}, t_{i})\cdot \Delta t  + o(\Delta t)                                                                                                                                                                                                                                                                                                                                                           \\
\end{align*}
In this case $\mmF(\cdot, \cdot)$ tries to predict the velocity field of the flow model, and we can see that the term $\tilde{\xx}+\mmF(\tilde{\xx}, t_{i})\cdot \Delta t$ corresponds to the sampling rule of the Euler ODE solver.

\paragraph{Stochastic sampling.}
As for case when the stochastic ratio $\rho \neq 0$, follow the Euler-Maruyama numerical methods of SDE,
the noise injected should be a Gaussian with zero mean and variance proportional to $\Delta t$, so when the updating rule is
$\xx^{\prime} = \alpha(t_{i+1}) \cdot (\sqrt{1-\rho} \cdot \hat{\zz} + \sqrt{\rho} \cdot \zz ) + \gamma(t_{i+1}) \cdot \hat{\xx}$, the coefficient of $\zz$ should satisfy
\[
    \alpha(t_{i+1})\sqrt{\rho} \propto \sqrt{\Delta t}, \ \rho \propto \frac{\Delta t}{\alpha^2(t_{i+1})}
\]
In practice, we set
\[
    \rho = \frac{2\Delta t\cdot \alpha(t_{i})}{\alpha^2(t_{i+1})} \, .
\]
which corresponds to $g(t)=\sqrt{2\alpha(t)}$ for the SDE $\dm \xx = \mf(\xx,t)\dm t + g(t)\dm \boldsymbol{w}$.

\subsubsection{Extrapolating Estimation}
\label{app:extrapolating}

\begin{theorem}[Local Truncation error of the extrapolated update]
    Let \(\{\tilde{\xx}_i\}\) be the sequence defined by the extrapolated update
    \[
        \tilde{\xx}_{i+1}
        \;=\;
        \tilde{\xx}_i
        \;+\;
        h\bigl(\vv_i + \kappa(\vv_i - \vv_{i-1})\bigr)
        \;+\;
        h^2\,\mepsilon_i,
        \quad
        h = t_{i+1}-t_i,
    \]
    where
    \(\vv_i = \vv(\tilde{\xx}_i,t_i)\)
    and \(\mepsilon_i = O(1)\).  Denote by \(\xx(t_{i+1})\) the exact solution of
    \(\dot\xx = \vv(\xx,t)\) at time \(t_{i+1}\).  Then the local truncation error satisfies
    \[
        \xx(t_{i+1}) - \tilde{\xx}_{i+1}
        \;=\;
        h^2\Bigl[\bigl(\tfrac12 - \kappa\bigr)\,\vv'\!(\tilde{\xx}_i,t_i)\;-\;\mepsilon_i\Bigr]
        \;+\;O(h^3),
    \]
    where \(\vv'\!(\tilde{\xx}_i,t_i)\) denotes the total derivative of \(\vv\) along the trajectory.  In particular, choosing \(\kappa=\tfrac12\) cancels the \(O(h^2)\) term (up to \(\mepsilon_i\)), yielding a second-order method.
\end{theorem}

\begin{proof}
    1. By Taylor’s theorem in time,
    \[
        \vv_{i-1}
        = \vv(\tilde{\xx}_{i-1},t_{i-1})
        = \vv_i
        - h\,\vv'\!(\tilde{\xx}_i,t_i)
        + O(h^2).
    \]

    2. Substitute into the update rule:
    \[
        \begin{aligned}
            \tilde{\xx}_{i+1}
             & = \tilde{\xx}_i
            + h\bigl[\vv_i + \kappa(\vv_i - \vv_{i-1})\bigr]
            + h^2\,\mepsilon_i \\
             & = \tilde{\xx}_i
            + h\Bigl(\vv_i
            + \kappa\bigl[\vv_i - (\vv_i - h\,\vv' + O(h^2))\bigr]\Bigr)
            + h^2\,\mepsilon_i \\
             & = \tilde{\xx}_i
            + h\,\vv_i
            + \kappa\,h^2\,\vv'\!(\tilde{\xx}_i,t_i)
            + h^2\,\mepsilon_i
            + O(h^3).
        \end{aligned}
    \]

    3. The exact solution expands as
    \[
        \xx(t_{i+1})
        = \xx(t_i)
        + h\,\vv(\xx(t_i),t_i)
        + \tfrac{h^2}{2}\,\vv'\!(\xx(t_i),t_i)
        + O(h^3).
    \]
    Replacing \(\xx(t_i)\) by \(\tilde{\xx}_i\) in the leading terms gives
    \[
        \xx(t_{i+1})
        = \tilde{\xx}_i
        + h\,\vv_i
        + \tfrac{h^2}{2}\,\vv'\!(\tilde{\xx}_i,t_i)
        + O(h^3).
    \]

    4. Subtracting yields the local truncation error:
    \[
        \begin{aligned}
            \xx(t_{i+1}) - \tilde{\xx}_{i+1}
             & = \bigl[\tilde{\xx}_i + h\vv_i + \tfrac{h^2}{2}\vv' + O(h^3)\bigr]
            - \bigl[\tilde{\xx}_i + h\vv_i + \kappa\,h^2\vv' + h^2\mepsilon_i + O(h^3)\bigr]         \\
             & = h^2\Bigl[\bigl(\tfrac12 - \kappa\bigr)\vv'\!(\tilde{\xx}_i,t_i) - \mepsilon_i\Bigr]
            + O(h^3).
        \end{aligned}
    \]
    This completes the proof.
\end{proof}

\begin{remark}[Error reduction via the extrapolation ratio~$\kappa$]
    From the local truncation error estimate
    \[
        \xx(t_{i+1})-\tilde\xx_{i+1}
        = h^2\Bigl[\bigl(\tfrac12-\kappa\bigr)\,\vv'(\tilde\xx_i,t_i)-\mepsilon_i\Bigr]
        +O(h^3),
    \]
    define
    \[
        E(\kappa)
        \;=\;(\tfrac12-\kappa)\,\vv'(\tilde\xx_i,t_i)-\mepsilon_i,
        \quad
        E(0)=\tfrac12\,\vv'(\tilde\xx_i,t_i)-\mepsilon_i.
    \]
    Note that
    \[
        \min_{\kappa\in[0,1]}\|E(\kappa)\|
        \;\le\;\|E(0)\|.
    \]
    By selecting an appropriate $\kappa$ value, the $O(h^2)$ coefficient—and thus the leading part of the local truncation error—is is smaller (or at least not larger) in norm than in the case $\kappa=0$.
\end{remark}

\subsection{Other Techniques}

\subsubsection{Beta Transformation}
\label{app:beta_transformation}

\begin{figure}[t]
    \centering
    \begin{subfigure}[b]{0.325\textwidth}
        \centering
        \includegraphics[width=\linewidth]{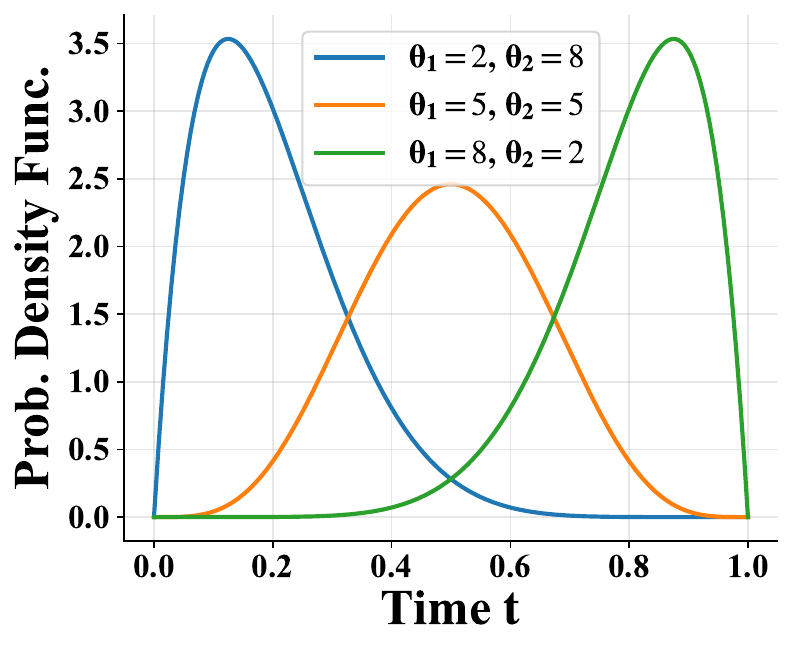}
        \caption{\textbf{Skewed and symmetric.}}
    \end{subfigure}
    \hfill
    \begin{subfigure}[b]{0.325\textwidth}
        \centering
        \includegraphics[width=\linewidth]{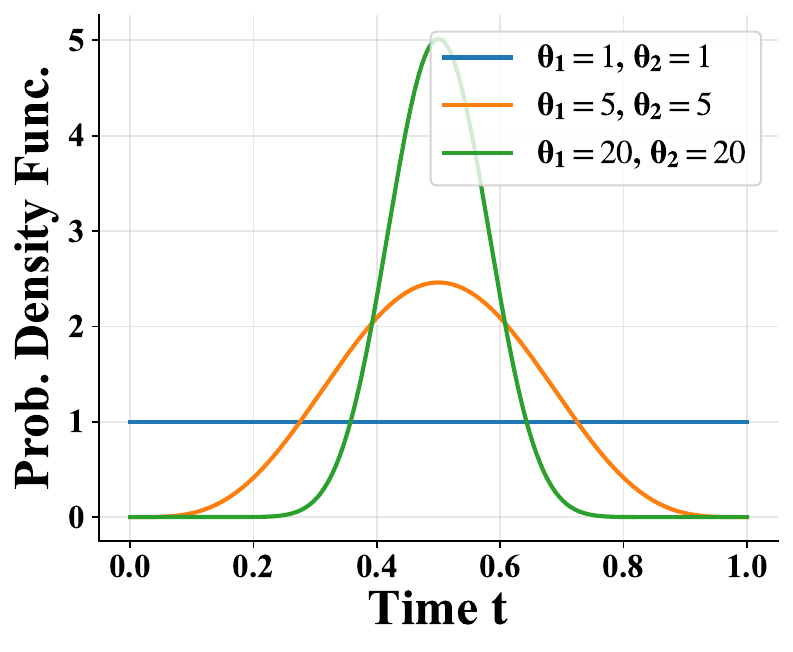}
        \caption{\textbf{Increasingly concentrated.}}
    \end{subfigure}
    \hfill
    \begin{subfigure}[b]{0.325\textwidth}
        \centering
        \includegraphics[width=\linewidth]{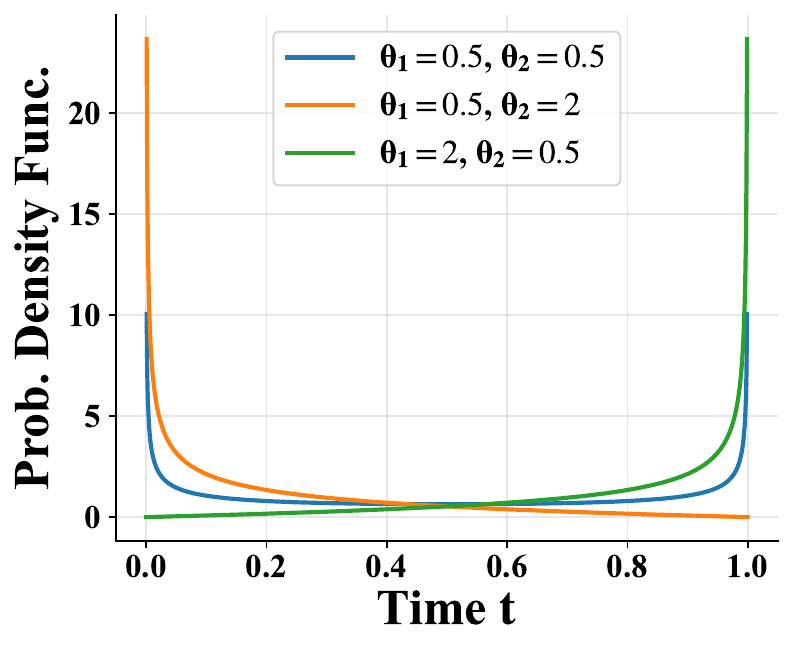}
        \caption{\textbf{J- and U-shaped.}}
    \end{subfigure}
    \caption{\small{
            \textbf{Probability density functions of the Beta distribution over the domain \(t\in[0,1]\) for various shape-parameter $\theta_1, \theta_2$.}
        }}
    \label{fig:beta_dists}
\end{figure}

We utilize three representative cases to illustrate how the Beta transformation \( f_{\mathrm{Beta}}(t; \theta_1, \theta_2) \) generalizes time warping mechanisms for \( t \in [0,1] \).

\paragraph{Standard logit-normal time transformation~\citep{yao2025reconstruction,esser2024scaling}.}
For \( t \sim \mathcal{U}(0,1) \), the logit-normal transformation \( f_{\mathrm{lognorm}}(t; 0, 1) = \frac{1}{1 + \exp(-\Phi^{-1}(t))} \) generates a symmetric density profile peaked at \( t = 0.5 \), consistent with the central maximum of the logistic-normal distribution. Analogously, the Beta transformation \( f_{\mathrm{Beta}}(t; \theta_1, \theta_2) \) (with \( \theta_1, \theta_2 > 1 \)) produces a density peak at \( t = \frac{\theta_1 - 1}{\theta_1 + \theta_2 - 2} \). When \( \theta_1 = \theta_2 > 1 \), this reduces to \( t = 0.5 \), mirroring the logit-normal case. Both transformations concentrate sampling density around critical time regions, enabling importance sampling for accelerated training. Notably, this effect can be equivalently achieved by directly sampling \( t \sim \mathrm{Beta}(\theta_1, \theta_2) \).

\paragraph{Uniform time distribution~\citep{yao2025reconstruction,yu2024representation,ma2024sit,lipman2022flow}.}
The uniform limit case emerges when \( \theta_1 = \theta_2 = 1 \), reducing \( f_{\mathrm{Beta}}(t; 1, 1) \) to an identity transformation. This corresponds to a flat density \( p(t) = 1 \), reflecting no temporal preference—a baseline configuration widely adopted in diffusion and flow-based models.

\paragraph{Approximately symmetrical time distribution~\citep{song2023consistency,song2023improved,karras2022elucidating,karras2024analyzing}.}
For near-symmetric configurations where \( \theta_1 \approx \theta_2 > 1 \), the Beta transformation induces quasi-symmetrical densities with tunable central sharpness. For instance, setting \( \theta_1 = \theta_2 = 2 \) yields a parabolic density peaking at \( t = 0.5 \), while \( \theta_1 = \theta_2 \to 1^+ \) asymptotically approaches uniformity. This flexibility allows practitioners to interpolate between uniform sampling and strongly peaked distributions, adapting to varying requirements for temporal resolution in training. Such approximate symmetry is particularly useful in consistency models where balanced gradient propagation across time steps is critical.

Furthermore, \figref{fig:beta_dists} further demonstrates the flexibility of the beta distribution.

\subsubsection{Kumaraswamy Transformation}
\label{app:kuma_transformation}

\begin{lemma}[Piecewise monotone error]
    \label{lemma:piecewise_monotone_error}
    Suppose \(f,g\) are continuous and nondecreasing on \([0,1]\), and agree at
    \[
        0 = x_0 < x_1 < \cdots < x_n = 1 \,,
    \]
    i.e.\ \(f(x_j)=g(x_j)\) for \(j=0,\dots,n\).  Let \(\Delta_j = g(x_j)-g(x_{j-1})\).  Then for every \(t\in[x_{j-1},x_j]\),
    \[
        |f(t)-g(t)|\le\Delta_j \,.
    \]
    In particular, if each \(\Delta_j\le\frac14\), then \(\|f-g\|_{L^\infty}\le\frac14\).
\end{lemma}

\begin{proof}
    On \([x_{j-1},x_j]\) monotonicity gives
    \[
        f(t)-g(t)\le f(x_j)-g(x_{j-1})
        = g(x_j)-g(x_{j-1})=\Delta_j,
    \]
    and similarly \(g(t)-f(t)\le\Delta_j\).
\end{proof}

\begin{theorem}[$L^2$ approximation bound of monotonic functions by generalized Kumaraswamy transformation]
    Let
    $
        \mathcal G = \bigl\{\,g\in C([0,1]):g\text{ nondecreasing},\;g(0)=0,\;g(1)=1\bigr\},
    $
    and define for \(a,b,c>0\),
    $
        f_{a,b,c}(t)=\bigl(1-(1-t^a)^b\bigr)^c
    $, $t\in[0,1]$.
    Then
    \[
        \sup_{g\in\mathcal G}\;\inf_{a,b,c>0}
        \int_{0}^{1}\bigl[f_{a,b,c}(t)-g(t)\bigr]^{2}\,\dm t
        \;\le\;\frac1{16} \,.
    \]
\end{theorem}

\begin{proof}
    Let \(g\in\mathcal G\).  By continuity and the Intermediate-Value Theorem there exist
    \[
        0<t_1<t_0<t_2<1,\quad g(t_1)=\tfrac14,\;g(t_0)=\tfrac12,\;g(t_2)=\tfrac34.
    \]
    We will choose \((a,b,c)>0\) so that
    \[
        f_{a,b,c}(t_j)=g(t_j)\quad (j=1,0,2),
    \]
    and then apply the piecewise monotone~\lemref{lemma:piecewise_monotone_error} on the partition
    \[
        0,\;t_1,\;t_0,\;t_2,\;1
    \]
    to conclude \(\|f_{a,b,c}-g\|_{L^\infty}\le\frac14\) and hence
    \(\|f_{a,b,c}-g\|_{L^2}^2\le\frac1{16}\).

    \medskip
    \paragraph{Existence via the implicit function theorem.}
    Define
    \[
        F:\mathbb R_{>0}^3\longrightarrow\mathbb R^3,\qquad
        F(a,b,c)=
        \begin{pmatrix}
            f_{a,b,c}(t_1)-\tfrac14 \\
            f_{a,b,c}(t_0)-\tfrac12 \\
            f_{a,b,c}(t_2)-\tfrac34
        \end{pmatrix}.
    \]
    Then \(F\) is \(C^1\), and at the “base point” \((a,b,c)=(1,1,1)\) with
    \((t_1,t_0,t_2)=(\tfrac14,\tfrac12,\tfrac34)\) we have \(f_{1,1,1}(t)=t\)
    so \(F(1,1,1)=0\), and the Jacobian \(\partial F/\partial(a,b,c)\) there is invertible.
    By the Implicit Function Theorem, for each fixed \((t_1,t_0,t_2)\) near \((\tfrac14,\tfrac12,\tfrac34)\) there is a unique local solution \((a,b,c)\).

    \vspace{1ex}
    \paragraph{Global non-degeneracy of the Jacobian.}

    In order to continue this local solution to \emph{all} triples \(0<t_1<t_0<t_2<1\), we show
    \(\det\bigl(\partial_{(a,b,c)}F(a,b,c)\bigr)\) never vanishes.

    Set
    \[
        u(t)=1-(1-t^a)^b,\quad u_j=u(t_j)\in(0,1),\quad f_j=u_j^c.
    \]
    Then
    \[
        \partial_a f_j
        = c\,u_j^{\,c-1}\,\partial_a u_j,
        \quad
        \partial_b f_j
        = c\,u_j^{\,c-1}\,\partial_b u_j,
        \quad
        \partial_c f_j
        = u_j^c\ln u_j.
    \]
    Hence
    \[
        \det J
        =\det
        \begin{pmatrix}
            c\,u_1^{c-1}u_{1,a} & c\,u_1^{c-1}u_{1,b} & u_1^c\ln u_1 \\
            c\,u_0^{c-1}u_{0,a} & c\,u_0^{c-1}u_{0,b} & u_0^c\ln u_0 \\
            c\,u_2^{c-1}u_{2,a} & c\,u_2^{c-1}u_{2,b} & u_2^c\ln u_2
        \end{pmatrix}.
    \]
    Factor \(c\) from the first two columns and \(u_j^{c-1}\) from each row:
    \[
        \det J
        = c^2\,(u_1u_0u_2)^{c-1}
        \;\det
        \begin{pmatrix}
            u_{1,a} & u_{1,b} & u_1\ln u_1 \\
            u_{0,a} & u_{0,b} & u_0\ln u_0 \\
            u_{2,a} & u_{2,b} & u_2\ln u_2
        \end{pmatrix}.
    \]
    Now
    \[
        u_{j,b}
        =-(1-t_j^a)^b\ln(1-t_j^a)
        =-(1-u_j)\ln(1-t_j^a),
    \]
    \[
        u_{j,a}
        =b\,(1-t_j^a)^{b-1}t_j^a\ln t_j
        =-b\,(1-u_j)\,\frac{t_j^a\ln t_j}{1-t_j^a}.
    \]
    A direct—but straightforward—expansion shows
    \[
        \det
        \begin{pmatrix}
            u_{1,a} & u_{1,b} & u_1\ln u_1 \\
            u_{0,a} & u_{0,b} & u_0\ln u_0 \\
            u_{2,a} & u_{2,b} & u_2\ln u_2
        \end{pmatrix}
        =
        c^{-2}b\,
        \frac{u_1u_0u_2}{(1-u_1)(1-u_0)(1-u_2)}\,
        (u_0-u_1)(u_2-u_1)(u_2-u_0).
    \]
    Therefore
    \[
        \det J(a,b,c)
        = b\,(u_1u_0u_2)^c
        \,\frac{(u_0-u_1)(u_2-u_1)(u_2-u_0)}
        {(1-u_1)(1-u_0)(1-u_2)}
        >0,
    \]
    since \(0<u_1<u_0<u_2<1\) and \(a,b,c>0\).  Hence the Jacobian is everywhere non-zero, and the local solution by the Implicit Function Theorem extends along any path in the connected domain \(\{0<t_1<t_0<t_2<1\}\).  We obtain a unique \((a,b,c)>0\) solving
    \[
        f_{a,b,c}(t_j)=g(t_j),\quad j=1,0,2,
    \]
    for \emph{every} choice \(0<t_1<t_0<t_2<1\).

    \medskip
    \paragraph{Completing the error estimate.}
    By construction \(f_{a,b,c}(0)=0\), \(f_{a,b,c}(1)=1\), and
    \(f_{a,b,c}(t_j)=g(t_j)\) for \(j=1,0,2\).  On the partition
    \[
        0,\;t_1,\;t_0,\;t_2,\;1
    \]
    the increments of \(g\) are each \(1/4\).  The piecewise monotone error~\lemref{lemma:piecewise_monotone_error} yields
    \(\|f_{a,b,c}-g\|_{L^\infty}\le\tfrac14\), hence
    \[
        \int_0^1\bigl[f_{a,b,c}(t)-g(t)\bigr]^2\,\dm t
        \le\|f-g\|_{L^\infty}^2\le\frac1{16}.
    \]
    Since \(g\) was arbitrary in \(\mathcal G\), we conclude
    \[
        \sup_{g\in\mathcal G}\;\inf_{a,b,c>0}
        \int_0^1\bigl[f_{a,b,c}(t)-g(t)\bigr]^2\,\dm t
        \le\frac1{16}.
    \]
    This completes the proof.
\end{proof}

\paragraph{Setting and notation.}

Fix a positive real number \( s > 0 \) and consider the \emph{shift function}
\[
    f_{\mathrm{shift}}(t;s)=\frac{s\,t}{1+(s-1)t},\qquad t\in[0,1].
\]
For \( a,b,c>0 \), define the \emph{Kumaraswamy transform} as
\[
    f_{\mathrm{Kuma}}(t;a,b,c)=\left(1-\Bigl(1-t^{a}\Bigr)^{b}\right)^{c},\qquad t\in[0,1].
\]
Notice that when \( a=b=c=1 \) one obtains
\[
    f_{\mathrm{Kuma}}(t;1,1,1) = 1 - \bigl(1-t^1\bigr)^1 = t,
\]
so that the identity function appears as a special case.

We work in the Hilbert space \( L^{2}([0,1]) \) with the inner product
\[
    \langle f,g\rangle=\int_{0}^{1}f(t)g(t)\,\mathrm{d} t.
\]
Accordingly, we introduce the error functional
\[
    J(a,b,c):=\Bigl\|f_{\mathrm{Kuma}}(\cdot;a,b,c)-f_{\mathrm{shift}}(\cdot;s)\Bigr\|_{2}^{2}\quad
    \text{and}\quad
    J_{\mathrm{id}}:=\Bigl\|\mathrm{id}-f_{\mathrm{shift}}(\cdot;s)\Bigr\|_{2}^{2}.
\]
It is known that for \( s\neq1 \) one has
\[
    \inf_{a,b,c}J(a,b,c)<J_{\mathrm{id}}.
\]
The goal is to quantify this improvement by optimally adjusting all three parameters \( (a,b,c) \).

\paragraph{Quadratic approximation around the identity.}

Since the interesting behavior occurs near the identity \( (a,b,c)=(1,1,1) \), we reparameterize as
\[
    \theta:=
    \begin{pmatrix}
        \alpha \\[1mm]\beta\\[1mm]\gamma
    \end{pmatrix}
    :=
    \begin{pmatrix}
        a-1 \\[1mm] b-1\\[1mm] c-1
    \end{pmatrix},
    \qquad\text{with } \|\theta\|\lll1.
\]
Thus, we study the function
\[
    f_{\mathrm{Kuma}}(t;1+\alpha,1+\beta,1+\gamma)
\]
in a small neighborhood of \((1,1,1)\). Writing
\[
    F(a,b,c;t):=f_{\mathrm{Kuma}}(t;a,b,c)=\Bigl(1-(1-t^a)^b\Bigr)^c,
\]
a second--order Taylor expansion around \( (a,b,c)=(1,1,1) \) gives
\begin{equation}\label{eq:Taylor}
    f_{\mathrm{Kuma}}(t;1+\alpha,1+\beta,1+\gamma)
    = t \;+\; \sum_{i=1}^{3}\theta_{i}\,g_{i}(t)
    \;+\; \frac{1}{2}\sum_{i,j=1}^{3}\theta_{i}\theta_{j}\,h_{ij}(t)
    \;+\; \mathcal{O}(\|\theta\|^{3}),
\end{equation}
where
\[
    g_{i}(t)=\frac{\partial}{\partial \theta_{i}} f_{\mathrm{Kuma}}(t;1+\theta) \Big|_{\theta=0}
    \quad\text{and}\quad
    h_{ij}(t)=\frac{\partial^{2}}{\partial\theta_{i}\partial\theta_{j}} f_{\mathrm{Kuma}}(t;1+\theta)
    \Big|_{\theta=0}.
\]
A short calculation yields:
\begin{enumerate}[label=(\alph*), nosep, leftmargin=16pt]
    \item With respect to \( a \) (noting that for \( b=c=1 \) one has \( f_{\mathrm{Kuma}}(t;a,1,1)=t^a \)):
          \[
              g_{1}(t)=\frac{\partial f_{\mathrm{Kuma}}}{\partial a}(t;1,1,1)= \frac{\mathrm{d}}{\mathrm{d} a}t^a\Big|_{a=1}
              = t\ln t.
          \]
    \item With respect to \( b \) (since for \( a=1,\,c=1 \) we have \( f_{\mathrm{Kuma}}(t;1,b,1)=1-(1-t)^b \)):
          \[
              g_{2}(t)=\frac{\partial f_{\mathrm{Kuma}}}{\partial b}(t;1,1,1)
              = - (1-t)\,\ln(1-t).
          \]
    \item With respect to \( c \) (noting that for \( a=b=1 \) we have \( f_{\mathrm{Kuma}}(t;1,1,c)=t^c \)):
          \[
              g_{3}(t)=\frac{\partial f_{\mathrm{Kuma}}}{\partial c}(t;1,1,1)
              = t\ln t.
          \]
\end{enumerate}
Thus, we observe that
\[
    g_{1}(t)=g_{3}(t),
\]
which indicates an inherent redundancy in the three-parameter model. In consequence, the Gram matrix (defined below) will be of rank at most two.

Next, define the difference between the identity and the shift functions:
\[
    g(t):=\mathrm{id}(t)-f_{\mathrm{shift}}(t;s)
    = t-\frac{s\,t}{1+(s-1)t}
    = (1-s)\frac{t(1-t)}{1+(s-1)t}.
\]
Then, \( J_{\mathrm{id}}=\langle g,g \rangle \). Also, introduce the first-order moments and the Gram matrix:
\[
    v_i:=\langle g,g_i\rangle,\qquad
    G_{ij}:=\langle g_i,g_j\rangle,\quad i,j=1,2,3.
\]

Inserting the expansion \eqref{eq:Taylor} into the error functional gives
\[
    J(1+\theta)= \bigl\|f_{\mathrm{Kuma}}(\cdot;1+\theta)-f_{\mathrm{shift}}(\cdot;s)\bigr\|_2^2
    = J_{\mathrm{id}} -2\sum_{i=1}^{3}\theta_{i}\,v_i +
    \sum_{i,j=1}^{3}\theta_i\theta_j\,G_{ij} +\mathcal{O}(\|\theta\|^{3}).
\]
Thus, the quadratic approximation (or model) of the error is
\[
    \widehat{J}(\theta):=
    J_{\mathrm{id}}-2\,\theta^{\!\top}v+\theta^{\!\top}G\,\theta.
\]
Since the Gram matrix \( G \) is positive semidefinite (and has a nontrivial null-space due to \( g_1=g_3 \)), the minimizer is determined only up to the null-space. To select the unique (minimum--norm) minimizer, we choose
\[
    \theta^{\star}=G^{\dagger}v,
\]
where \( G^{\dagger} \) denotes the Moore-Penrose pseudoinverse. The quadratic model is then minimized at
\[
    \widehat{J}_{\min}=J_{\mathrm{id}}-v^{\!\top}G^{\dagger}v.
\]
A scaling argument now shows that for any sufficiently small \( \varepsilon>0 \) one has
\[
    J(1+\varepsilon\,\theta^{\star}) \le \widehat{J}(\varepsilon\,\theta^{\star})
    =J_{\mathrm{id}}-\varepsilon^{2}\,v^{\!\top}G^{\dagger}v < J_{\mathrm{id}},
\]
so that the full nonlinear functional is improved by following the direction of \( \theta^{\star} \).

For convenience we introduce the explicit improvement factor
\begin{equation}\label{eq:rho3_def}
    \rho_{3}(s):=\frac{v^{\!\top}G^{\dagger}v}{J_{\mathrm{id}}(s)} \in (0,1),
    \qquad s\neq1,
\end{equation}
so that our main bound can be written succinctly as
\begin{equation}\label{eq:main_bound}
    \min_{a,b,c>0}J(a,b,c)
    \;\le\;
    \Bigl(1-\rho_{3}(s)\Bigr)\,J_{\mathrm{id}}(s).
    \qquad (s>0,\; s\neq1)
\end{equation}

\paragraph{Computation of the Gram matrix \( G \).}

We now compute the inner products
\[
    G_{ij}=\langle g_i, g_j\rangle, \quad i,j=1,2,3.
\]
Since the functions \( g_1 \) and \( g_3 \) are identical, only two independent functions appear in the system.
A standard fact from Beta-function calculus is that
\[
    \int_{0}^{1} t^{n}\ln^{2}t\,\mathrm{d} t
    = \frac{2}{(n+1)^{3}}, \quad n>-1.
\]
Thus, one has
\begin{align*}
    \langle g_{1},g_{1}\rangle
     & =\int_{0}^{1}t^{2}\ln^{2}t\,\mathrm{d} t
    =\frac{2}{3^{3}}=\frac{2}{27},                        \\[2mm]
    \langle g_{2},g_{2}\rangle
     & =\int_{0}^{1}(1-t)^{2}\ln^{2}(1-t)\,\mathrm{d} t
    =\frac{2}{27},
    \intertext{since the change of variable \(u=1-t\) yields the same result.}
    \langle g_{1},g_{2}\rangle
     & = -\int_{0}^{1}t(1-t)\ln t\,\ln(1-t)\,\mathrm{d} t
    = \frac{3\pi^2-37}{108}.
\end{align*}
It is now convenient to express the Gram matrix with an overall factor:
\[
    G = \frac{2}{27}\begin{pmatrix}
        1 & r & 1 \\[2mm]
        r & 1 & r \\[2mm]
        1 & r & 1
    \end{pmatrix}, r= \frac{3\pi^2-37}{8}.
\]

Since \( g_{1}=g_{3} \), it is clear that the columns (and rows) corresponding to parameters \( a \) and \( c \) are identical, so that \( \mathrm{rank}(G)=2 \). One can compute the Moore-Penrose pseudoinverse \( G^\dagger \) by eliminating one of the redundant rows/columns, inverting the resulting \( 2 \times 2 \) block, and then re-embedding into \( \mathbb{R}^{3\times 3} \). One obtains
\[
    G^{\dagger} = \frac{27}{8(1-r^2)}\begin{pmatrix}
        1   & -2r & 1   \\[2mm]
        -2r & 4   & -2r \\[2mm]
        1   & -2r & 1
    \end{pmatrix}.
\]

\paragraph{Computation of the first-order moments \( v_i \).}

Recall that
\[
    g(t)=\mathrm{id}(t)-f_{\mathrm{shift}}(t;s)= t-\frac{s\,t}{1+(s-1)t}.
\]
This expression can be rewritten as
\[
    g(t) = (1-s) \, t(1-t)D_{s}(t),\qquad\text{with}\quad
    D_{s}(t):=\frac{1}{1+(s-1)t}.
\]
Then, the first--order moments read
\begin{align*}
    v_{1}=v_{3} & =(1-s)
    \int_{0}^{1} t(1-t)D_{s}(t)\; t\ln t\,\mathrm{d} t, \\[2mm]
    v_{2}       & =-(1-s)
    \int_{0}^{1} t(1-t)D_{s}(t)\; (1-t)\ln(1-t)\,\mathrm{d} t.
\end{align*}
These integrals can be expressed in closed form (involving logarithms and powers of \((s-1)\)); in the case \( s\neq1 \) at least one of the \(v_i\) is nonzero so that \( \rho_{3}(s)>0 \).

\paragraph{A universal numerical improvement.}

Since projecting onto the three-dimensional subspace spanned by \( \{g_1,g_2,g_3\} \) is at least as effective as projecting onto any one axis, we immediately deduce that
\[
    \rho_{3}(s)\ge \rho_{1}(s),
\]
where the one-parameter improvement factor is defined by
\[
    \rho_{1}(s):=\frac{v_{1}(s)^{2}}{\langle g_{1},g_{1}\rangle\,J_{\mathrm{id}}(s)}.
\]
By an elementary (albeit slightly tedious) estimate --- for example, using the bounds \(\frac{1}{2}\le D_{s}(t)\le2\) valid for \(|s-1|\le1\) --- one can show that
\[
    \rho_{1}(s)\ge\frac{49}{1536}.
\]
Hence, one deduces that
\[
    \rho_{3}(s)\ge\frac{49}{1536}\approx 0.0319,
    \qquad \text{for }|s-1|\le1.
\]
In particular, for \( s\in[0.5,2]\setminus\{1\} \) the optimal three-parameter Kumaraswamy transform reduces the squared \( L^{2} \) error by at least \( 3.19\% \) compared with the identity mapping. Analogous bounds can be obtained on any compact subset of \((0,\infty)\setminus\{1\}\).

\paragraph{Interpretation of the bound.}

Inequality \eqref{eq:main_bound} strengthens the known qualitative result (namely, that the three-parameter model can outperform the identity mapping) in two important respects:
\begin{enumerate}[label=(\alph*), nosep, leftmargin=16pt]
    \item {Quantitative improvement:} The explicit factor \(\rho_{3}(s)\) is computable via one-dimensional integrals, providing a concrete measure of the error reduction.
    \item {Utilization of all three parameters:} Even though the redundancy (i.e. \(g_1=g_3\)) implies that the Gram matrix is singular, the full three-parameter model still offers strict improvement; indeed, one has \(\rho_{3}(s)\ge\rho_{1}(s)>0\) for \( s\neq1 \). (Equality would require, hypothetically, that \( v_2(s)=0 \), which does not occur in practice.)
\end{enumerate}

\paragraph{Summary.}

For every shift parameter \( s>0 \) with \( s\neq1 \) there exist parameters \( (a,b,c) \) (in a neighborhood of \( (1,1,1) \)) such that
\[
    \Bigl\|f_{\mathrm{Kuma}}(\cdot;a,b,c)-f_{\mathrm{shift}}(\cdot;s)\Bigr\|_{2}^{2}
    \le
    \Bigl(1-\rho_{3}(s)\Bigr)\,\Bigl\|\mathrm{id}-f_{\mathrm{shift}}(\cdot;s)\Bigr\|_{2}^{2},
\]
with the improvement factor \(\rho_{3}(s)\) defined in \eqref{eq:rho3_def} and satisfying
\[
    \rho_{3}(s)\ge 0.0319\quad \text{on } s\in[0.5,2]\setminus\{1\}.
\]
Thus, the full three-parameter Kumaraswamy transform not only beats the identity mapping but does so by a quantifiable margin.

\subsubsection{Derivative Estimation}
\label{app:derivative_estimation}
\begin{proposition}[Error estimates for forward and central difference quotients]
    Let \(f\in C^3(I)\) where \(I\subset\mathbb{R}\) is an open interval, and let \(t\in I\).  For \(0<\varepsilon\) small enough that \([t-\varepsilon,t+\varepsilon]\subset I\), define the forward and central difference quotients
    \[
        D_+f(t)=\frac{f(t+\varepsilon)-f(t)}{\varepsilon},
        \qquad
        D_0f(t)=\frac{f(t+\varepsilon)-f(t-\varepsilon)}{2\varepsilon}.
    \]
    Then
    \begin{align*}
        D_+f(t) & =f'(t)+\frac{\varepsilon}{2}\,f''(t)
        +\frac{\varepsilon^2}{6}\,f^{(3)}(t+\theta_1\varepsilon),
                &                                      & \text{for some }0<\theta_1<1,          \\
        D_0f(t) & =f'(t)
        +\frac{\varepsilon^2}{12}\Bigl[f^{(3)}(t+\theta_2\varepsilon)
        +f^{(3)}(t-\theta_3\varepsilon)\Bigr],
                &                                      & \text{for some }0<\theta_2,\theta_3<1.
    \end{align*}
    In particular,
    \[
        D_+f(t)-f'(t)=O(\varepsilon),
        \qquad
        D_0f(t)-f'(t)=O(\varepsilon^2),
    \]
    so for sufficiently small \(\varepsilon\), the forward‐difference error exceeds the central‐difference error.
\end{proposition}

\begin{proof}
    By Taylor’s theorem with Lagrange remainder, for some \(0<\theta_1<1\),
    \[
        f(t+\varepsilon)
        = f(t)+f'(t)\,\varepsilon
        +\tfrac12f''(t)\,\varepsilon^2
        +\tfrac16f^{(3)}(t+\theta_1\varepsilon)\,\varepsilon^3.
    \]
    Dividing by \(\varepsilon\) gives the formula for \(D_+f(t)\).  Hence
    \[
        D_+f(t)-f'(t)
        =\frac12f''(t)\,\varepsilon
        +\frac{1}{6}f^{(3)}(t+\theta_1\varepsilon)\,\varepsilon^2
        =O(\varepsilon).
    \]

    Similarly, applying Taylor’s theorem at \(t+\varepsilon\) and \(t-\varepsilon\),
    \begin{align*}
        f(t+\varepsilon)
         & =f(t)+f'(t)\,\varepsilon+\tfrac12f''(t)\,\varepsilon^2
        +\tfrac16f^{(3)}(t+\theta_2\varepsilon)\,\varepsilon^3,   \\
        f(t-\varepsilon)
         & =f(t)-f'(t)\,\varepsilon+\tfrac12f''(t)\,\varepsilon^2
        -\tfrac16f^{(3)}(t-\theta_3\varepsilon)\,\varepsilon^3,
    \end{align*}
    for some \(0<\theta_2,\theta_3<1\).  Subtracting and dividing by \(2\varepsilon\) yields the formula for \(D_0f(t)\) and
    \[
        D_0f(t)-f'(t)
        =\frac{\varepsilon^2}{12}\bigl[f^{(3)}(t+\theta_2\varepsilon)
        +f^{(3)}(t-\theta_3\varepsilon)\bigr]
        =O(\varepsilon^2).
    \]
    This completes the proof.
\end{proof}

\begin{proposition}
    Let \(f:\mathbb{R}\to\mathbb{R}\) be differentiable, let \(t\in\mathbb{R}\) and \(\varepsilon>0\).  In BF16 arithmetic (1‐bit sign, 8‐bit exponent, 7‐bit significand) with unit roundoff
    $
        \eta = 2^{-7}
    $,
    define
    \[
        f_\pm = f(t\pm\varepsilon),\quad
        \Delta = f_+ - f_-,
    \]
    \[
        E_1 = \frac{\mathrm{fl}(f_+)\;-\;\mathrm{fl}(f_-)}{2\,\varepsilon},
        \qquad
        E_2 = \mathrm{fl}\!\Bigl(\frac{f_+}{2\varepsilon}\Bigr)
        -\mathrm{fl}\!\Bigl(\frac{f_-}{2\varepsilon}\Bigr).
    \]
    Suppose in addition that

    (1) \(|\Delta|<2^{-126}\), so that \(\Delta\) (and any nearby perturbation) lies in the BF16 subnormal range;

    (2) writing \(\mathrm{fl}(f_\pm)=f_\pm(1+\delta_\pm)\) with \(|\delta_\pm|\le\eta\), one has \(\bigl|f_+\delta_+ - f_-\delta_-\bigr|<2^{-126}\), so \(\tilde f_+-\tilde f_-\) remains subnormal;

    (3) \(\bigl|f_\pm/(2\varepsilon)\bigr|\ge2^{-126}\), so each product \(f_\pm/(2\varepsilon)\) lies in the normalized range;

    (4) \(|f_+| + |f_-| = O(|\Delta|)\), so that any rounding in the two multiplications is not amplified by a large subtraction.

    Then the “subtract‐then‐scale” formula \(E_1\) may incur a relative error of order \(O(1)\), whereas the “scale‐then‐subtract” formula \(E_2\) retains a relative error of order \(O(\eta)\).
\end{proposition}

\begin{proof}
    We use two BF16 rounding models:
    (i) if \(x\in[2^{-126},2^{128})\) then \(\mathrm{fl}(x)=x(1+\delta)\), \(|\delta|\le\eta\);
    (ii) for any \(x\) (including subnormals), \(\bigl|\mathrm{fl}(x)-x\bigr|\le\tfrac12\,\mathrm{ulp}(x)\), where \(\mathrm{ulp}_{\rm sub}=2^{-133}\) for subnormals.

    Set \(\tilde f_\pm=\mathrm{fl}(f_\pm)=f_\pm(1+\delta_\pm)\), \(|\delta_\pm|\le\eta\).

    \medskip\noindent\textbf{Error in \(E_1\).}
    By (1) and (2),
    \(\tilde f_+-\tilde f_-=\Delta+(f_+\delta_+ - f_-\delta_-)\)
    lies in the subnormal range.  Hence
    \[
        d = \mathrm{fl}(\tilde f_+ - \tilde f_-)
        = (\tilde f_+ - \tilde f_-) + e_d,
        \quad |e_d|\le\tfrac12\,\mathrm{ulp}_{\rm sub}=2^{-134}.
    \]
    Thus
    \[
        d = \Delta + (f_+\delta_+ - f_-\delta_-) + e_d,
        \qquad
        |e_d|/|\Delta| = O(2^{-134}/|\Delta|) \gg \eta.
    \]
    Dividing by \(2\varepsilon\) and rounding gives
    \[
        E_1 = \mathrm{fl}\!\bigl(d/(2\varepsilon)\bigr)
        = \frac{d}{2\varepsilon}(1+\delta_q),
        \quad |\delta_q|\le\eta,
    \]
    so the relative error in \(E_1\) can be \(O(1)\).

    \medskip\noindent\textbf{Error in \(E_2\).}
    By (3), each \(f_\pm/(2\varepsilon)\) is normalized, so
    \[
        g_\pm = \mathrm{fl}\!\Bigl(\frac{f_\pm}{2\varepsilon}\Bigr)
        = \frac{f_\pm}{2\varepsilon}(1+\delta'_\pm),
        \quad |\delta'_\pm|\le\eta.
    \]
    Subtracting and rounding (still normalized) gives
    \[
        E_2 = \mathrm{fl}(g_+ - g_-)
        = (g_+ - g_-)(1+\delta'_d),
        \quad |\delta'_d|\le\eta.
    \]
    Since
    \[
        g_+ - g_- = \frac{\Delta}{2\varepsilon}
        + \frac{f_+\delta'_+ - f_-\delta'_-}{2\varepsilon},
    \]
    we obtain
    \[
        E_2
        = \frac{\Delta}{2\varepsilon}(1+\delta'_d)
        + \frac{f_+\delta'_+ - f_-\delta'_-}{2\varepsilon}(1+\delta'_d).
    \]
    The second term has magnitude \(\le\eta\,\frac{|f_+|+|f_-|}{2\varepsilon}(1+\eta)\),
    and by (4) its relative size to \(\Delta/(2\varepsilon)\) is
    \(\displaystyle O\bigl(\eta\,\tfrac{|f_+|+|f_-|}{|\Delta|}\bigr)=O(\eta).\)

    Hence \(E_1\) may suffer \(O(1)\) relative error, while \(E_2\) attains \(O(\eta)\) relative accuracy under (1)–(4).
\end{proof}

\subsubsection{Calcluation of Transport}
\label{app:cal_transport}
\paragraph{Transport transformation from EDM to \method.}
Take the formula (8) from EDM~\citep{karras2022elucidating}, one can deduce:
\small{
    \begin{align*}
        \E_{\sigma, \xx, \nn} & \left[\lambda(\sigma) c_{out}(\sigma)^2 \norm{\mF_\theta(c_{in}(\sigma) \cdot (\xx + \nn); c_{noise}(\sigma)) - \frac{1}{c_{out}(\sigma)}(\xx - c_{skip}(\sigma) \cdot (\xx + \nn))}_2^2 \right]                                                                                                                                                                \\
                              & = \E_{\sigma, \xx, \zz} \left[\norm{\mF_\theta(\frac{1}{\sqrt{\sigma^2 + \sigma^2_{data}}} \cdot (\xx + \sigma \zz); \frac{1}{4}\ln(\sigma)) - \frac{\sqrt{\sigma^2_{data} + \sigma^2}}{\sigma \cdot \sigma_{data}}(\xx - \frac{\sigma^2_{data}}{\sigma^2 + \sigma^2_{data}} \cdot (\xx + \sigma\zz))}_2^2 \right]                                              \\
                              & = \E_{\sigma, \xx, \zz} \left[\norm{\mF_\theta(\frac{1}{\sqrt{\sigma^2 + \sigma^2_{data}}} \cdot (\xx + \sigma \zz); \frac{1}{4}\ln(\sigma)) - \frac{\sqrt{\sigma^2_{data} + \sigma^2}}{\sigma \cdot \sigma_{data}}(\frac{\sigma^2}{\sigma^2 + \sigma^2_{data}} \cdot \xx - \frac{\sigma^2_{data}}{\sigma^2 + \sigma^2_{data}} \cdot \sigma\zz)}_2^2 \right]    \\
                              & = \E_{\sigma, \xx, \zz} \left[\norm{\mF_\theta(\frac{1}{\sqrt{\sigma^2 + \sigma^2_{data}}} \cdot (\xx + \sigma \zz); \frac{1}{4}\ln(\sigma)) - (\frac{\sigma \sigma_{data}^{-1}}{\sqrt{\sigma^2_{data} + \sigma^2}} \cdot \xx - \frac{\sigma_{data}}{\sqrt{\sigma^2 + \sigma^2_{data}}} \cdot \zz)}_2^2 \right]                                                 \\
                              & = \E_{\sigma, \xx, \zz} \left[\norm{\mF_\theta(\frac{\sigma}{\sqrt{\sigma^2 + \sigma^2_{data}}} \cdot \zz + \frac{1}{\sqrt{\sigma^2 + \sigma^2_{data}}} \cdot \xx ; \frac{1}{4}\ln(\sigma)) - (\frac{-\sigma_{data}}{\sqrt{\sigma^2 + \sigma^2_{data}}} \cdot \zz + \frac{\sigma \sigma_{data}^{-1}}{\sqrt{\sigma^2_{data} + \sigma^2}} \cdot \xx)}_2^2 \right] \\
                              & = \E_{\sigma, \xx, \zz} \left[\norm{\mF_\theta(\frac{\sigma}{\sqrt{\sigma^2 + \frac{1}{4}}} \cdot \zz + \frac{1}{\sqrt{\sigma^2 + \frac{1}{4}}} \cdot \xx ) - (\frac{-0.5}{\sqrt{\sigma^2 + \frac{1}{4}}} \cdot \zz + \frac{2\sigma}{\sqrt{\sigma^2 + \frac{1}{4}}} \cdot \xx)}_2^2 \right]
    \end{align*}}
where $\sigma_{data} = \frac{1}{2},\ \nn = \sigma \cdot \zz$.
\normalsize

\end{document}